\documentclass[conference]{IEEEtran}
\IEEEoverridecommandlockouts
\usepackage{cite}
\usepackage{amsmath,amssymb,amsfonts}
\usepackage{algorithmic}
\usepackage{graphicx}
\usepackage{textcomp}
\usepackage{xcolor}

\usepackage{booktabs}
\usepackage{caption}
\usepackage{amsthm}
\usepackage{subcaption}
\usepackage{float}
\usepackage{mathtools}
\usepackage{amsthm}
\usepackage{xcolor}
\usepackage{multirow}
\usepackage{colortbl}
\usepackage{algorithm}
\usepackage{balance}
\usepackage{etoolbox}
\usepackage{blindtext}

\usepackage{stmaryrd}

\newtheorem{theorem}{Theorem}

\DeclareMathOperator*{\argmin}{arg\,min}

\theoremstyle{definition}

\theoremstyle{remark}

\DeclareMathOperator{\Tr}{Tr}
\usepackage{soul}

\def\BibTeX{{\rm B\kern-.05em{\sc i\kern-.025em b}\kern-.08em
    T\kern-.1667em\lower.7ex\hbox{E}\kern-.125emX}}
\begin{document}

\title{Towards Robust Graph Structural Learning Beyond Homophily via Preserving Neighbor Similarity}

\author{\IEEEauthorblockN{Yulin Zhu\IEEEauthorrefmark{1}, Yuni Lai\IEEEauthorrefmark{2}\thanks{The first two authors contributed equally to the paper.}, Xing Ai\IEEEauthorrefmark{2}, Wai Lun LO\IEEEauthorrefmark{1}, Gaolei Li\IEEEauthorrefmark{3}, \\Jianhua Li\IEEEauthorrefmark{3}, Di Tang\IEEEauthorrefmark{4}, Xingxing Zhang\IEEEauthorrefmark{4}, Mengpei Yang\IEEEauthorrefmark{5}, Kai Zhou\thanks{Prof. Kai Zhou is the corresponding author.}\IEEEauthorrefmark{2}}
\IEEEauthorblockA{\IEEEauthorrefmark{1}\textit{Dept. of Computer Science}, \textit{Hong Kong Chu Hai College}, HKSAR\\ 
ylzhu@chuhai.edu.hk, wllo@chuhai.edu.hk}
\IEEEauthorblockA{\IEEEauthorrefmark{2}\textit{Dept. of Computing}, \textit{The Hong Kong Polytechnic University}, HKSAR, \\ cs-yuni.lai@polyu.edu.hk, xing96.ai@connect.polyu.hk, kaizhou@polyu.edu.hk}
\IEEEauthorblockA{\IEEEauthorrefmark{3}\textit{School of Cyber Science and Engineering}, \textit{Shanghai Jiao Tong University}, Shanghai, China, \\ gaolei\_li@sjtu.edu.cn, lijh888@sjtu.edu.cn}
\IEEEauthorblockA{\IEEEauthorrefmark{4}\textit{Shanghai CESI Technology Co., Ltd}, Shanghai, China, \\ ditonytang@hotmail.com, zhangxx@cesi.cn}
\IEEEauthorblockA{\IEEEauthorrefmark{5}\textit{China Electronics Standardization Institute}, Shanghai, China, \\ yangmp@cesi.cn}
}

\maketitle

\begin{abstract}
Despite the tremendous success of graph-based learning systems in handling structural data, it has been widely investigated that they are fragile to adversarial attacks on homophilic graph data, where adversaries maliciously modify the semantic and topology information of the raw graph data to degrade the predictive performances. Motivated by this, a series of robust models are crafted to enhance the adversarial robustness of graph-based learning systems on homophilic graphs. However, the security of graph-based learning systems on heterophilic graphs remains a mystery to us. To bridge this gap, in this paper, we start to explore the vulnerability of graph-based learning systems regardless of the homophily degree, and theoretically prove that the update of the negative classification loss is negatively correlated with the pairwise similarities based on the powered aggregated neighbor features. The theoretical finding inspires us to craft a novel robust graph structural learning strategy that serves as a useful graph mining module in a robust model that incorporates a dual-kNN graph constructions pipeline to supervise the neighbor-similarity-preserved propagation, where the graph convolutional layer adaptively smooths or discriminates the features of node pairs according to their affluent local structures. In this way, the proposed methods can mine the ``better" topology of the raw graph data under diverse graph homophily and achieve more reliable data management on homophilic and heterophilic graphs.
\end{abstract}


\section{Introduction}
Relational data is ubiquitous in diverse domains, including biostatistics, finance, and cryptocurrency~\cite{battaglia2018relational,zhang2020deep,wu2020comprehensive}. The remarkable success of graph-based learning methods has resulted in their widespread adoption in various graph representation learning frameworks. One of the most typical graph-based learning methods is graph neural networks (GNNs). GNNs have shown exceptional performance in tasks such as node classification~\cite{GCN,GraphSage}, graph classification~\cite{li2019semi,shervashidze2011weisfeiler} and link prediction~\cite{shibata2012link,daud2020applications} etc. GNNs excel in graph representation learning due to their tailored propagation mechanism, which is especially effective for handling relational data.

GNNs commonly assume \textit{graph homophily}~\cite{homophily}, which means that connected nodes tend to share similar features. However, real-world graphs often exhibit \textit{heterophily}, where dissimilar nodes also tend to connect with each other. %connected nodes tend to share dissimilar features) that exist in the real world and the 
For instance, in fraud detection networks, the fraudsters are usually connected with benign users and mimic their behaviors so as to evade anomaly detection. 
%Vanilla GNNs struggle to handle heterophilic graphs effectively~\cite{H2GCN,FAGNN,GPRGNN}. 
Fortunately, numerous GNN variants~\cite{H2GCN,GBKGNN,FAGNN,GPRGNN,ACMGNN} have emerged to address this limitation via introducing useful techniques such as ego- and neighbor-embeddings separation~\cite{H2GCN}, aggregation with high-pass filter~\cite{FAGNN}, concatenating with higher order neighbors' embeddings~\cite{GPRGNN}. These powerful techniques expand the application of the GNN framework and enrich the family of graph-based learning methods. 
%That is, each node aggregates semantic information from its neighbors. 
%This semantic information is combined with ego information to create a more useful representation and serves as high-quality inputs for downstream tasks such as node classification. 
%Hence, the nature of the propagation mechanism demonstrates that the vanilla GNN tends to smooth the node features and gains a big achievement over homophilic graphs~\cite{homophily} (connected nodes tend to share similar features). 

%However, the limitation of GNN is that the function of the aggregation mechanism of the graph convolutional layer degenerates under the heterophilic (connected nodes tend to share dissimilar features) scenario~\cite{H2GCN,FAGNN,GPRGNN}. Intuitively, the smoothness property of the graph convolution operation will blur the discrimination of the connected node pairs with distinct labels. Fortunately, plenty of GNN variants~\cite{H2GCN,GBKGNN,FAGNN,GPRGNN,ACMGNN} endeavor to mitigate this problem by introducing some useful techniques such as ego- and neighbor-embeddings separation~\cite{H2GCN}, aggregation with high-pass filter~\cite{FAGNN}, concatenating with higher order neighbors' embeddings~\cite{GPRGNN} etc. Those powerful techniques extend the application of the GNN framework and enrich the family of graph-based-deep-learning models.
\begin{figure}[h]
    \centering
    \begin{subfigure}[b]{0.234\textwidth}
    	\centering
    	\includegraphics[width=\textwidth,height=3.cm]{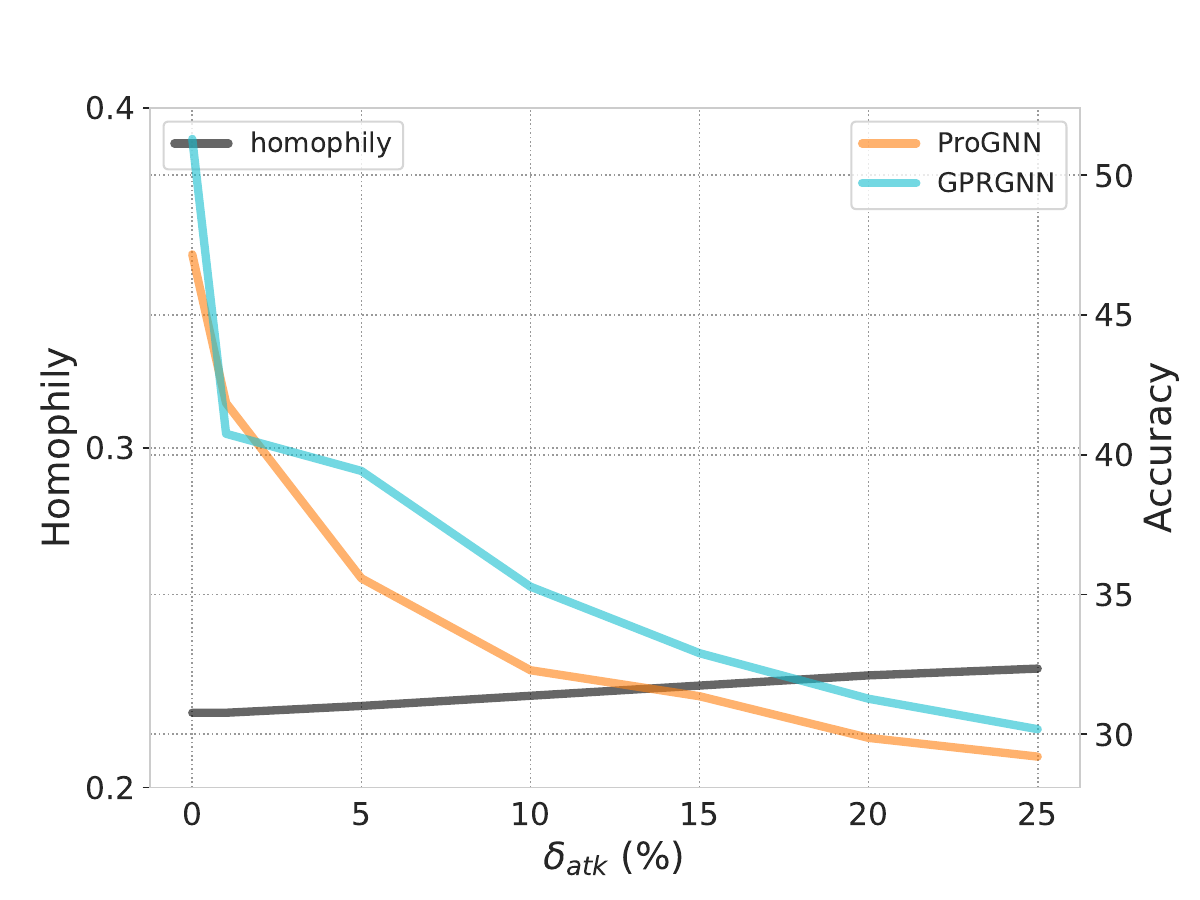}
    	\caption{Mettack}
    \end{subfigure}
    \hfill
    \begin{subfigure}[b]{0.234\textwidth}
    \centering
     	\includegraphics[width=\textwidth,height=3.cm]{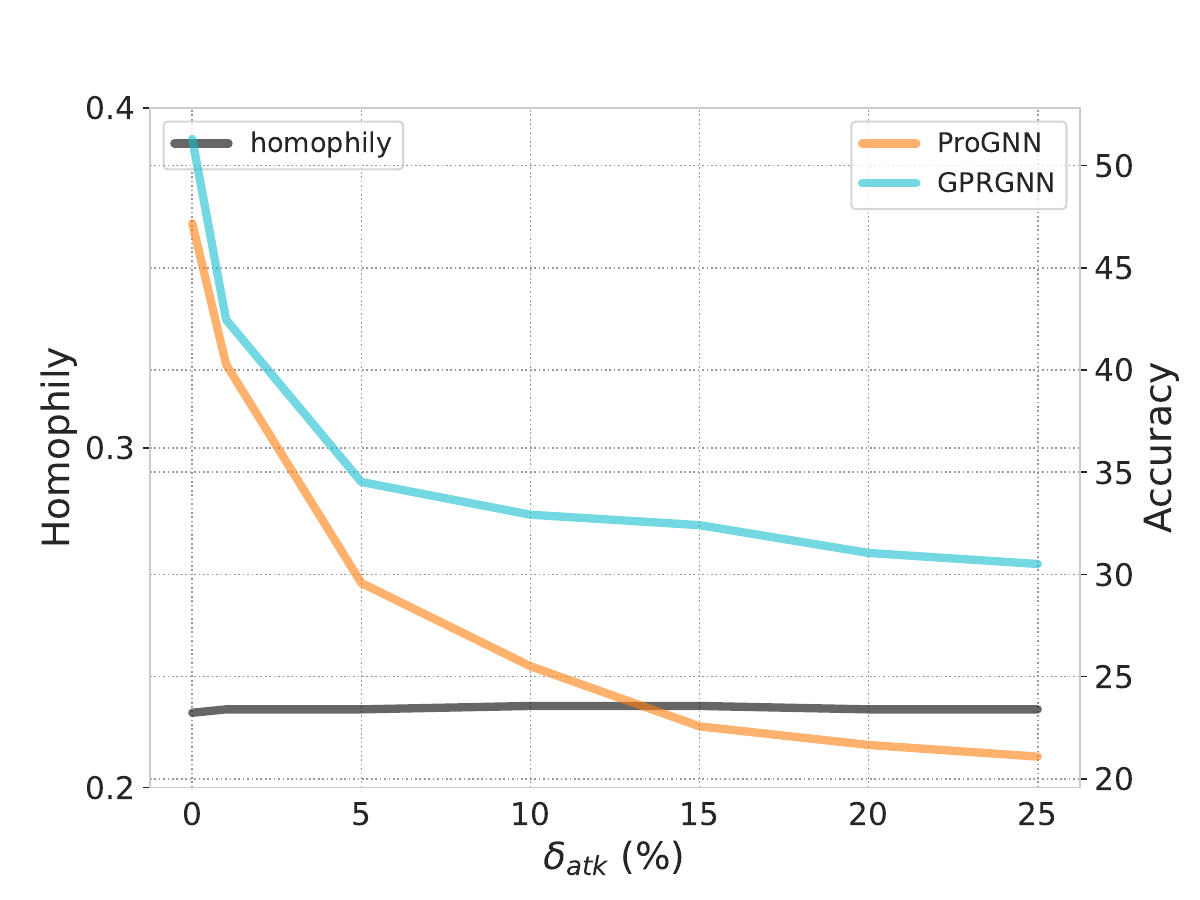}
     	\caption{Minmax}
    \end{subfigure}
    \caption{Homophily degrees and accuracies of the poisoned heterophilic graphs under different attacking scenarios.}
    \label{fig-homophily-ratio-squirrel}
    %\vspace{-0.cm}
\end{figure}

Despite the remarkable achievements of GNNs, extensive research has revealed their vulnerability to graph adversarial attacks, which involve adding or deleting a fraction of links in the original graph~\cite{Nettack,Mettack,TopologyAttack,BinarizedAttack}. %In the adversary scenario, the graph attacker executes imperceptible structural perturbations by adding or deleting a fraction of links to the clean graph to degenerate the semi-supervised node classification performances of GNNs. 
%In this way, the victim nodes will connect to different nodes for malicious purposes without significantly altering the statistics of the graph data such as node degree distribution~\cite{Nettack}. 
In particular, the adversarial robustness of GNNs over homophilic graphs has been thoroughly examined~\cite{ProGNN,GCNJaccard,GCNSVD,GNNGUARD}. One crucial observation is that adversarial attacks tend to introduce connections between dissimilar node pairs while removing links that connect similar nodes~\cite{GCNJaccard}, \textit{thus decreasing the overall homophily level of the graph}. Intuitively, when inter-class links are added or intra-class links are deleted, the message-passing operation becomes less effective in distinguishing between different clusters, leading to a degradation in the quality of node representations.
%The similarity can be measured by the distinction of node features or node labels. 
%Intuitively, adding inter-class links or deleting intra-class links will decrease the homophily level of the graph data and in turn, the message-passing operation will mix up the differentiation among different clusters and thus degrade the quality of node embeddings. 
Then, based on this observation, researchers have developed various robust models~\cite{GCNJaccard,GNNGUARD,ProGNN} that adhere to the same principle of increasing the homophily level of the graph to achieve robustness. For instance, ProGNN~\cite{ProGNN} learns a new graph structure to minimize the distance between the connected nodes' features, etc. %GNNGUARD~\cite{GNNGUARD} prune the links which connect dissimilar node pairs during GNN training.
%However, the above-mentioned observations fail to curve the adversarial robustness of GNNs over heterophilic graphs. 
Unfortunately, the previous observation does not hold for attacks over heterophilic graphs. To illustrate this, we employ two representative adversarial attacks, Mettack~\cite{Mettack} and Minmax~\cite{TopologyAttack}, to poison heterophilic graph--Squirrel~\cite{chameleon}. It is worth noting that we use one typical robust model (ProGNN~\cite{ProGNN}) and one GNN variant crafted for heterophilic graphs (GPRGNN~\cite{GPRGNN}) to evaluate the classification performances. As shown in Fig.~\ref{fig-homophily-ratio-squirrel}, the attacks are still significantly effective in decreasing the node classification accuracy; however, the homophily ratio~\cite{H2GCN} of the graph does not change significantly (in some cases it even increases slightly). Consequently, \textit{previous robust models like GCNJaccard, GNNGUARD and ProGNN ~\cite{GCNJaccard,GNNGUARD,ProGNN} that rely on restoring the homophily degree fail to exhibit robustness over heterophilic graphs}.

Thus, we are motivated to propose a robust model beyond homophily supervised by a useful robust graph structural learning approach. This objective can be divided into two coherent tasks: \textbf{1)} analyze the characteristics of attacks on graph data, considering properties that go beyond homophily degrees; \textbf{2)} based on the insights gained from the vulnerability analysis, develop a robust model supervised by a crafted graph structural learning approach for both heterophilic and homophilic graphs. 
%For instance, Fig.~\ref{fig-homophily-ratio-squirrel} depicts the homophily ratios~\cite{H2GCN} (Referring to Sec.~\ref{sec-homophily}) and classification accuracies of GPRGNN~\cite{GPRGNN} (a GNN variant specially crafted for heterophilic graphs) for Squirrel~\cite{chameleon} dataset under different attacking scenarios. The results show that the graph attackers can still significantly degenerate the node classification performance of GNN for the heterophilic graph. However, in contrast to the homophilic graphs, the homophily ratio of the heterophilic graph increases after being attacked. These contradictions demonstrate that current robust models~\cite{GCNJaccard,GNNGUARD,ProGNN} fail to defend against the graph structural attacks for heterophilic graphs. In reality, most robust models rely on the similarities of ego features to mitigate the malicious effects caused by the graph attacker. However, for heterophilic graphs, the current robust techniques cannot tell apart malicious inter-class links from normal inter-class links and hence fail to prune the malicious effects on the poisoned graphs. Essentially, it is urgent to craft a novel technique to effectively pick out malicious effects of the graph data beyond ego features' similarities.
%with great probability since the majority of the robust models rely on the assumption that the graph attacker tends to add inter-class links. Hence, it is vital to explore the adversarial robustness of GNNs for heterophilic graphs to fill this gap. 
To this end, we start from a theoretical observation--\textit{the update of the negative classification loss} (log-likelihood loss $d\mathcal{L}_{atk}$) and \textit{the pairwise similarity matrix} (computed from the $\tau$-th powered aggregated neighbor features $\mathbf{A}^{\tau}\mathbf{X}$) \textit{are negatively correlated}. Consequently, the adversaries tend to connect the dissimilar node pairs as measured by the neighbor features. That is, we generalize the notion of similarity from measuring \textit{ego features} to \textit{neighbor features}. 
%That is, we extend the idea of the adversarial property from considering the similarities based on ego features to neighbor features. 
We present compelling empirical evidence for this finding in Fig.~\ref{Fig-density-attributes}. 
%support this finding and demonstrate that \textit{malicious links are more likely to be such links that connect dissimilar node pairs based on $\mathbf{A}^{\tau}\mathbf{X}$ rather than $\mathbf{X}$}.
As a result, this valuable insight can be leveraged to detect malicious links and guide the development of a robust graph structural learning strategy to mine for ``better" topologies.
%As a result, this insight can serve as malicious links detection and supervise the robust graph learning strategy.

%To handle the above-mentioned problem, we propose a novel GNN framework termed \underline{N}eighbor \underline{S}imilarity \underline{P}reserving \underline{G}raph \underline{N}eural \underline{N}etwork (\textbf{NSPGNN}) and start from a theoretical observation--\textit{the update of the attack loss (log-likelihood loss) $d\mathcal{L}_{atk}$ and the pairwise similarity matrix based on the $\tau$-th powered aggregated neighbor features $\mathbf{A}^{\tau}\mathbf{X}$ are negatively correlated}. Consequently, the graph attacker tends to connect the dissimilar node pairs based on the neighbor features. In other words, we extend the idea of the adversarial property from considering the similarities based on ego features to neighbor features. 
%This extension helps us to craft a universal robust GNN framework that can enhance the adversarial robustness of GNNs on both homophilic graphs and heterophilic graphs. 
%Affluent empirical studies in Fig.~\ref{Fig-density-attributes} support this finding and demonstrate that \textit{malicious links are more likely to be such links that connect dissimilar node pairs based on $\mathbf{A}^{\tau}\mathbf{X}$ rather than $\mathbf{X}$}. As a result, this insight can serve as malicious links detection and supervise the robust graph learning strategy. 

%Based on the above-mentioned insight, 
We thus propose a novel robust graph learning framework termed \underline{N}eighbor \underline{S}imilarity \underline{P}reserving \underline{G}raph \underline{N}eural \underline{N}etwork (\textbf{NSPGNN}) which contains the \underline{R}obust \underline{G}raph \underline{S}tructural \underline{L}earning (RGSL) approach to supervise the \textit{neighbor-similarity-preserved propagation} during training to obtain high-quality node representations. 
Specifically, we construct positive kNN graphs and negative kNN graphs according to the similarity scores of $\mathbf{A}^{\tau}\mathbf{X}$, which endeavor to contrastively supervise the message-passing mechanism to propagate node features while preserving the neighbor features consistency. 
More specifically, the structural information of the positive kNN graphs is implemented with a crafted adaptive attention mechanism to smooth the node's features with its nearest neighbors according to the similarity scores. 
In contrast, the structural information of the negative kNN graphs is implemented with another crafted adaptive attention mechanism to discriminate the node's features with its foremost neighbors to preserve the aggregated node features' similarity from an opposite perspective. 
In the sequel, the node embeddings are fused to form the final node embeddings contain richer localized information, and are fed into the objective for training. As a result, the RGSL can serve as an effective supervisor to form a reliable message-passing mechanism in the learning module. 

%It is highlighted that the theoretical results do not restrict the homophily ratio of the graph data. Thus, preserving neighbor similarity can be effective both for homophilic and heterophilic graphs. 
It is important to highlight that while our RGSL guided robust model was initially designed for heterophilic graphs, it has also exhibited remarkable robustness against attacks on homophilic graphs, even outperforming specifically designed robust baselines. This versatility makes \textbf{NSPGNN} a robust model beyond homophily and can achieve a more trustworthy graph data management system both on heterophilic and homophilic graphs. In summary, our work presents three main contributions:
\begin{itemize}
    \item We both theoretically and empirically analyze the vulnerability of the graph learning system and reveal that preserving the neighbor similarity can enhance its adversarial robustness regardless of homophily degree.
    \item Based on the insights of our vulnerability analysis, we propose a robust graph learning framework --\textbf{NSPGNN} by introducing an RGSL approach to contrastively supervise the neighbor-similarity-preserved propagation and adaptively capture the affluent localized information in the graph data.
    \item We conduct comprehensive experiments to demonstrate the remarkable performances of the proposed method, which outperforms other baselines on both clean and noisy graph data under diverse homophily.
\end{itemize}

\section{Related Works}
\subsection{Graph Learning for Heterophilic Graphs}
GNNs have achieved tremendous success in tackling the semi-supervised learning problem for relational data. However, there exists a limitation of the vanilla GNNs--The aggregation mechanism of the graph convolutional operation is specially crafted for the homophilic graphs, which narrows the application of GNN in the real world. Fortunately, a series of GNN variants have been proposed to bypass this limitation and can handle heterophilic graphs. For instance, H2GCN~\cite{H2GCN} crafted three vital designs: ego-
and neighbor-embedding separation, higher-order neighborhoods and a combination of intermediate representations to enhance the expressive power of GNN for heterophilic graphs. FAGNN~\cite{FAGNN} can adaptively change the proportion of low-frequency and high-frequency signals to efficiently mine the semantic and structural information of heterophilic graphs. GPRGNN~\cite{GPRGNN} introduced the generalized PageRank GNN framework to adaptively assign the GPR weights to jointly optimize node features and topological information extraction. GBKGNN~\cite{GBKGNN} adopts a bi-kernel for feature extraction and a selection gate to enhance the representation learning of GNN over uneven homophily levels of heterophilic graphs. BMGNN~\cite{BMGNN} incorporates block modeling information into the aggregation process, which can help GNN to aggregate information from neighbors with distinct homophily degrees. ACMGNN~\cite{ACMGNN} proposes the adaptive channel mixing framework adaptively exploits aggregation, diversification and identity channels node-wisely to extract richer localized information for diverse node heterophily situations.

\subsection{Robust Models}
It has been widely explored that GNNs are vulnerable to graph structural attacks~\cite{Mettack,TopologyAttack,BinarizedAttack,Nettack} since the aggregation mechanism of the graph convolutional layer highly relies on topology information of the relational data. To address this problem, a battery of defense methods against the graph structural attacks has been investigated. For example, GCNJaccard~\cite{GCNJaccard} prunes the malicious links via the Jaccard index on the node attributes. GNNGUARD~\cite{GNNGUARD} removes the malicious links during training by considering the cosine similarity of node features. ProGNN~\cite{ProGNN} learns a new dense adjacency matrix and GNN simultaneously by penalizing three graph properties: matrix rank, the nuclear norm of the adjacency matrix and feature smoothness. However, the above-mentioned robust models are highly reliant on the homophily assumption and may not be suitable for boosting the robustness of GNN over heterophilic graphs.
Alternatively, GARNET~\cite{garnet} learns a new reduced-rank graph topology via spectral graph embedding and probabilistic graphical model to enhance the GNN's robustness. It can enhance the robustness of GNNs over heterophilic graphs since it does not depend on the homophily assumption. However, the strong assumption of the Gaussian graphical model will mitigate the quality of its base graph construction. 
\section{Preliminaries}
\subsection{Homophily and Heterophily}
\label{sec-homophily}
The diversification of the homophilic graph and heterophilic graph primarily arises from the matching degree between the target node and its surroundings. 
%Connected node pairs tend to share the same labels for the homophilic graphs and vice versa. 
It is worth noting that recent literature~\cite{H2GCN,GPRGNN,GBKGNN,FAGNN} uses the following metric to measure the homophily degree of the graph data: 
\begin{equation}
    \mathcal{H}(G)=\frac{|\{e_{uv}|e_{uv}\in\mathcal{E}, y_{u}=y_{v}\}|}{|\mathcal{E}|},
\end{equation}
where $\mathcal{E}$ is the edge set, $y_{u}$ represents the label for node $u$. %In reality, we regard those graphs with low-level $\mathcal{H}(G)$ as heterophilic graphs and high-level $\mathcal{H}(G)$ as homophilic graphs. 

%\subsection{Adversarial Environment}
\subsection{Semi-supervised Node Classification}
%We consider the adversarial robustness of heterophilic graphs for semi-supervised node classification. 
The input is an attributed graph $G=\{\mathcal{V},\mathbf{X},\mathbf{A},\mathcal{Y}\}$, where $\mathbf{X}\in\mathbb{R}^{N\times p}$ is the nodal attribute matrix, $N$ is the node number, $\mathbf{A}\in\{0,1\}^{N\times N}$ is the adjacency matrix where $\mathbf{A}_{uv}=1$ represents that the node $u$ is connected with node $v$ and vice versa, $\{\mathcal{Y}_{u}\}_{u=1}^{N}$ is the label for node $u$. The node set $\mathcal{V}$ is usually partitioned into training set $\mathcal{V}_{tr}$, validation set $\mathcal{V}_{val}$ and testing set $\mathcal{V}_{te}$ respectively. The most representative graph-based deep learning models are GNN and its variants. In particular, the GNN is formulated as an encoder $f_{\mathbf{W}}(\mathbf{X},\mathbf{A})\rightarrow \mathbf{Z}$, which maps the complex structural data to an Euclidean embedding space. Specially, there are two common graph filters~\cite{SGC,FAGNN,ACMGNN} for node representation learning:
\begin{equation}
    \text{low-pass: }\mathbf{Z}=\sigma(\mathbf{\hat{A}}\mathbf{X}\mathbf{W}),\ 
    \text{high-pass: }\mathbf{Z}=\sigma((\mathbf{I}-\mathbf{\hat{A}})\mathbf{X}\mathbf{W}),
\end{equation}
where $\mathbf{\hat{A}}=\tilde{\mathbf{D}}^{-\frac{1}{2}}\mathbf{\tilde{A}}\tilde{\mathbf{D}}^{-\frac{1}{2}}$, $\mathbf{\tilde{A}}=\mathbf{A}+\mathbf{I}$, $\mathbf{\tilde{D}}=\text{Diag}\{d_{ii}\}_{i=1}^{N}$, $d_{ii}=\sum_{j=1}^{N}\mathbf{A}_{ij}$.
Specially, the low-pass filter smooths the signal by averaging the information from neighboring nodes, while the high-pass filter enhances the discrimination between the representations of neighboring nodes.
%Then, the node embeddings matrix $\mathbf{Z}$ is fed into the negative log-likelihood (NLL) loss for training:
%\begin{equation}
%    \begin{split}
%        &\mathbf{W}^{*}=\argmin\mathcal{L}_{NLL}(\mathbf{X},\mathbf{A},\mathbf{Y},\mathcal{V}_{tr}),\\
%        &\text{where } \mathcal{L}_{NLL}=-\Tr(\mathbf{Y}\log\mathbf{S}^{\top}), \ \mathbf{S}=\text{softmax}(\mathbf{Z}),
%    \end{split}
%\end{equation}
%where $\mathbf{Y}\in\mathbb{R}^{N\times C}$ is the one-hot encoding label matrix and $C$ is the class number. Finally, we get the prediction for the testing set as:
%$\mathbf{Z}_{te}^{*}=f_{\mathbf{W}^{*}}(\mathbf{X},\mathbf{A},\mathcal{V}_{te})$.

\begin{figure*}[htp]
    \centering
    \begin{subfigure}[b]{0.32\textwidth}
    	\centering
    	\includegraphics[width=\textwidth,height=4.cm]{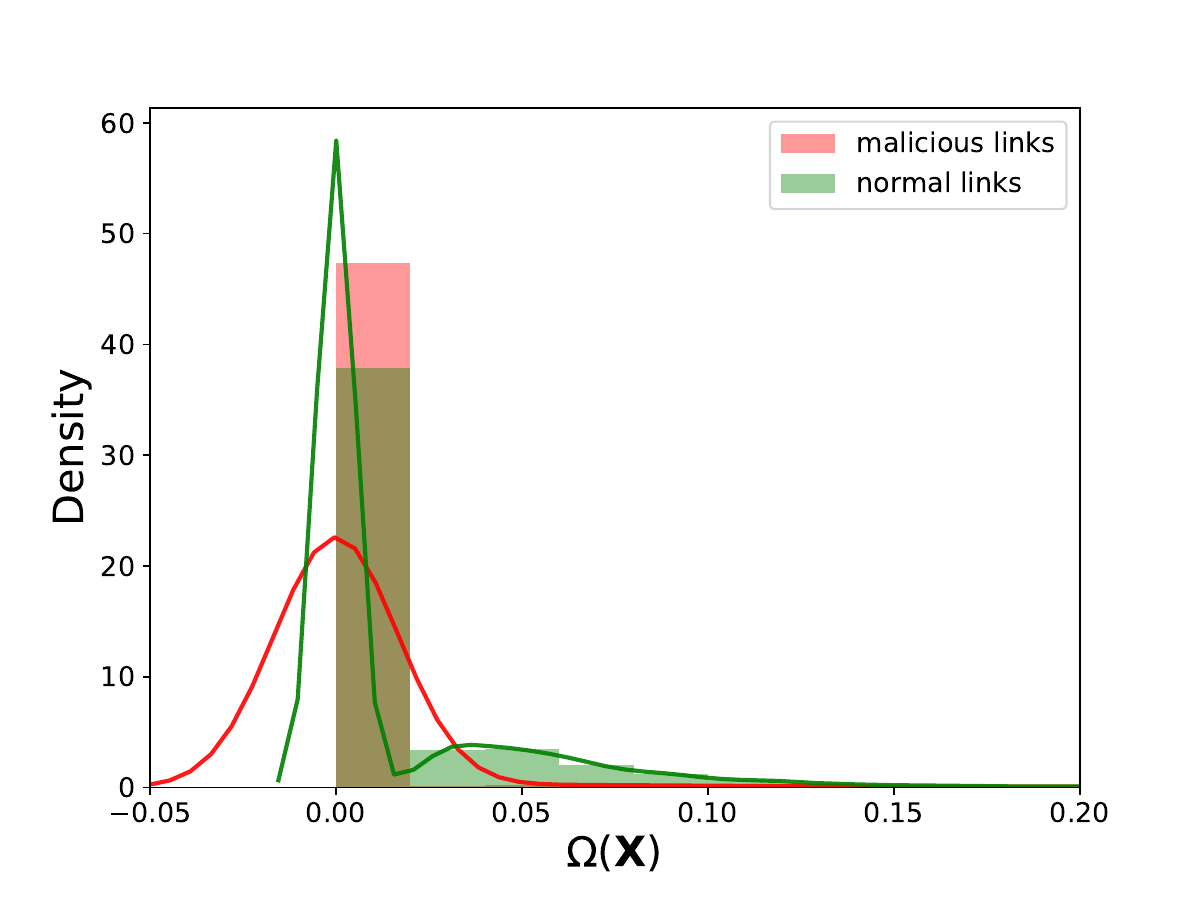}
    	\caption{Ego similarities for Mettack}
            \label{Fig-density-attributes-Mettack-X}
    \end{subfigure}
    \hfill
    \begin{subfigure}[b]{0.32\textwidth}
    	\centering
     	\includegraphics[width=\textwidth,height=4.cm]{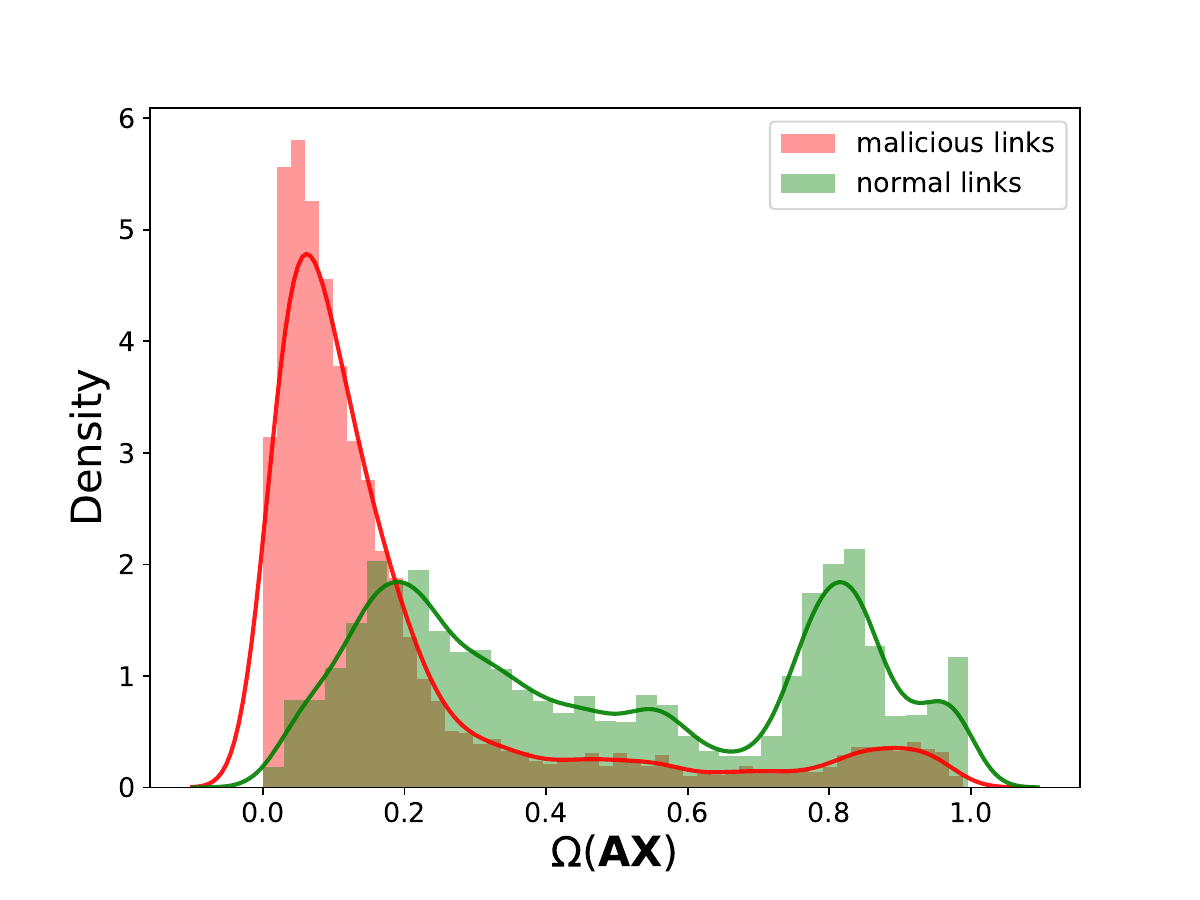}
     	\caption{One-hop similarities for Mettack}
    \end{subfigure}
    \hfill
    \begin{subfigure}[b]{0.32\textwidth}
    	\centering
     	\includegraphics[width=\textwidth,height=4.cm]{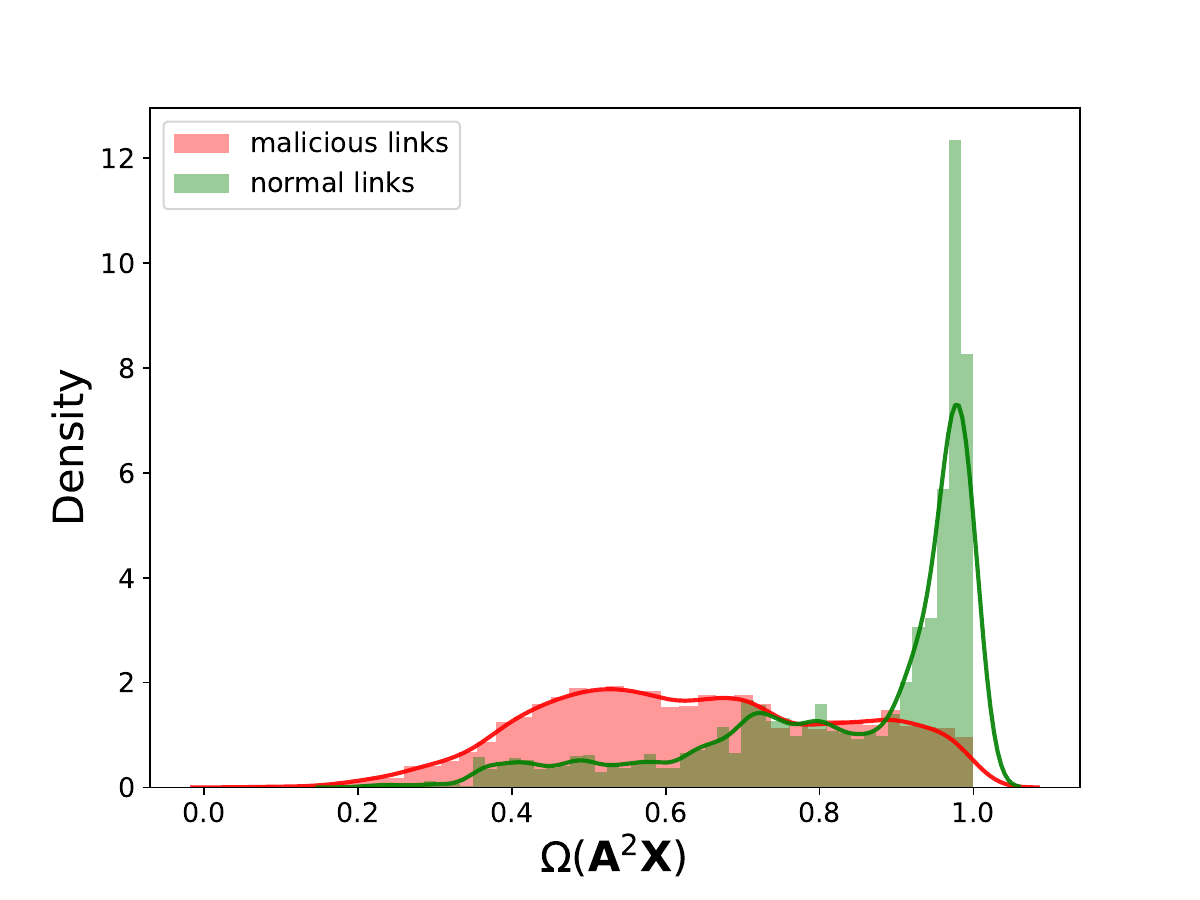}
     	\caption{Two-hop similarities for Mettack}
    \end{subfigure}
    \begin{subfigure}[b]{0.32\textwidth}
    	\centering
    	\includegraphics[width=\textwidth,height=4.cm]{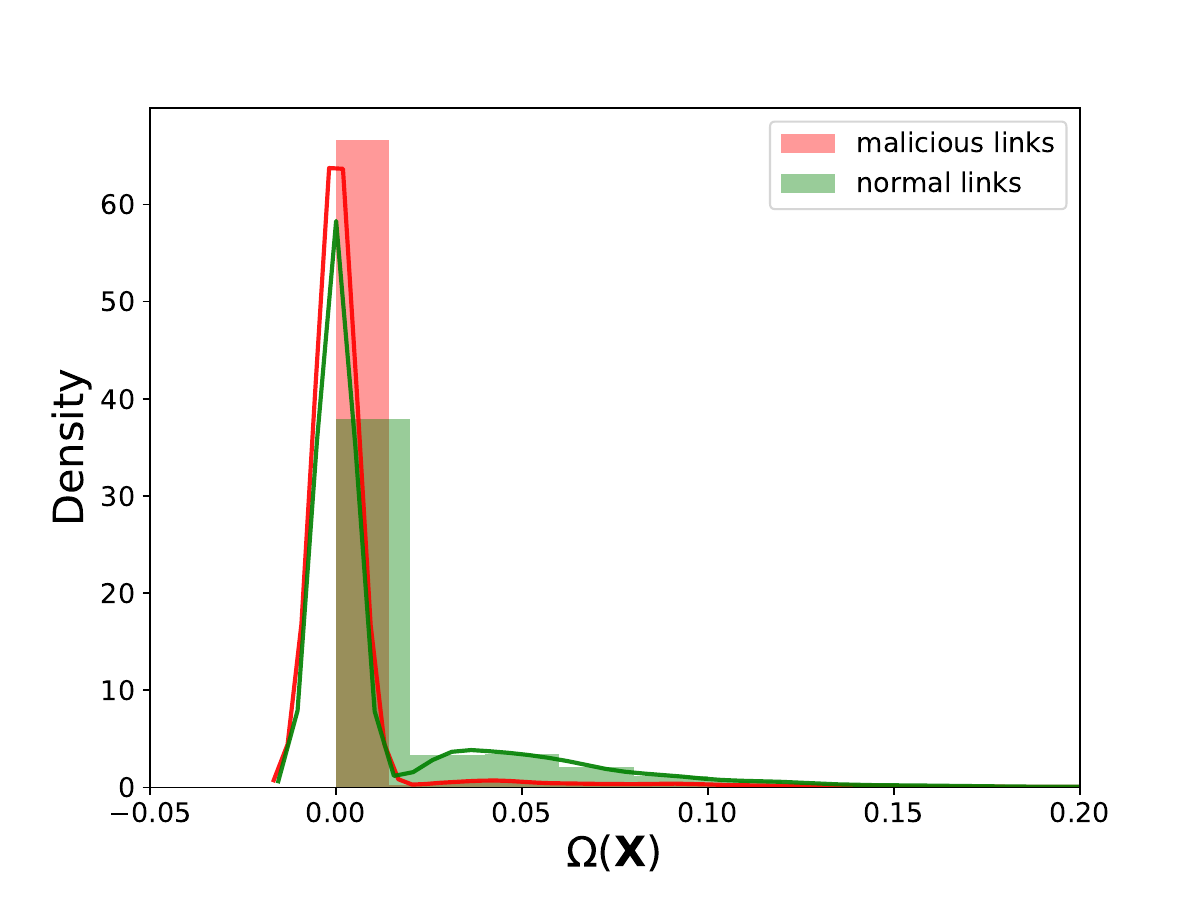}
    	\caption{Ego similarities for Minmax}
            \label{Fig-density-attributes-Minmax-X}
    \end{subfigure}
    \hfill
    \begin{subfigure}[b]{0.32\textwidth}
    	\centering
     	\includegraphics[width=\textwidth,height=4.cm]{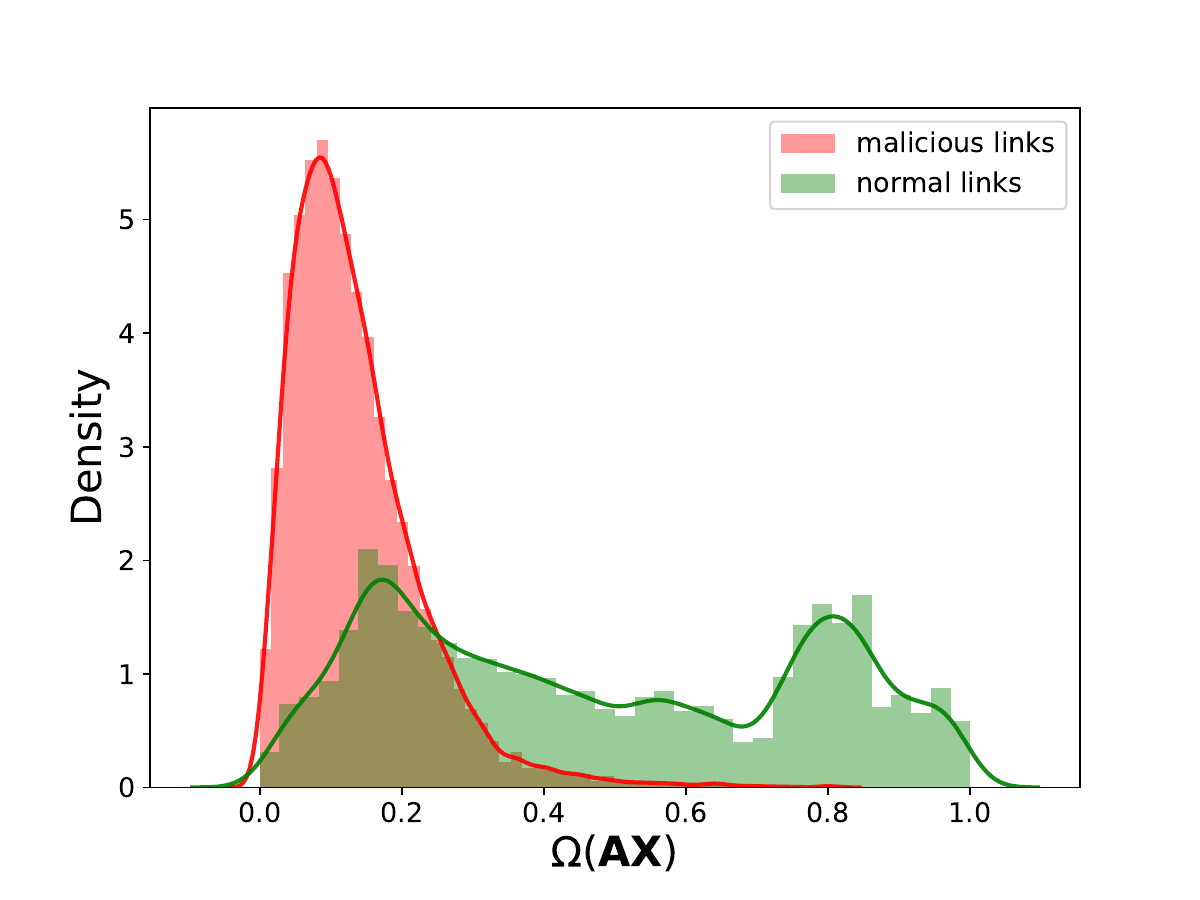}
     	\caption{One-hop similarities for Minmax}
    \end{subfigure}
    \hfill
    \begin{subfigure}[b]{0.32\textwidth}
    	\centering
     	\includegraphics[width=\textwidth,height=4.cm]{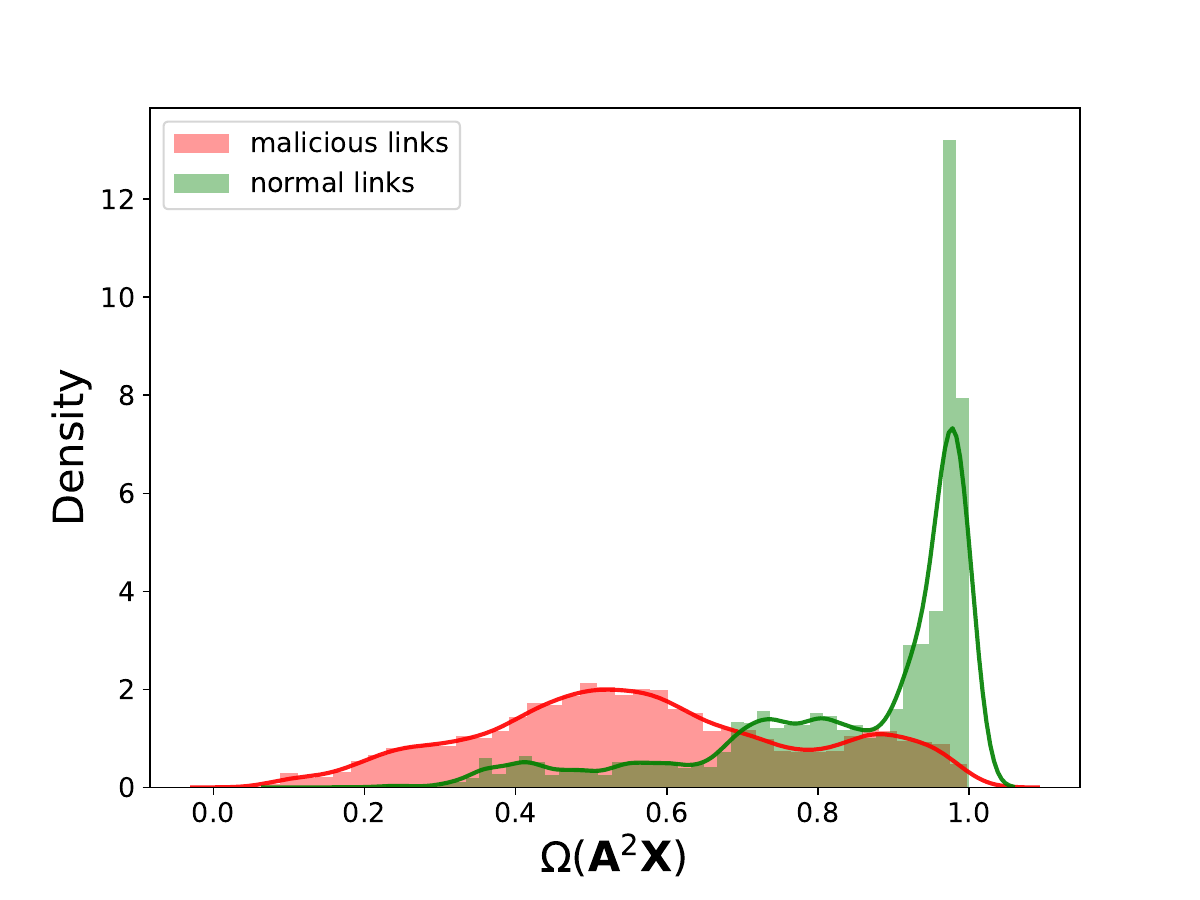}
     	\caption{Two-hop similarities for Minmax}
    \end{subfigure}
    %\vspace{-0.65cm}
    \caption{Density plots of Ego similarities $\Omega(\mathbf{X})$, one-hop similarities $\Omega(\mathbf{AX})$, and two-hop similarities $\Omega(\mathbf{A}^{2}\mathbf{X})$ between benign links and malicious links on heterophilic graph. The red area represents the density of similarity scores for malicious links injected by the structural attacks and the green area for benign links.}
    \label{Fig-density-attributes}
    %\vspace{-0.45cm}
\end{figure*}

\subsection{Graph Adversarial Attacks}
It has been widely explored that GNNs are vulnerable to graph structural attacks~\cite{Nettack,Mettack,TopologyAttack,BinarizedAttack,AttackCD}. 
%Specially, the semi-supervised nature of GNN is particularly suitable for the graph attacker to inject the poisons into the graph during the GNN training step~\cite{Mettack}. 
To this end, the graph attacker aims at manipulating the clean graph's structure $\mathbf{A}$ with limited budgets to minimize the predicting accuracy on the unlabeled nodes without significantly altering the graph property such as node degree distribution~\cite{BinarizedAttack}. Mathematically, the graph structural attacks can be formulated as a discrete bi-level optimization problem: 
\begin{subequations}
    \label{eqn-bilevel}
    \begin{align}
        &\mathbf{A}^{p}=\argmin_{\mathbf{A}} \mathcal{L}_{atk}=-\mathcal{L}_{NLL}(\mathbf{X},\mathbf{A},\mathbf{W}^{*},\mathbf{Y},\mathcal{V}_{te}), \label{eqn-bilevel-outer}\\
        &\text{s.t. } \mathbf{W}^{*}=\argmin_{\mathbf{W}} \mathcal{L}_{NLL}(\mathbf{X},\mathbf{A},\mathbf{W},\mathbf{Y},\mathcal{V}_{tr}), \label{eqn-bilevel-inner} \ \|\mathbf{A}^{p}-\mathbf{A}\|\leq B. %\nonumber
    \end{align}
\end{subequations}
%Usually, the researcher implements the greedy approach based on the gradient information~\cite{Mettack} or relaxes the discrete optimization problem to continuous space for ease of training~\cite{TopologyAttack}.

%Particularly, there exist two typical ways to solve this discrete bi-level optimization problem. Mettack~\cite{Mettack} firstly trains the inner loop Eqn.~\ref{eqn-bilevel-inner} for several steps. Then, the gradient $\nabla_{\mathbf{A}}\mathcal{L}_{atk}$ is computed and sorted in descending order. Next, The attacker chooses the node pair with the largest magnitude of gradients to modify until reaches the budget. Alternatively, MinMax~\cite{TopologyAttack} firstly optimizes the bi-level optimization problem under the continuous space and obtains the optimized structural manipulation matrix. Then, the graph attacker sorts the entries of the manipulation matrix in descending order and picks out the top-$B$ node pairs to modify. 

\section{Vulnerability Analysis} 
\label{sec-vulnerability}
\subsection{Inefficacy of Ego Features}
Our investigation begins with examining the vulnerability of GNNs in the context of heterophilic graphs.
Previous robust models (e.g., GCNJaccard, ProGNN, GNNGUARD  ~\cite{GCNJaccard, ProGNN, GNNGUARD}) developed for homophilic graphs essentially rely on \textit{similarities of ego features} (denoted as $\Omega(\mathbf{X})$; defined later) to distinguish between malicious and benign links in the poisoned graph. However, our observation is that the ego features are not as effective as previously for heterophilic graphs. Without loss of generalizability, we employ two representative graph structural attacks, i.e., Mettack~\cite{Mettack} and MinMax~\cite{TopologyAttack} to manipulate a heterophilic graph--Chameleon~\cite{chameleon}. As shown in Fig.~\ref{Fig-density-attributes-Mettack-X} and \ref{Fig-density-attributes-Minmax-X}, the ego similarities fail to effectively differentiate between malicious and benign links in heterophilic graphs.
%may not be suitable for enhancing the adversarial robustness of GNNs in heterophilic scenarios. For instance, as shown in Fig.~\ref{Fig-density-attributes-Mettack-X} and \ref{Fig-density-attributes-Minmax-X}, unlike homophilic graphs, the density of ego feature similarities fails to effectively differentiate between malicious and benign links in heterophilic graphs. 
In particular, a significant proportion of heterophilic links (connecting dissimilar node pairs) in the poisoned heterophilic graphs are actually benign.
%, and the vanilla strategy cannot distinguish malicious heterophilic links from benign ones. 
Consequently, existing robust GNN models based on homophily assumption cannot improve node classification performances for poisoned heterophilic graphs. 
Therefore, the main challenge lies in searching for \textit{a new strategy to differentiate malicious heterophilic links from benign ones} that goes beyond relying solely on ego similarities. To this end, we begin by analyzing the preference of attacks detailed below.

%Hence, the main challenge is to search for a new graph property to differentiate the malicious links from benign ones under the heterophilic scenario. 

%\subsection{Similarity for Heterophilic Graphs}

\subsection{Analysis of Attack Loss}
%To provide a comprehensive explanation for the ineffectiveness of GNNs on poisoned graphs and analyze the underlying rationale behind the remarkable effectiveness of graph structural attacks (Mettack~\cite{Mettack}, Minmax~\cite{TopologyAttack}), 
We begin with analyzing how the attack loss (defined in Eqn.~\ref{eqn-bilevel-outer}) would change according to structural perturbations. Note that in structural attacks (e.g., Mettack~\cite{Mettack}, Minmax~\cite{TopologyAttack}), the decision on which edges to perturb is determined by their impact on the attack loss (negative classification loss).
%attack loss, i.e., $d\mathcal{L}_{atk}$ defined in Eqn.~\ref{eqn-bilevel-outer}. This is because a common way to determine the structural perturbations is to measure the magnitude of the gradient information~\cite{Mettack}. 
Therefore, the exploration of the relationship between the update (small impact for each perturbation) of the attack loss and the graph data information such as the topology information $\mathbf{A}$ and the semantic information $\mathbf{X}$ can provide insights into the attack preferences. Understanding this relationship can guide the development of new defense strategies aimed at enhancing the robustness of GNNs.

We consider the simple SGC~\cite{SGC} as the victim model in our analysis, which is often used as a surrogate model for graph structural attacks. The SGC model is formulated as:
\begin{equation}
    \mathbf{S}=\text{softmax}(\mathbf{Z})=\text{softmax}(\hat{\mathbf{A}}^{\tau}\mathbf{XW}), 
\end{equation}
where $\tau$ is the number of graph convolutional layers. The attacker's goal is to introduce structural poisons to minimize the attack loss $\mathcal{L}_{atk}$. We delve into the update of the attack loss of the SGC model and obtain the following theoretical observation:
\begin{theorem}
    \label{theorem-1}
    The magnitude of the update of the attack loss $d\mathcal{L}_{atk}$ is negatively related to the nodes' aggregated feature similarity matrix $\mathbf{K}=\mathbf{A}^{\tau}\mathbf{X}(\mathbf{A}^{\tau}\mathbf{X})^{\top}$.
\end{theorem}
\begin{proof}
    In the inner loop of the attack objective defined in Eqn.~\ref{eqn-bilevel}, we have 
    \begin{equation}
        d\mathbf{Z}=\hat{\mathbf{A}}^{\tau}\mathbf{X}d\mathbf{W}=\hat{\mathbf{A}}^{\tau}\mathbf{X}\cdot\gamma\nabla_{\mathbf{W}}\mathcal{L}_{nll}=\hat{\mathbf{A}}^{\tau}\mathbf{X}\cdot\gamma\nabla_{\mathbf{Z}}\mathcal{L}_{nll}\nabla_{\mathbf{W}}\mathbf{Z}.
    \end{equation}
    On the other hand, we have:
    \begin{equation}
        \begin{split}
            &\nabla_{\mathbf{Z}_{i}}\mathcal{L}_{nll}=\nabla_{\mathbf{Z}_{i}}(-\sum_{c=1}^{C}\mathbf{Y}_{ic}\log\mathbf{S}_{ic})=%-\sum_{c=1}^{C}\mathbf{Y}_{ic}\nabla_{\mathbf{Z}_{i}}\log\mathbf{S}_{ic}\\
        %&\quad\quad\quad\ \ \ =-
        -\sum_{c=1}^{C}\frac{\mathbf{Y}_{ic}}{\mathbf{S}_{ic}}\nabla_{\mathbf{Z}_{i}}\mathbf{S}_{ic}.
        \end{split}
    \end{equation}
    However, we can have:
    \begin{equation}
        \begin{split}
            \nabla_{\mathbf{Z}_{j}}\mathbf{S}_{i}&=\mathbf{S}_{i}\cdot\nabla_{\mathbf{Z}_{j}}\log(\mathbf{S}_{i})=\mathbf{S}_{i}\cdot(1\{i=j\}\\
            &-\frac{1}{\sum_{l=1}^{N}e^{\mathbf{Z}_{l}}}\cdot(\nabla_{\mathbf{Z}_{l}}\sum_{l=1}^{N}e^{\mathbf{Z}_{l}})) \\
            &=\mathbf{S}_{i}\cdot(1\{i=j\}-\frac{e^{\mathbf{Z}_{j}}}{\sum_{l=1}^{N}e^{\mathbf{Z}_{l}}})=\mathbf{S}_{i}\cdot(1\{i=j\}-\mathbf{S}_{i}).
        \end{split}
    \end{equation}
    For each element, we have:
    \begin{equation}
        [\nabla_{\mathbf{Z}_{ic^{\prime}}}\mathbf{S}_{ic}]_{c^{\prime}=1}^{C}=[\mathbf{S}_{ic}(1\{c=c^{\prime}\}-\mathbf{S}_{ic})]_{c^{\prime}=1}^{C}.
    \end{equation}
    Then, 
    \begin{equation}
        \begin{split}
            [\nabla_{\mathbf{Z}_{ic^{\prime}}}\mathcal{L}_{nll}]_{c^{\prime}=1}^{C}&=-[\sum_{c=1}^{C}\frac{\mathbf{Y}_{ic}}{\mathbf{S}_{ic}}\mathbf{S}_{ic}(1\{c=c^{\prime}\}-\mathbf{S}_{ic})]_{c^{\prime}=1}^{C}\\
        &=-[\sum_{c=1}^{C}\mathbf{Y}_{ic}(1\{c=c^{\prime}\}-\mathbf{S}_{ic})]_{c^{\prime}=1}^{C}\\
        &=[\mathbf{S}_{ic^{\prime}}-\mathbf{Y}_{ic^{\prime}}]_{c^{\prime}=1}^{C}.
        \end{split}
    \end{equation}
    In matrix formation, we have $\nabla_{\mathbf{Z}_{i}}\mathcal{L}_{nll}=\mathbf{S}_{i}-\mathbf{Y}_{i}$. Also, the gradient of $\mathbf{Z}$ w.r.t $\mathbf{W}$ is $\nabla_{\mathbf{W}}\mathbf{Z}=(\hat{\mathbf{A}}^{\tau}\mathbf{X})^{\top}$. Hence, we have:
    \begin{equation}
        \begin{split}
            &d\mathbf{Z}|_{\mathbf{Z}=\mathbf{Z}^{*}}\propto\hat{\mathbf{A}}^{\tau}\mathbf{X}(\hat{\mathbf{A}}^{\tau}\mathbf{X})^{\top}(\mathbf{S}^{*}-\mathbf{Y}),\\
            &d\mathcal{L}_{atk}=-\nabla_{\mathbf{Z}}\mathcal{L}_{nll}d\mathbf{Z}\\
            &\quad\quad\ \ \propto-\Tr((\mathbf{S}^{*}-\mathbf{Y})^{T}\hat{\mathbf{A}}^{\tau}\mathbf{X}(\hat{\mathbf{A}}^{\tau}\mathbf{X})^{\top}(\mathbf{S}^{*}-\mathbf{Y})) \\
            &\quad\quad\ \ =-\Tr((\mathbf{S}^{*}-\mathbf{Y})^{T}\mathbf{D}^{-1}\mathbf{A}^{\tau}\mathbf{X}(\mathbf{D}^{-1}\mathbf{A}^{\tau}\mathbf{X})^{\top}(\mathbf{S}^{*}-\mathbf{Y})) \\
            &\quad\quad\ \ =-\Tr(\Delta\mathbf{A}^{\tau}\mathbf{X}(\mathbf{A}^{\tau}\mathbf{X})^{\top}\Delta^{\top})=-\sum_{i,j=1}^{N}\mathbf{K}_{ij}\delta_{i}\delta_{j}^{\top}, 
        \end{split}
    \end{equation}
    where $\mathbf{K}=\mathbf{A}^{\tau}\mathbf{X}(\mathbf{A}^{\tau}\mathbf{X})^{\top}$ is the kernel matrix based on $\mathbf{A}^{\tau}\mathbf{X}$, $\delta_{i}$ is the $i$-the vector of $\Delta=(\mathbf{S}^{*}-\mathbf{Y})^{\top}\mathbf{D}^{-1}$. 
\end{proof}
Theorem.~\ref{theorem-1} illustrates that the magnitude of the update of the attack loss is inversely correlated with the similarity matrix $\mathbf{K}=\mathbf{A}^{\tau}\mathbf{X}(\mathbf{A}^{\tau}\mathbf{X})^{\top}$. Hence, the graph attacker tends to connect the node pair $(u,v)$ with a low value of the similarity score $\mathbf{K}_{uv}$ to influence the attack loss $\mathcal{L}_{atk}$ as much as possible. 

\begin{table}[h]
	\centering
	\caption{KL Divergence between probability densities of malicious links and benign links.}
	\label{tab-KL-Divergence}
	\resizebox{1.\columnwidth}{!}{%
		\begin{tabular}{c|cccccc}
			\toprule[1.pt]
			Attack&$\Omega(X)$&$\Omega(AX)$&$\Omega(A^2X)$&$\Omega(A^3X)$&$\Omega(A^5X)$&$\Omega(A^{10}X)$\\
			\hline
            Mettack & $0.241$ & $\mathbf{1.670}$ & $\mathbf{1.494}$ & $0.791$ & $0.636$ & $0.451$\\ 
            Minmax & $0.193$ & $\mathbf{0.952}$ & $\mathbf{1.078}$ & $0.860$ & $0.542$ & $0.454$\\
			\bottomrule[1.pt]
		\end{tabular}
	}
\end{table}

\subsection{Exploiting Attack Preferences}
Theorem.~\ref{theorem-1} validates that connecting dissimilar node pairs using the similarity matrix $\mathbf{K}$ leads to a larger magnitude of the update of the attack loss, thereby significantly affecting node classification performance. In this section, we further investigate and exploit the attack preferences of the graph attacker, which can serve as the cornerstone of our proposed defense strategy. Specifically, we define the similarity matrix as:  
\begin{equation}
    \label{eqn-cos-similarity}
    \Omega(\mathbf{A}^{\tau}\mathbf{X})[i,j]=\frac{\mathbf{A}^{\tau}\mathbf{X}[i]^{\top}\cdot\mathbf{A}^{\tau}\mathbf{X}[j]}{\|\mathbf{A}^{\tau}\mathbf{X}[i]\|\cdot\|\mathbf{A}^{\tau}\mathbf{X}[j]\|},
\end{equation}
where $\tau$ is the power of a matrix. It is worth noting that vanilla robust models utilize the ego similarity, i.e., $\Omega(\mathbf{X})$ to shape the attack preferences for graph structural attacks. 
%More specifically, the graph attacker tends to connect dissimilar node pairs based on $\Omega(\mathbf{X})$. 
To observe the attack preferences of the graph attacker in heterophilic graphs, we use Chameleon~\cite{chameleon} as an example. We employ two representative graph structural attacks Mettack~\cite{Mettack} and MinMax~\cite{TopologyAttack} on the clean graph. Subsequently, we report the cosine similarity score~\cite{cossim} based on ego features $\Omega(\mathbf{X})$ and one-hop neighbor features $\Omega(\mathbf{AX})$ (one-hop similarities) and two-hop neighbor features $\Omega(\mathbf{A}^{2}\mathbf{X})$ (two-hop similarities) for both benign links and malicious links in Fig.~\ref{Fig-density-attributes}. It is observed that for a heterophilic graph, ego similarities cannot differentiate malicious links from benign links since their densities are similar. However, there exists a significant difference between the density of one-hop similarities and two-hop similarities for benign links and malicious links and the mean values for malicious links are far less than benign links. Moreover, we want to emphasize that the attack preference based on $\tau$-hop similarity also works for homophilic graphs since Theorem.~\ref{theorem-1} is independent of the homophily ratio of the graph data. We will provide additional experiments in a later section to verify this issue.

%\subsection{Over-smoothing of $\Omega(\mathbf{A}^{\tau}\mathbf{X})$}
The selection of the graph convolutional layer $\tau$ remains a challenge. In particular, Fig.~\ref{Fig-density-attributes-largeK} illustrates the density of the similarity scores for higher-order layers. The observed phenomenon demonstrates that higher-order similarities cannot distinguish the malicious and benign links, as the two densities become mixed. This issue arises due to the problem of over-smoothing~\cite{oversmoothing}. Specifically, repeatedly applying the aggregation operation can blend the attributes of nodes from different clusters, making them indistinguishable. Hence, it is crucial to select an appropriate model depth $\tau$ for crafting an efficient robust model. On the other hand, we also provide a quantitative analysis of the discrimination between the distribution of the similarity scores of malicious links and benign links by utilizing the K-L divergence~\cite{KLD}. The results show that $\tau=1,2$ achieves the largest distances between the distribution of similarity scores of malicious links and benign links. These results are coincided with the visualization of probability density plots in Fig.~\ref{Fig-density-attributes} and \ref{Fig-density-attributes-largeK}. 
\begin{figure}[h]
    \centering
    %\vspace{-0.3cm}
    \begin{subfigure}[b]{0.24\textwidth}
    	\centering
    	\includegraphics[width=\textwidth,height=3.cm]{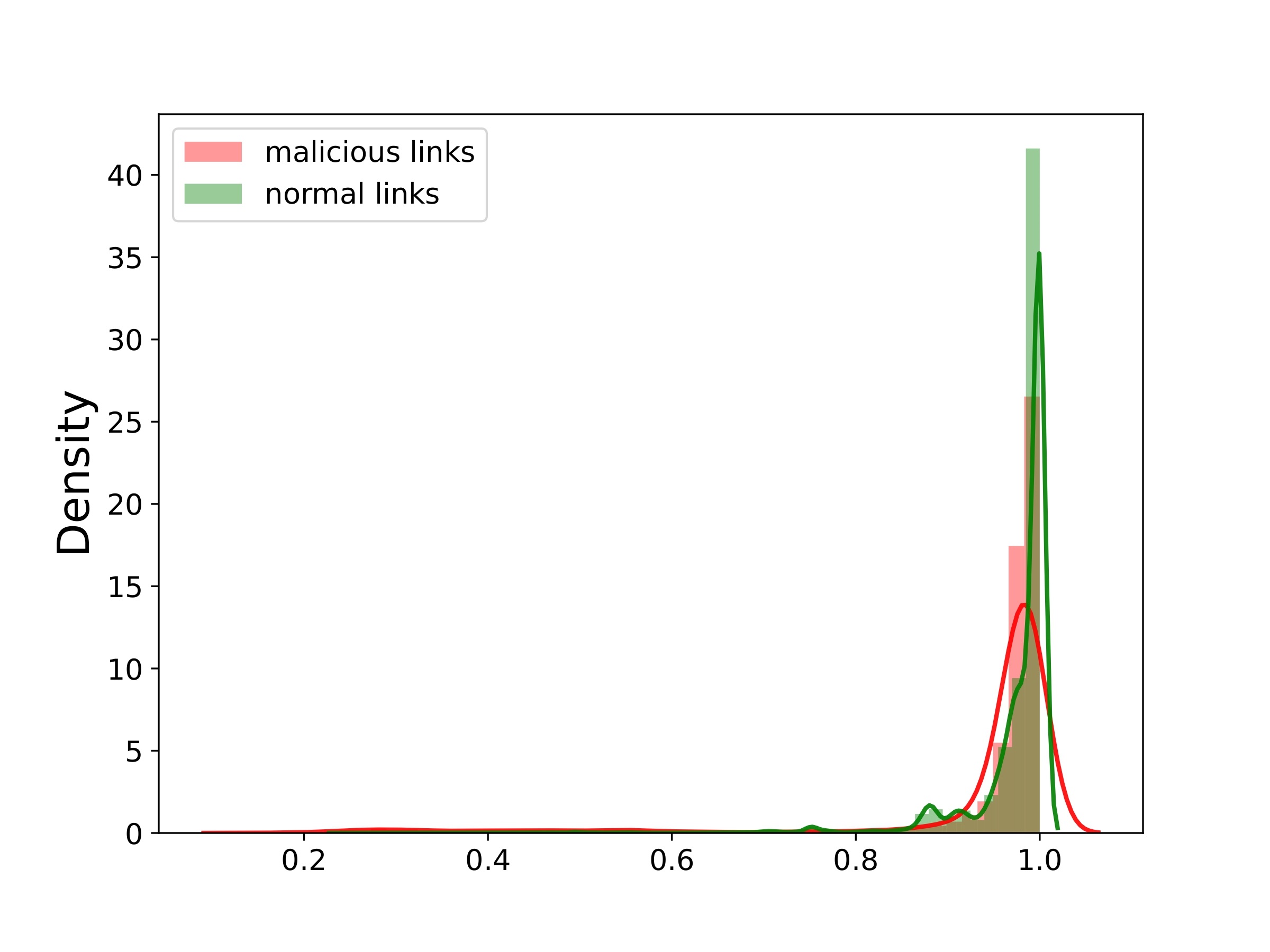}
    	\caption{$\Omega(\mathbf{A}^{5}\mathbf{X})$}
    \end{subfigure}
    \hfill
    \begin{subfigure}[b]{0.24\textwidth}
    \centering
     	\includegraphics[width=\textwidth,height=3.cm]{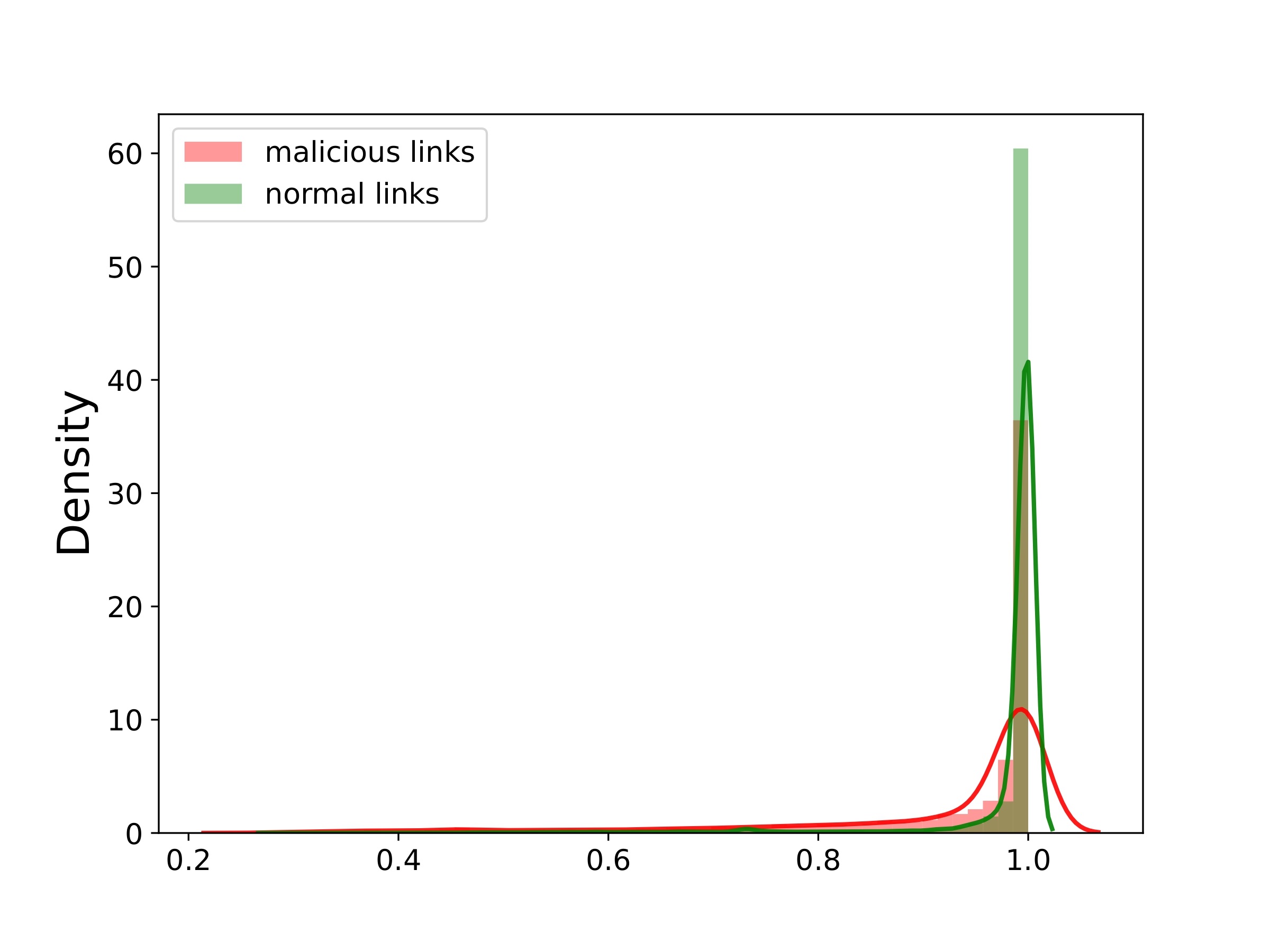}
     	\caption{$\Omega(\mathbf{A}^{10}\mathbf{X})$}
    \end{subfigure}
    \caption{Density plots of $\Omega(\mathbf{A}^{5}\mathbf{X})$ and $\Omega(\mathbf{A}^{10}\mathbf{X})$ between benign links and malicious links for poisoned graph.}
    \label{Fig-density-attributes-largeK}
    %\vspace{-0.45cm}
\end{figure}

\section{Proposed Model: NSPGNN}
%In this section, we propose \underline{N}eighbor \underline{S}imilarity \underline{P}reserving \underline{G}raph \underline{N}eural \underline{N}etwork (\textbf{NSPGNN}) framework to enhance the adversarial robustness of GNN over heterophilic as well as homophilic graphs. 
%The achievement comes from the structural learning based on the neighbor similarity information $\Omega(\mathbf{A}^{\tau}\mathbf{X})$ to mitigate the malicious effect caused by the adversarial perturbations.

\begin{figure*}
    \centering
    \includegraphics[width=1.\textwidth,height=6.cm]{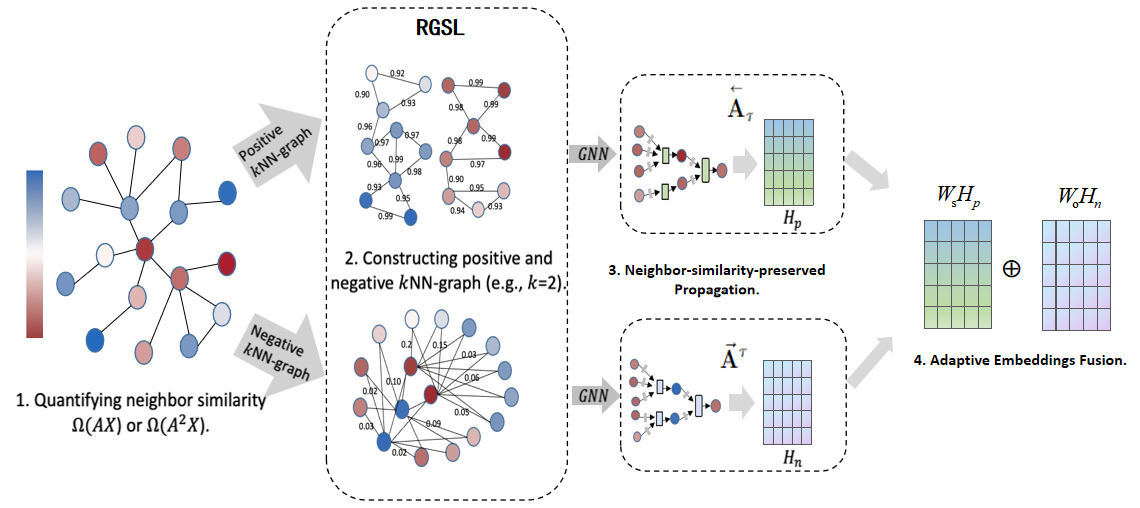}
    %\vspace{-0.3cm}
    \caption{An overview of the proposed framework.}
    \label{fig-NSPGCN}
    %\vspace{-0.45cm}
\end{figure*}
\subsection{Framework Overview}
The overall architecture of the proposed \textbf{NSPGNN} is depicted in Fig.~\ref{fig-NSPGCN}. The goal of \textbf{NSPGNN} is to aggregate the node features while preserving the neighbor's similarity. To this end, \textbf{NSPGNN} contains two modules: dual-kNN graphs construction and neighbor-similarity-preserved propagation. First, in the dual-kNN graphs construction phase, the model generates two positive k-nearest neighbor (kNN) graphs~\cite{kNN-graph} and two negative kNN graphs (to be detailed later) based on the one-hop and two-hop similarities to encode the neighbor similarity information into the graph's topology. Next, in order to incorporate the neighbor similarities information into node representations, we propagate the node features along the structural information of positive kNN graphs with a low-pass filter~\cite{SGC} to smooth the representations of connected nodes with high-level neighbor similarities. On the other hand, we also consider the dissimilar information to serve as the contrast to the similar information by utilizing the high-pass filter to enhance the discrimination of the connected nodes with low-level neighbor similarities. As a result, the model can capture both similar information and dissimilar information to double-preserve the neighbor similarities.  

%First, \textbf{NSPGNN} constructs two positive $k$-nearest neighbor (kNN) graphs~\cite{kNN-graph} $\mathbf{A}^{kNN_{1}}$ and $\mathbf{A}^{kNN_{2}}$ and two negative kNN graphs (to be detailed later) $\mathbf{A}^{nkNN_{1}}$ and $\mathbf{A}^{nkNN_{2}}$ based on the one-hop and two-hop aggregated feature matrix $\mathbf{AX}$ and $\mathbf{A}^{2}\mathbf{X}$. Then, the node features are propagated adaptively from $\mathbf{A}^{kNN_{1}}$, $\mathbf{A}^{kNN_{2}}$, $\mathbf{A}^{nkNN_{1}}$ and $\mathbf{A}^{nkNN_{2}}$ with the learnable weights to balance the relative importance. Especially, the node features are aggregated with the low-pass filter~\cite{SGC} ($\hat{\mathbf{A}}^{kNN_{1}}$ and $\hat{\mathbf{A}}^{kNN_{2}}$) to push the embeddings of node pairs with high similarity score similar and aggregated with the high-pass filter~\cite{ACMGNN} ($\mathbf{I}-\hat{\mathbf{A}}^{nkNN_{1}}$ and $\mathbf{I}-\hat{\mathbf{A}}^{nkNN_{2}}$) simultaneously to push the embeddings of node pairs with low similarity score dissimilar.          

\subsection{Dual-kNN Graphs Construction}
As previously mentioned in Sec.~\ref{sec-vulnerability}, it has been theoretically and empirically illustrated that the vulnerability of the GNN framework is highly dependent on neighbor similarities. 
To achieve universal robustness of the GNN framework over homophilic and heterophilic graphs, we craft a robust GNN model and preserve the neighbor similarities by encoding the similarities information into the graph's topology. 
%To universally enhance the adversarial robustness of the GNN framework over homophilic and heterophilic graphs, we endeavor to craft a robust GNN framework guided by the neighbor similarity information to get rid of the malicious effects in the poisoned graphs. 
To this end, we construct positive kNN graphs to capture the high-level neighbor similarities information and negative kNN graphs to serve as their counterparts. 
%As a result, propagating along the two distinct structure information (high similarity v.s. low similarity) can contrast with each other and double-preserve the neighbor similarity. 

\subsubsection{Positive kNN Graphs}
To encode the neighbor similarities into the GNN framework, we first endeavor to capture the high-level similarities information of the $\tau$-hop similarity matrix by constructing positive kNN graphs. 
%The crafted structural information can serve as the positive path to smooth the features of node pairs with high similarity scores. 
To this end, we first transform the features $\mathbf{A}^{\tau}\mathbf{X}$ into a topology by generating a kNN graph based on the $\tau$-hop similarity matrix $\Omega(\mathbf{A}^{\tau}\mathbf{X})$ (presented in Eqn.~\ref{eqn-cos-similarity}). Subsequently, we pick out the top-$k$ nearest neighbors for each node to construct the corresponding kNN graph, i.e.,
\begin{equation}
    \overset{\shortleftarrow}{\mathbf{A}}_{\tau}[u]\in \text{Argsort\_Desc}_{v\in\mathcal{V}\setminus u}\frac{\mathbf{A}^{\tau}\mathbf{X}[u]^{\top}\cdot\mathbf{A}^{\tau}\mathbf{X}[v]}{\|\mathbf{A}^{\tau}\mathbf{X}[u]\|\cdot\|\mathbf{A}^{\tau}\mathbf{X}[v]\|}, 
\end{equation}
where $\text{Argsort\_Desc}(\cdot)$ represents picking out indices of samples with descending order. Consequently, the target node in the positive graph $\overset{\shortleftarrow}{\mathbf{A}}_{\tau}$ will connect with other similar nodes based on the $\tau$-hop neighbor similarities.

\subsubsection{Negative kNN Graphs}
The limitation of the positive kNN graphs is that they only contain information on high-level similarities and omit the dissimilarity information between node pairs.  To tackle this problem, we contrastively introduce negative kNN graphs $\vec{\mathbf{A}}^{\tau}$ to encode the low-level neighbor similarities into the structure information to supervise propagation:
\begin{equation}
    \vec{\mathbf{A}}^{\tau}[u]\in \text{Argsort\_Asc}_{v\in\mathcal{V}\setminus u}\frac{\mathbf{A}^{\tau}\mathbf{X}[u]^{\top}\cdot\mathbf{A}^{\tau}\mathbf{X}[v]}{\|\mathbf{A}^{\tau}\mathbf{X}[u]\|\cdot\|\mathbf{A}^{\tau}\mathbf{X}[v]\|}, 
\end{equation}
where $\text{Argsort\_Asc}(\cdot)$ represents picking out indices of samples with ascending order. In this way, the negative graphs $\vec{\mathbf{A}}^{\tau}$ can capture the extremely dissimilar information and serve as the negative samples to preserve neighbor similarities in an opposite perspective.  

\subsection{neighbor-similarity-preserved Propagation}
After the dual-kNN graphs construction phase, it is important to determine the best choices for the aggregation mechanism to propagate node features while preserving the neighbor similarities effectively. Toward this end, we introduce an adaptive neighbor-similarity-preserved propagation mechanism. 
%Without loss of generality and restricting the computational complexity of the proposed method, we determine to encode the one-hop and two-hop neighbor similarity information to guide the propagation of the graph convolutional layer, i.e., we choose $\tau=1,2$ for positive and negative kNN graphs. 
The formal form of the information flow for the node representation learning is:
\begin{subequations}
    \label{eqn-propagation}
    \begin{align}
        & \label{alpha comp} [\alpha_{1}^{(l)},\alpha_{2}^{(l)}]=\sigma(\mathbf{H}^{(l-1)}\mathbf{W}^{(l)}_{m}+\mathbf{b}^{(l)}_{m}),\\
        & \label{beta comp}[\beta_{1}^{(l)},\beta_{2}^{(l)}]=\sigma(\mathbf{H}^{(l-1)}\mathbf{W}^{(l)}_{n}+\mathbf{b}^{(l)}_{n}),\\
        & \label{AkNN comp}
        \overset{\shortleftarrow}{\mathbf{\mathcal{\hat{A}}}}_{\tau}=\overset{\shortleftarrow}{\mathbf{\mathcal{D}}}_{\tau}^{-\frac{1}{2}}\overset{\shortleftarrow}{\mathbf{\mathcal{A}}}_{\tau}\overset{\shortleftarrow}{\mathbf{\mathcal{D}}}_{\tau}^{-\frac{1}{2}}, \vec{\mathbf{\mathcal{\hat{A}}}}_{\tau}=\vec{\mathbf{\mathcal{D}}}_{\tau}^{-\frac{1}{2}}\vec{\mathbf{\mathcal{A}}}_{\tau}\vec{\mathbf{\mathcal{D}}}_{\tau}^{-\frac{1}{2}}, \\
        & \label{Ap comp} \overset{\shortleftarrow}{\mathbf{\mathcal{A}}}=\alpha_{1}^{(l)}\odot\overset{\shortleftarrow}{\mathbf{\mathcal{\hat{A}}}}_{1}+\alpha_{2}^{(l)}\odot\overset{\shortleftarrow}{\mathbf{\mathcal{\hat{A}}}}_{2},\\
        & \label{An comp} 
        \vec{\mathbf{\mathcal{A}}}=\alpha_{1}^{(l)}\odot\vec{\mathbf{\mathcal{\hat{A}}}}_{1}+\alpha_{2}^{(l)}\odot\vec{\mathbf{\mathcal{\hat{A}}}}_{2},\\
        & \label{H comp} 
        \mathbf{H}^{(l)}=\sigma(\mathbf{H}^{(l-1)}\mathbf{W}_{s}^{(l)}+
        \overset{\shortleftarrow}{\mathbf{\mathcal{A}}}\mathbf{H}^{(l-1)}\mathbf{W}_{o}^{(l)}\\
        &\quad\quad+
        (\mathbf{I}-\vec{\mathbf{\mathcal{A}}})\mathbf{H}^{(l-1)}\mathbf{W}_{d}^{(l)})
        ,
    \end{align}
\end{subequations}
where $\overset{\shortleftarrow}{\mathbf{\mathcal{A}}}_{\tau}=\overset{\shortleftarrow}{\mathbf{A}}_{\tau}+\mathbf{I}$ and $\vec{\mathbf{\mathcal{A}}}_{\tau}=\vec{\mathbf{A}}_{\tau}+\mathbf{I}$, $\tau=1,2$, $\mathbf{H}^{(0)}=\mathbf{X}$, $\sigma(\cdot)$ is the activation function such as ReLU~\cite{ReLU}. 

In this process, we train the learnable weight matrix $\alpha_{1}^{(l)}$, $\alpha_{2}^{(l)}$, $\beta_{1}^{(l)}$ and $\beta_{2}^{(l)}$ at each layer by implementing a multi-layer perceptron (MLP~\cite{MLP}) on the nodal feature matrix $\mathbf{H}^{(l-1)}$ to adaptively balance the relative importance between the one-hop and two-hop neighbor similarities for positive kNN graphs and negative kNN graphs. Next, we propagate the node features along the positive hybrid structures $\overset{\shortleftarrow}{\mathbf{\mathcal{A}}}$ with a low-pass filter to smooth the features of similar node pairs. Alternatively, the node features are aggregated along the negative hybrid structure $\vec{\mathbf{\mathcal{A}}}$ with a high-pass filter to enhance the discrimination of the features of dissimilar node pairs. In the meanwhile, we assign different weight matrices for ego-embeddings $\mathbf{H}^{(l-1)}$ and neighbor-embeddings obtained from low-pass and high-pass filters to let the model determine the relative importance of self-loops for graph representation learning. It is worth noting that $\odot$ denotes the operation to multiply the $i$-th element vector with the $i$-th row of a matrix. Finally, the three parties (ego-embeddings $\mathbf{H}^{(l-1)}\mathbf{W}_{s}^{(l)}$, low-pass embeddings $\overset{\shortleftarrow}{\mathbf{\mathcal{A}}}\mathbf{H}^{(l-1)}\mathbf{W}_{o}^{(l)}$ and high-pass embeddings $(\mathbf{I}-\vec{\mathbf{\mathcal{A}}})\mathbf{H}^{(l-1)}\mathbf{W}_{d}^{(l)}$) together determine the final node representations. Alg.~\ref{alg-NSPGNN} presents the algorithm details of \textbf{NSPGNN}, which contains two modules: dual-kNN graph constructions and neighbor-similarity-preserved propagation.
\begin{algorithm}  
\caption{\textbf{NSPGNN}.}  
\label{alg-NSPGNN}
\begin{algorithmic}[1]  
\REQUIRE Input graph adjacency matrix $\mathbf{A}$ and its attributes $\mathbf{X}$, hyperparameters $k_{1}$ and $k_{2}$, training labels $\mathbf{Y}_{tr}$, node-set $\mathcal{V}$, learning rate $\eta$, training epoch $T$, number of hidden layers $L$, parameters set $\Theta$.
\STATE{\textbf{Dual-kNN Graph Contructions}:}
\STATE{Compute similarity matrix: \\$\Omega(\mathbf{A}^{\tau}\mathbf{X})[u,v]=\frac{\mathbf{A}^{\tau}\mathbf{X}[u]^{\top}\cdot\mathbf{A}^{\tau}\mathbf{X}[v]}{\|\mathbf{A}^{\tau}\mathbf{X}[u]\|\cdot\|\mathbf{A}^{\tau}\mathbf{X}[v]\|}$},
\STATE{Obtain positive kNN graphs: \\ $\overset{\shortleftarrow}{\mathbf{A}}_{\tau}[u]\in \text{Argsort\_Desc}_{v\in\mathcal{V}\setminus u}\ \Omega(\mathbf{A}^{\tau}\mathbf{X})[u,v]$,}
\STATE{Obtain negative kNN graphs: \\ $\vec{\mathbf{A}}_{\tau}[u]\in \text{Argsort\_Asc}_{v\in\mathcal{V}\setminus u}\ \Omega(\mathbf{A}^{\tau}\mathbf{X})[u,v]$.}
\STATE{\textbf{neighbor-similarity-preserved Propagation}:}
\FOR{$t\leq T$}
\STATE{Initialize $\mathbf{H}^{(0)}=\mathbf{X}$,}
\FOR{$l=1,2,...,L$}
\STATE{Obtain learnable weights: \\ $[\alpha_{1}^{(l)},\alpha_{2}^{(l)}]=\sigma(\mathbf{H}^{(l-1)}\mathbf{W}^{(l)}_{m}+\mathbf{b}^{(l)}_{m})$, \\ $[\beta_{1}^{(l)},\beta_{2}^{(l)}]=\sigma(\mathbf{H}^{(l-1)}\mathbf{W}^{(l)}_{n}+\mathbf{b}^{(l)}_{n})$;}
\STATE{Obtain positive and negative propagation matrix:\\ $\overset{\shortleftarrow}{\mathbf{\mathcal{A}}}=\alpha_{1}^{(l)}\odot\overset{\shortleftarrow}{\mathbf{\mathcal{\hat{A}}}}_{1}+\alpha_{2}^{(l)}\odot\overset{\shortleftarrow}{\mathbf{\mathcal{\hat{A}}}}_{2}$, \\ $\vec{\mathbf{\mathcal{A}}}=\alpha_{1}^{(l)}\odot\vec{\mathbf{\mathcal{\hat{A}}}}_{1}+\alpha_{2}^{(l)}\odot\vec{\mathbf{\mathcal{\hat{A}}}}_{2}$},
\STATE{Obtain node embeddings via guided propagation: $\mathbf{H}^{(l)}=\sigma(\mathbf{H}^{(l-1)}\mathbf{W}_{s}^{(l)}+
        \overset{\shortleftarrow}{\mathbf{\mathcal{A}}}\mathbf{H}^{(l-1)}\mathbf{W}_{o}^{(l)}+
        (\mathbf{I}-\vec{\mathbf{\mathcal{A}}})\mathbf{H}^{(l-1)}\mathbf{W}_{d}^{(l)})$.}
\STATE{Gradient descent: $\Theta^{t}\leftarrow \Theta^{t-1}-\eta\frac{\partial\mathcal{L}_{nll}(\mathbf{H}^{(L)},\mathbf{Y}_{tr})}{\partial\Theta}$.}
\ENDFOR
\ENDFOR
\RETURN{\textbf{NSPGNN} node embeddings $\mathbf{H}^{(L)}$.}
\end{algorithmic}  
\end{algorithm}

\subsection{Time Complexity Analysis}
The additional time complexity of \textbf{NSPGNN} compared to the vanilla GNN is derived from the dual-kNN graph constructions phase. Its time complexity is $\mathcal{O}(N^2)$ for positive and negative kNN graph constructions. However, we can speed up computing the pairwise similarity matrix via Recursive Lanczos Bisection~\cite{RLBisection} or MapReduce~\cite{MapReduce} which can reduce the time complexity to $\mathcal{O}(N^{1.14})$. 

%The time complexity of the proposed framework is determined by two main modules: the construction of Dual-kNN Graphs and neighbor-similarity-preserved Propagation. It is important to note that the former can be completed during preprocessing. Assuming the dimensions of $\mathbf{H}^{(l)}$ and $\mathbf{H}^{(l-1)}$ are $d^{(l)}$ and $d^{(l-1)}$ respectively, the time complexity of both Eq.~\eqref{alpha comp} and Eq.~\eqref{beta comp} is $O(d^{(l-1)}n^{2})$, where $n$ represents the number of nodes. Additionally, the time complexity of Eq.~\eqref{AkNN comp} and Eq.~\eqref{AnKNN comp} is $O(n^3)$, while the time complexity of Eq.~\eqref{Ap comp} and Eq.~\eqref{An comp} is $O(2n^2)$. Furthermore, Eq.~\eqref{H comp} yields a time complexity of $O(n(d^{(l)}+n)d^{(l-1)})$. Note Eq.~\eqref{alpha comp} and Eq.~\eqref{beta comp}, as well as Eq.~\eqref{AkNN comp} and Eq.~\eqref{AnKNN comp}, and Eq.~\eqref{Ap comp} and Eq.~\eqref{An comp} can be computed in parallel. Moreover, since $d^{(l-1)}$ and $d^{(l)}$ are consistently 128, 64, or 32, which are significantly smaller than $n$, thus can be disregarded. Consequently, the overall time complexity of neighbor-similarity-preserved Propagation is $O(n^{3}+n^{2})$.

\section{Experiments}
In this section, we evaluate our proposed framework and aim to answer the following research question (RQ):
\begin{itemize}
    \item \textbf{RQ1}: What are the performances of \textbf{NSPGNN} compared to baselines on clean homophilic and heterophilic graphs?
    \item \textbf{RQ2}: What are the adversarial robustness of \textbf{NSPGNN} compared to baselines on poisoned homophilic and heterophilic graphs?
    \item \textbf{RQ3}: What are the influences of varying hyperparameters and different components?
\end{itemize}

\begin{table}[h]
	\centering
	\caption{Dataset statistics.}
	\label{tab-dataset}
	\resizebox{1.\columnwidth}{!}{%
		\begin{tabular}{c|ccccc}
			\toprule[1.pt]
			Datasets &$N$&$|\mathcal{E}|$&Classes&Features&$\mathcal{H}(G)$\\
			\hline
            Citeseer & $2110$ & $3668$ & $6$ & $3703$ & $0.81$\\
           Chameleon & $2277$ & $31371$ & $5$ & $2325$ & $0.23$\\ 
           Cora & $2485$ & $5069$ & $7$ & $1433$ & $0.74$\\ 
            Squirrel & $5201$ & $198353$ & $5$ & $2089$ & $0.22$\\
            Photo & $7650$ & $119081$ & $8$ & $745$ & $0.60$ \\
           Crocodile & $11631$ & $170773$ & $5$ & $128$ & $0.23$\\
           Tolokers & $11758$ & $519000$ & $2$ & $10$ & $0.59$\\
			\bottomrule[1.pt]
		\end{tabular}
	}
\end{table}

\begin{table*}[h]
	\centering
	\caption{Robust performances of heterophilic GNNs over \textbf{heterophilic graphs} against \textit{Mettack}.}
	\label{tab-exp-Mettack}
        %\vspace{-0.35cm}
	\resizebox{1.\textwidth}{!}{%
        \begin{tabular}{c|c|ccccccc|cc}
        \toprule 
        Dataset & $\delta_{atk}$ &  GPRGNN &  FAGNN &  H2GCN &  GBKGNN &  BMGCN &  ACMGNN &  GARNET &  NSPGNN w.o. & NSPGNN\\
        \hline
        \multirow{6}*{Chameleon}& 1\% & 64.04 (0.98) & 67.28 (0.54) & 59.41 (0.95) & 64.43 (1.00) & 65.79 (0.95) & 64.32 (1.25) & 64.63 (1.02) & 70.59 (0.99) & \textbf{70.88 (1.00)}\\
        & 5\% & 62.87 (0.96) & 62.76 (0.62) & 57.06 (0.91) & 57.00 (0.92) & 62.17 (0.85) & 60.44 (0.81) & 60.88 (1.01) & \textbf{70.02 (1.01)} & 69.39 (0.74)\\
        & 10\% & 59.34 (0.96) & 56.56 (0.87) & 54.43 (0.80) & 54.34 (1.12) & 59.82 (1.06) & 58.62 (1.29) & 59.04 (0.93) & 67.54 (0.88) & \textbf{68.99 (0.87)}\\
        & 15\% & 56.91 (1.29) & 55.22 (1.06) & 54.04 (1.00) & 51.10 (1.34) & 56.62 (0.86) & 56.56 (1.10) & 56.54 (1.19) & 64.78 (0.74) & \textbf{65.31 (1.17)}\\
        & 20\% & 54.21 (0.73) & 52.46 (0.66) & 53.62 (1.58) & 49.43 (1.08) & 55.66 (0.91) & 55.15 (1.43) & 55.09 (1.00) & 63.14 (0.79) & \textbf{63.53 (0.89)}\\
        & 25\% & 52.32 (0.66) & 50.48 (0.59) & 52.50 (1.01) & 47.41 (0.90) & 54.69 (0.94) & 54.36 (1.09) & 54.84 (1.12) & 60.46 (0.77) & \textbf{60.53 (0.62)}\\
        \hline
        \multirow{6}*{Squirrel}& 1\% & 40.75 (2.00) & 46.83 (3.19) & 33.28 (1.00) & 50.56 (1.19) & 44.61 (0.94) & 48.06 (1.06) & 45.87 (1.20) & 55.80 (0.93) & \textbf{57.21 (0.92)} \\
        & 5\% & 39.42 (1.20) & 41.99 (1.76) & 33.06 (1.13) & 46.11 (0.92) & 42.33 (0.77) & 46.41 (1.62) & 44.07 (1.09) & 52.91 (0.79) & \textbf{56.02 (0.54)} \\
        & 10\% & 35.28 (1.21) & 38.13 (1.61) & 33.64 (1.05) & 40.98 (2.61) & 40.89 (0.81) & 42.78 (0.55) & 41.46 (1.11) & 51.83 (0.63) & \textbf{53.85 (1.10)} \\
        & 15\% & 32.90 (0.83) & 36.25 (1.00) & 32.83 (0.86) & 37.71 (1.57) & 39.28 (0.38) & 40.74 (1.31) & 40.25 (1.25) & 51.24 (0.47) & \textbf{53.53 (0.92)} \\
        & 20\% & 31.27 (0.65) & 34.20 (1.81) & 32.26 (1.27) & 34.71 (2.06) & 38.27 (0.71) & 38.77 (1.17) & 38.12 (1.03) & 49.98 (0.78) & \textbf{51.20 (0.72)} \\
        & 25\% & 30.18 (1.26) & 32.86 (2.58) & 32.65 (1.20) & 32.79 (1.21) & 36.77 (0.69) & 37.45 (1.40) & 37.38 (0.84) & 48.60 (0.81) & \textbf{49.98 (0.80)} \\
        \bottomrule
        
        \end{tabular}
	}
\end{table*}

\begin{table*}[h]
	\centering
        %\vspace{-0.1cm}
	\caption{Robust performances of heterophilic GNNs over \textbf{heterophilic graphs} against \textit{Minmax}.}
	%\vspace{-0.35cm}
        \label{tab-exp-Minmax}
	\resizebox{1.\textwidth}{!}{%
        \begin{tabular}{c|c|ccccccc|cc}
        \toprule 
        Dataset & $\delta_{atk}$ &  GPRGNN &  FAGNN &  H2GCN &  GBKGNN &  BMGCN &  ACMGNN &  GARNET &  NSPGNN w.o. & NSPGNN \\
        \hline
        \multirow{7}*{Chameleon}& 1\% & 61.47 (0.93) & 61.40 (2.46) & 58.75 (1.15) & 64.01 (0.56) & 63.47 (0.87) & 64.30 (0.92) & 61.75 (0.93) & 67.71 (0.68) & \textbf{67.76 (1.00)} \\
        & 5\% & 51.71 (0.65) & 50.72 (1.11) & 51.40 (1.05) & 52.06 (0.57) & 52.59 (0.73) & 53.05 (1.60) & 57.92 (0.98) & 58.31 (0.67) & \textbf{60.29 (0.58)} \\
        & 10\% & 46.29 (0.61) & 43.77 (1.18) & 47.32 (1.01) & 44.85 (0.74) & 47.24 (0.96) & 48.07 (0.74) & 51.56 (0.89) & \textbf{55.64 (0.82)} & 55.11 (0.78) \\
        & 15\% & 43.20 (0.89) & 39.41 (1.54) & 43.88 (1.44) & 39.71 (0.65) & 44.10 (0.97) & 44.39 (1.19) & 46.69 (1.07) & 52.01 (1.09) & \textbf{52.39 (1.06)} \\
        & 20\% & 39.67 (0.78) & 35.39 (1.07) & 41.16 (1.64) & 33.25 (0.62) & 43.38 (0.99) & 44.43 (1.52) & 48.89 (0.78) & \textbf{51.40 (1.76)} & 51.34 (1.04) \\
        & 25\% & 38.38 (0.51) & 34.45 (1.40) & 41.38 (1.43) & 32.94 (0.39) & 43.82 (1.11) & 43.55 (2.11) & 47.79 (1.29) & 51.56 (1.16) & \textbf{53.05 (1.45)} \\
        \hline
        \multirow{7}*{Squirrel}& 1\% & 42.46 (0.65) & 43.22 (0.61) & 29.49 (2.23) & 48.61 (0.63) & 40.65 (0.66) & 44.10 (1.23) & 42.60 (0.81) & 55.69 (1.00) & \textbf{57.52 (0.68)} \\
        & 5\% & 34.52 (1.20) & 33.50 (0.98) & 28.17 (2.98) & 38.70 (0.59) & 33.14 (0.49) & 36.09 (0.58) & 36.50 (0.80) & 45.38 (0.62) & \textbf{48.32 (0.86)} \\
        & 10\% & 32.92 (1.24) & 30.21 (1.05) & 29.06 (1.30) & 33.15 (0.81) & 31.98 (0.73) & 33.80 (0.93) & 34.69 (0.34) & 43.51 (0.96) & \textbf{46.37 (0.69)} \\
        & 15\% & 32.41 (1.19) & 27.80 (0.94) & 27.93 (2.26) & 30.30 (0.54) & 32.04 (0.57) & 34.47 (0.82) & 34.89 (1.08) & 40.76 (0.67) & \textbf{43.77 (0.62)} \\
        & 20\% & 31.05 (2.34) & 25.14 (1.05) & 28.61 (1.61) & 26.95 (0.85) & 30.71 (0.74) & 32.73 (1.41) & 33.54 (0.57) & 40.18 (0.63) & \textbf{42.31 (0.49)} \\
        & 25\% & 30.51 (0.92) & 23.34 (1.09) & 28.87 (1.63) & 24.33 (1.09) & 31.78 (0.41) & 32.67 (1.04) & 33.19 (0.98) & 39.63 (0.92) & \textbf{41.53 (1.08)} \\
        \hline
        \multirow{7}*{Crocodile}& 1\% & 57.50 (0.28) & 61.13 (0.51) & 54.40 (2.48) & 62.32 (0.58) & 63.63 (0.28) & 62.97 (0.76) & 64.50 (0.27) & \textbf{68.75 (0.46)} & 68.72 (0.45) \\
        & 5\% & 51.50 (0.37) & 55.57 (0.40) & 53.46 (2.32) & 53.58 (0.29) & 58.92 (1.31) & 57.76 (1.44) & 62.51 (0.62) & 64.76 (0.53) & \textbf{64.78 (0.45)} \\
        & 10\% & 47.96 (0.38) & 52.72 (0.49) & 46.11 (2.40) & 49.71 (0.50) & 55.71 (0.47) & 57.82 (1.73) & 59.52 (0.59) & 63.28 (0.58) & \textbf{64.20 (0.40)} \\
        & 15\% & 46.54 (0.38) & 50.59 (0.47) & 43.01 (0.27) & 47.46 (0.66) & 53.68 (0.47) & 55.84 (0.89) & 58.52 (1.07) & 59.32 (0.56) & \textbf{61.58 (0.54)} \\
        & 20\% & 45.66 (0.58) & 49.99 (0.42) & 42.88 (0.25) & 46.17 (0.67) & 52.99 (0.44) & 56.27 (1.25) & 60.29 (0.63) & 60.32 (0.79) & \textbf{61.75 (0.86)} \\
        & 25\% & 45.99 (0.56) & 50.24 (0.32) & 42.71 (0.32) & 45.90 (0.66) & 53.84 (0.74) & 55.37 (1.56) & \textbf{60.76 (0.42)} & 59.11 (0.82) & 59.36 (0.57) \\
        \hline
        \multirow{7}*{Tolokers}& 1\% & 0.66 (0.017) & 0.70 (0.011) & 0.75 (0.015) & 0.69 (0.011) & 0.70 (0.012) & 0.66 (0.005) & 0.73 (0.016) & 0.75 (0.014) & \textbf{0.76 (0.008)} \\
        & 5\% & 0.65 (0.013) & 0.68 (0.011) & 0.72 (0.015) & 0.67 (0.013) & 0.69 (0.011) & 0.66 (0.013) & 0.69 (0.010) & \textbf{0.74 (0.014)} & \textbf{0.74 (0.010)} \\
        & 10\% & 0.65 (0.005) & 0.64 (0.018) & 0.71 (0.013) & 0.66 (0.012) & 0.68 (0.010) & 0.62 (0.014) & 0.70 (0.015) & \textbf{0.74 (0.013)} & 0.73 (0.010) \\
        & 15\% & 0.64 (0.010) & 0.65 (0.013) & 0.70 (0.018) & 0.66 (0.009) & 0.67 (0.014) & 0.61 (0.013) & 0.69 (0.010) & \textbf{0.74 (0.015)} & \textbf{0.74 (0.013)} \\
        & 20\% & 0.64 (0.015) & 0.65 (0.012) & 0.69 (0.019) & 0.65 (0.010) & 0.67 (0.013) & 0.59 (0.011) & 0.69 (0.014) & \textbf{0.73 (0.018)} & \textbf{0.73 (0.015)} \\
        & 25\% & 0.65 (0.010) & 0.65 (0.014) & 0.69 (0.011) & 0.65 (0.016) & 0.68 (0.012) & 0.62 (0.008) & 0.69 (0.010) & \textbf{0.73 (0.011)} & \textbf{0.73 (0.014)} \\
        \bottomrule
        \end{tabular}
	}
\end{table*}

\subsection{Experiment Settings}
\subsubsection{Dataset}
In this paper, we mainly investigate the adversarial robustness of GNN framework against the graph adversarial attacks. Thus, we conduct experiments on four typical heterophilic graphs: Chameleon, Squirrel~\cite{chameleon}, Crocodile and Tolokers~\cite{tolokers}\footnote{Mettack fails to attack Crocodile and Tolokers due to cuda out of memory.} and three homophilic graphs: Cora, CiteSeer~\cite{Cora} and Photo~\cite{Photo}. The statistics of the datasets are shown in Tab.~\ref{tab-dataset}. 

\subsubsection{Setup}
We use \textit{Pytorch-Geometric}~\cite{PytorchGeometric} to preprocess the six graph datasets and implement robust models, i.e., GCN-Jaccard~\cite{GCNJaccard}, ProGNN~\cite{ProGNN} on \textit{DeepRobust}~\cite{deeprobust}, and GNNGUARD~\cite{GNNGUARD}, RGCN~\cite{RGCN}, AirGNN~\cite{AirGNN} and ElasticGNN~\cite{ElasticGNN} on GreatX~\cite{GreatX}. We implement GNN variants crafted for heterophilic graphs (HGNNs) like GPRGNN~\cite{GPRGNN}, FAGNN~\cite{FAGNN}, GBKGNN~\cite{GBKGNN}, BMGCN~\cite{BMGNN} and ACMGNN~\cite{ACMGNN}, GARNET~\cite{garnet} based on their source code. It is worth noting that all the baselines contain up-to-date competitive robust GNNs and HGNNs (For more descriptions please refer to Sec.~\ref{sec-baseline}). We evaluate the robustness of GNNs over homophilic and heterophilic graphs on the semi-supervised node classification task and conduct $10$ individual experiments with varying seeds and report the mean and standard error of the test accuracies for fair comparisons. We consider two typical graph adversarial attacks: Mettack~\cite{Mettack} and Minmax~\cite{TopologyAttack} with the attacking power $\delta_{atk}=\{1\%,5\%,10\%,15\%,20\%,25\%\}$, which represents the proportion of the modified links over the link set $\mathcal{E}$. We tune the hyperparameters $k_{1}$ and $k_{2}$ from the set $\{1,5,10,15,20,25,30\}$ and determine the best choices based on the validation accuracy. We train $500$ epochs in all experiments using the Adam~\cite{Adam} optimizer with a learning rate of $0.01$ for all the models. 

\subsection{Baselines}
\label{sec-baseline}
The baselines include the state-of-the-art GNN variants for heterophilic graphs and robust GNN models. 
\begin{itemize}
    \item \textbf{GCN}~\cite{GCN} is the most representative GNN model which utilizes the graph convolutional layer to propagate node features with a low-pass filter.
\end{itemize}
The following are GNNs under heterophily:
\begin{itemize}
    \item \textbf{GPRGNN}~\cite{GPRGNN} adaptively learns the generalized PageRank weights to optimize nodal feature and topological information extraction jointly. 
    \item \textbf{FAGNN}~\cite{FAGNN} adaptively integrates low-pass and high-pass signals in the message-passing mechanism to learn graph representations for homophilic and heterophilic graphs. 
    \item \textbf{H2GCN}~\cite{H2GCN} identifies ego- and neighbor-embedding separation, higher-order neighborhoods, and a combination of intermediate representations to boost learning from the graph structure under heterophily. 
    \item \textbf{GBKGNN}~\cite{GBKGNN} proposes a bi-kernel feature transformation and a selection gate to capture homophily and heterophily information respectively.
    \item \textbf{BMGCN}~\cite{BMGNN} introduces block modeling into the framework of GNN to automatically learn the corresponding aggregation rules for neighbors of different classes. 
    \item \textbf{ACMGNN}~\cite{ACMGNN} proposes the adaptive channel mixing framework to adaptively exploit aggregation, diversification and identity channels node-wisely to extract richer localized information for diverse node heterophily situations.
    \item \textbf{GARNET}~\cite{garnet} is a scalable spectral method that leverages weighted spectral embedding to construct a base graph, and then refines the base graph by pruning additional uncritical edges based on a probabilistic graphical model. 
\end{itemize}
The following are robust GNN models:
\begin{itemize}
    \item \textbf{GCN-Jaccard}~\cite{GCNJaccard} preprocesses the graph data by pruning links that connect nodes with low values of Jaccard similarity of node attributes. 
    \item \textbf{ProGNN}~\cite{ProGNN} jointly learns a structural graph and a robust GNN model from the poisoned graph guided by the three properties: low-rank, sparsity and feature smoothness.
    \item \textbf{GNNGUARD}~\cite{GNNGUARD} detects and quantifies the relationship between the graph structure and node features to assign higher weights to edges connecting similar nodes while pruning edges between unrelated nodes during training.
    \item \textbf{RGCN}~\cite{RGCN} learns Gaussian distributions for each node feature and employs an attention mechanism to penalize nodes with high variance.   
    \item \textbf{AirGNN}~\cite{AirGNN} proposes an adaptive message passing scheme to learn a GNN framework with adaptive residual to tackle the trade-off between abnormal and normal features during GNN training.
    \item \textbf{ElasticGNN}~\cite{ElasticGNN} to enhance the smoothness of the graph data locally and globally by $L_{1}$ and $L_{2}$ penalties to enhance the adversarial robustness of the GNN framework.  
    \item \textbf{H2GCN-SVD}~\cite{H2GCNSVD} combines the H2GCN~\cite{H2GCN} with singular value decomposition techniques to mitigate the malicious influences of the high-rank topology attacks during the message-passing mechanism of the high-pass filter of the GCN layer. 
\end{itemize}

\begin{figure*}[h]
    \centering
    %\vspace{-0.3.97cm}
    \begin{subfigure}[b]{0.32\textwidth}
    	\centering
    	\includegraphics[width=\textwidth,height=3.6cm]{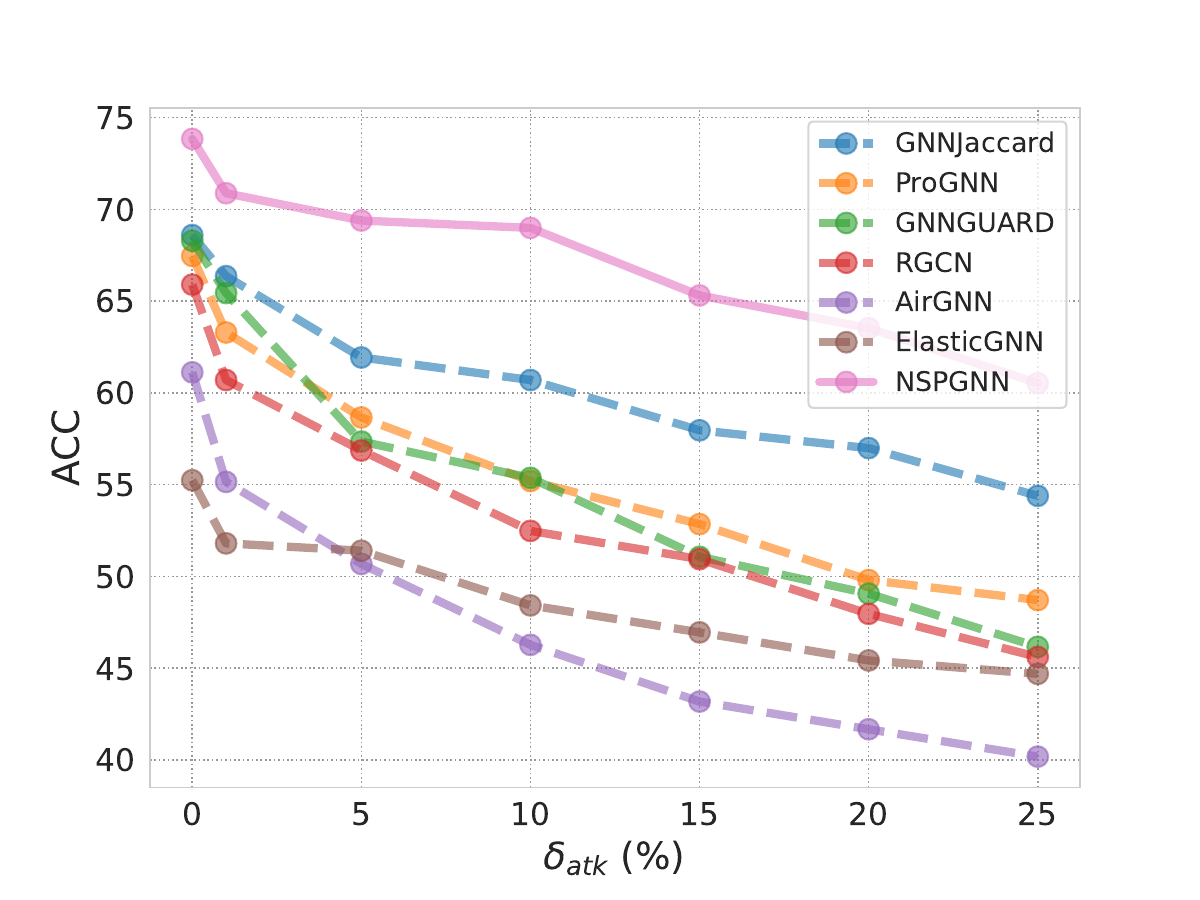}
    	\caption{Chameleon for Mettack}
        \label{fig-acc-robust-model-chameleon-Mettack}
    \end{subfigure}
    \hfill
    \begin{subfigure}[b]{0.32\textwidth}
    	\centering
     	\includegraphics[width=\textwidth,height=3.6cm]{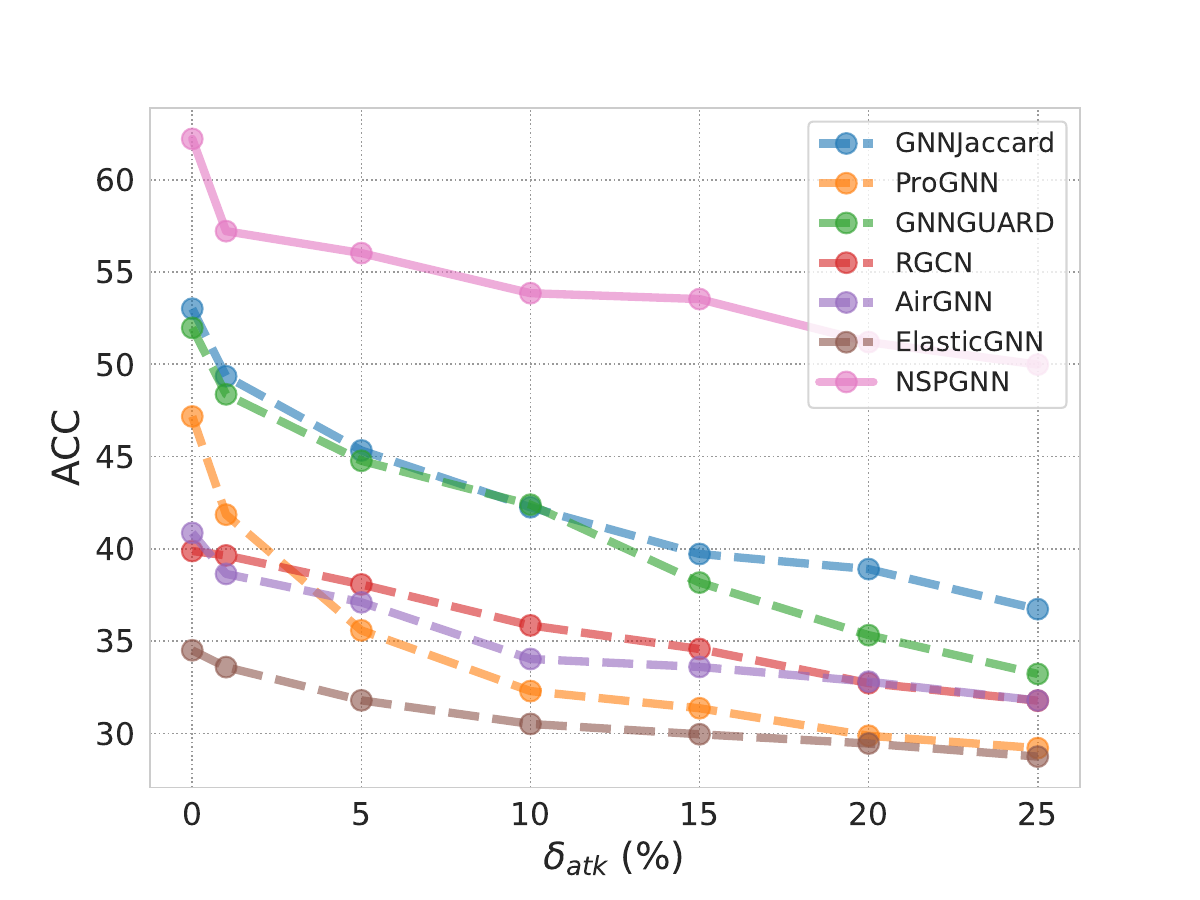}
     	\caption{Squirrel for Mettack}
            \label{fig-acc-robust-model-squirrel-Mettack}
    \end{subfigure}
    \hfill
    \begin{subfigure}[b]{0.32\textwidth}
    	\centering
     	\includegraphics[width=\textwidth,height=3.6cm]{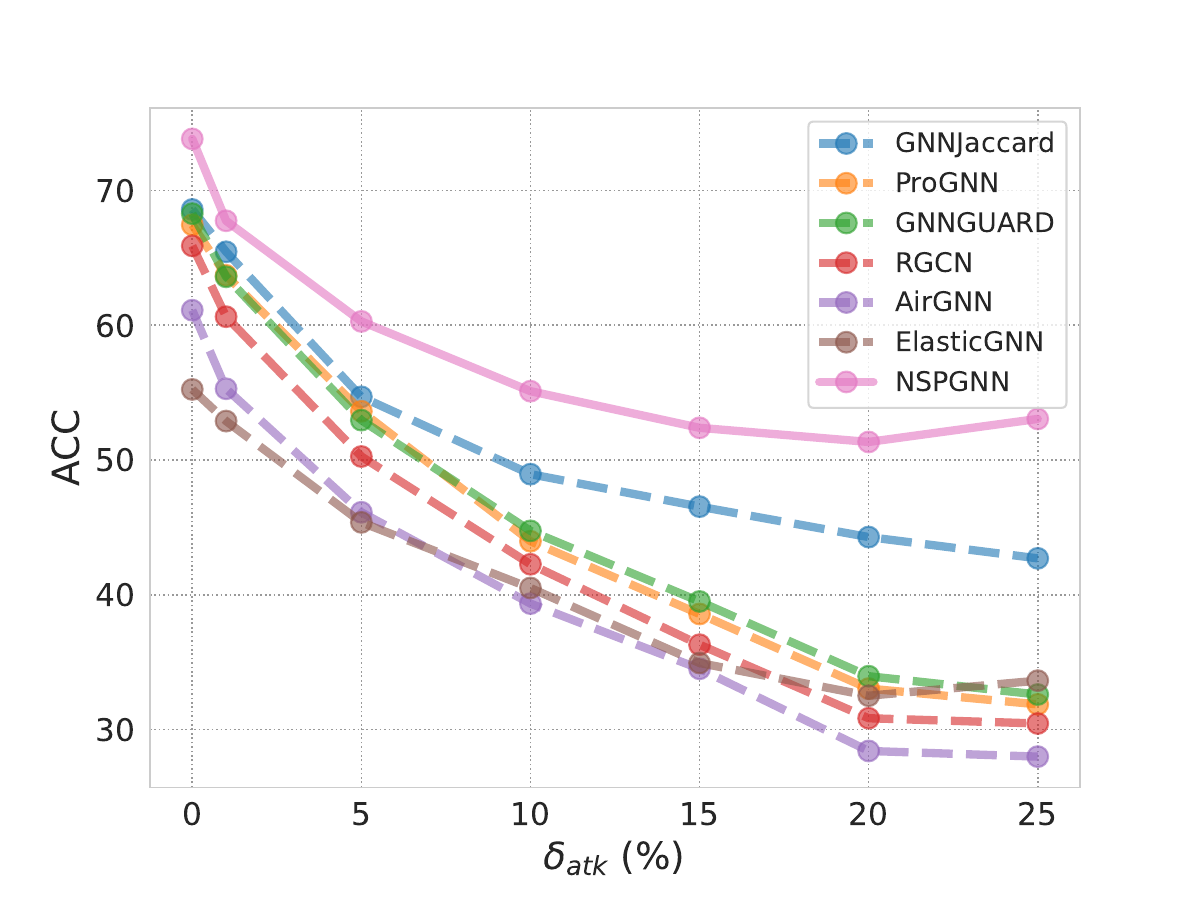}
     	\caption{Chameleon for Minmax}
            \label{fig-acc-robust-model-chameleon-Minmax}
    \end{subfigure}
    \begin{subfigure}[b]{0.32\textwidth}
    	\centering
    	\includegraphics[width=\textwidth,height=3.6cm]{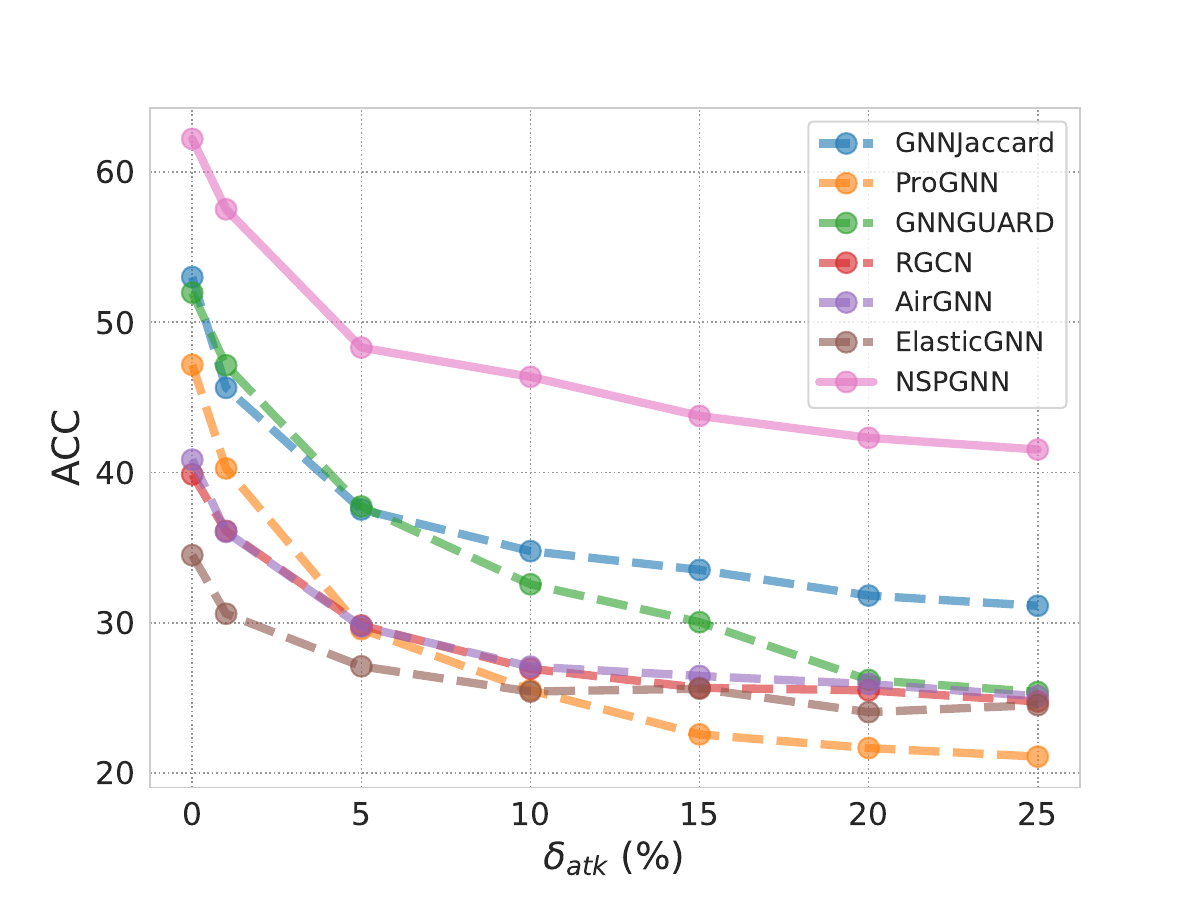}
    	\caption{Squirrel for Minmax}
        \label{fig-acc-robust-model-squirrel-Minmax}
    \end{subfigure}
    \hfill
    \begin{subfigure}[b]{0.32\textwidth}
    	\centering
     	\includegraphics[width=\textwidth,height=3.6cm]{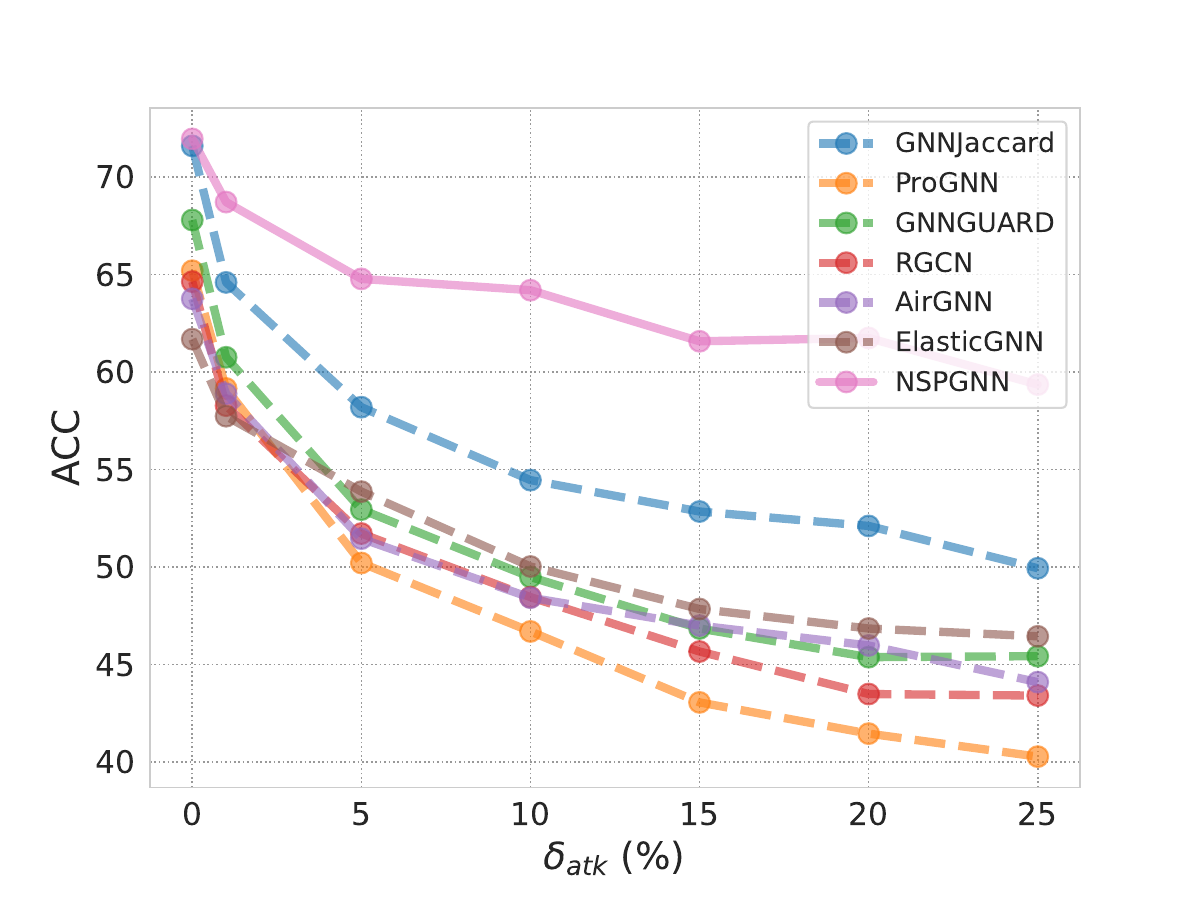}
     	\caption{Crocodile for Minmax}
            \label{fig-acc-robust-model-crocodile-Minmax}
    \end{subfigure}
    \hfill
    \begin{subfigure}[b]{0.32\textwidth}
    	\centering
     	\includegraphics[width=\textwidth,height=3.6cm]{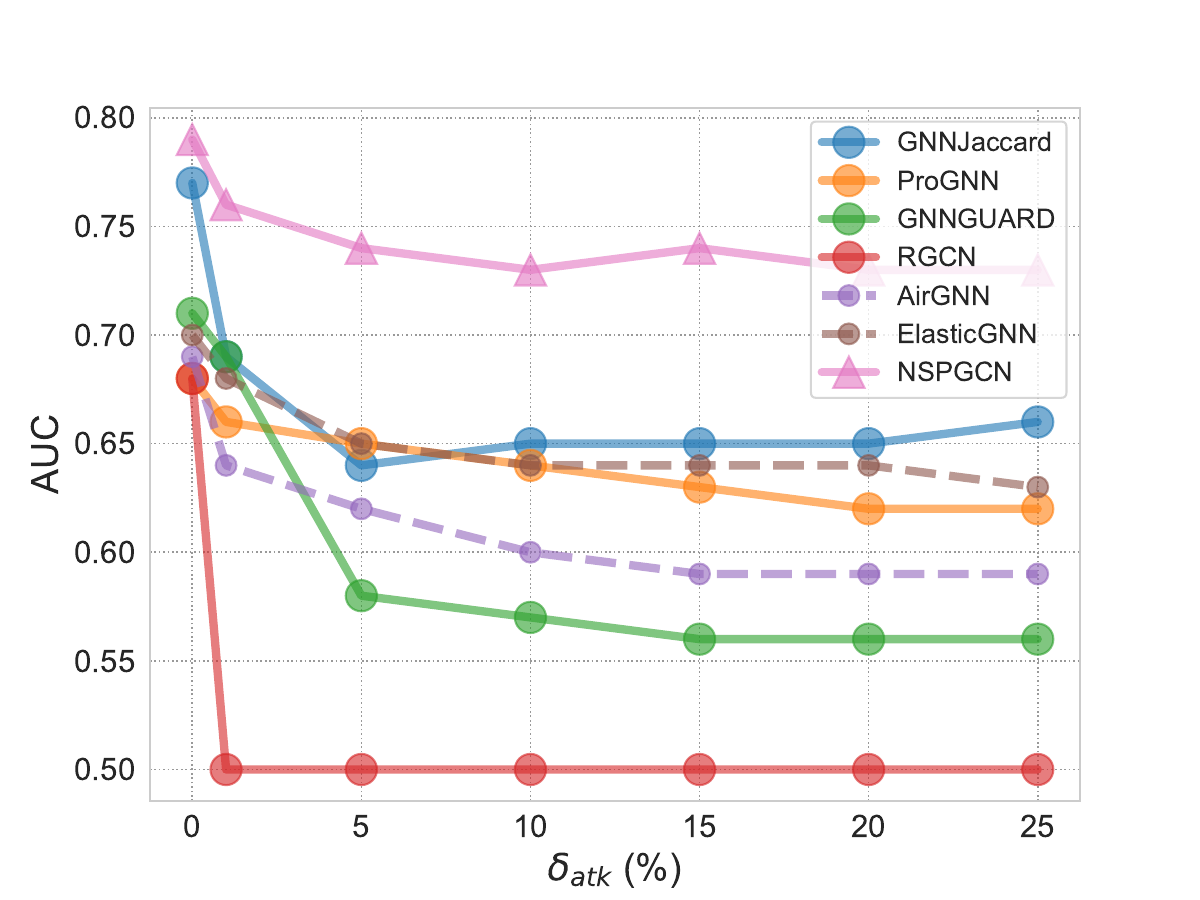}
     	\caption{Tolokers for Minmax}
            \label{fig-acc-robust-model-crocodile-Minmax}
    \end{subfigure}
    \hfill
    \begin{subfigure}[b]{0.32\textwidth}
        \centering
        \includegraphics[width=\textwidth,height=3.6cm]{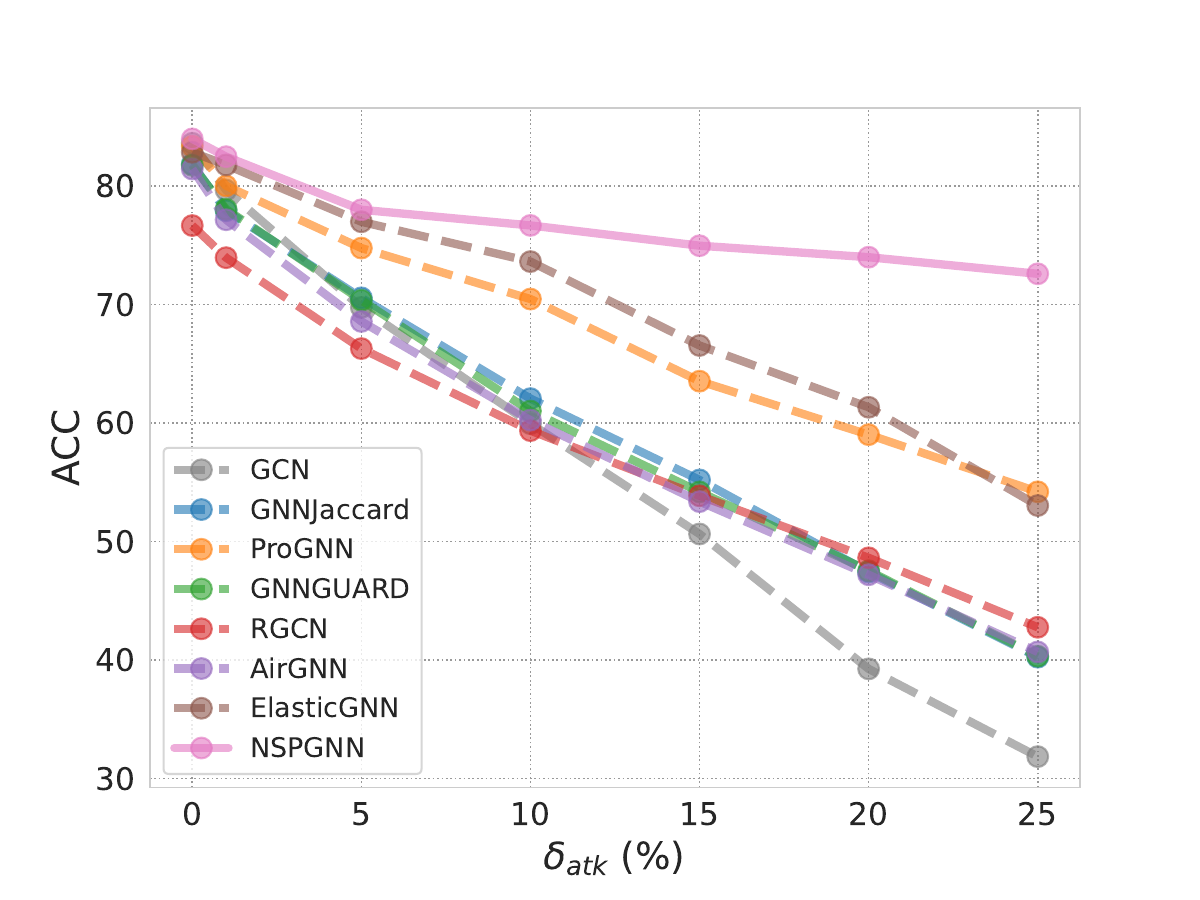}
        \caption{Cora for Mettack}
        \label{Fig-acc-robust-model-cora-mettack}
    \end{subfigure}
    \hfill
    \begin{subfigure}[b]{0.32\textwidth}
        \centering
        \includegraphics[width=\textwidth,height=3.6cm]{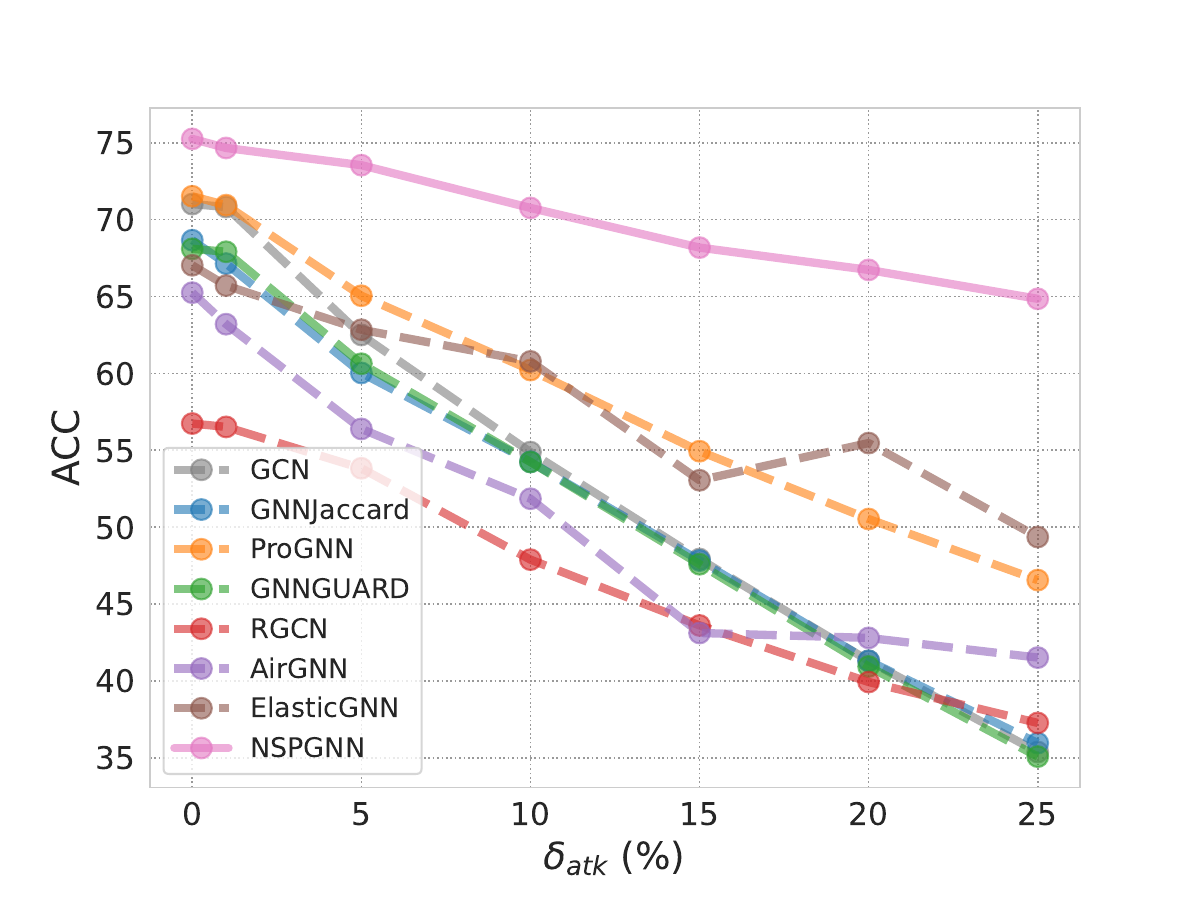}
        \caption{CiteSeer for Mettack}
        \label{Fig-acc-robust-model-citeseer-mettack}
    \end{subfigure}
    \hfill
    \begin{subfigure}[b]{0.32\textwidth}
        \centering
        \includegraphics[width=\textwidth,height=3.6cm]{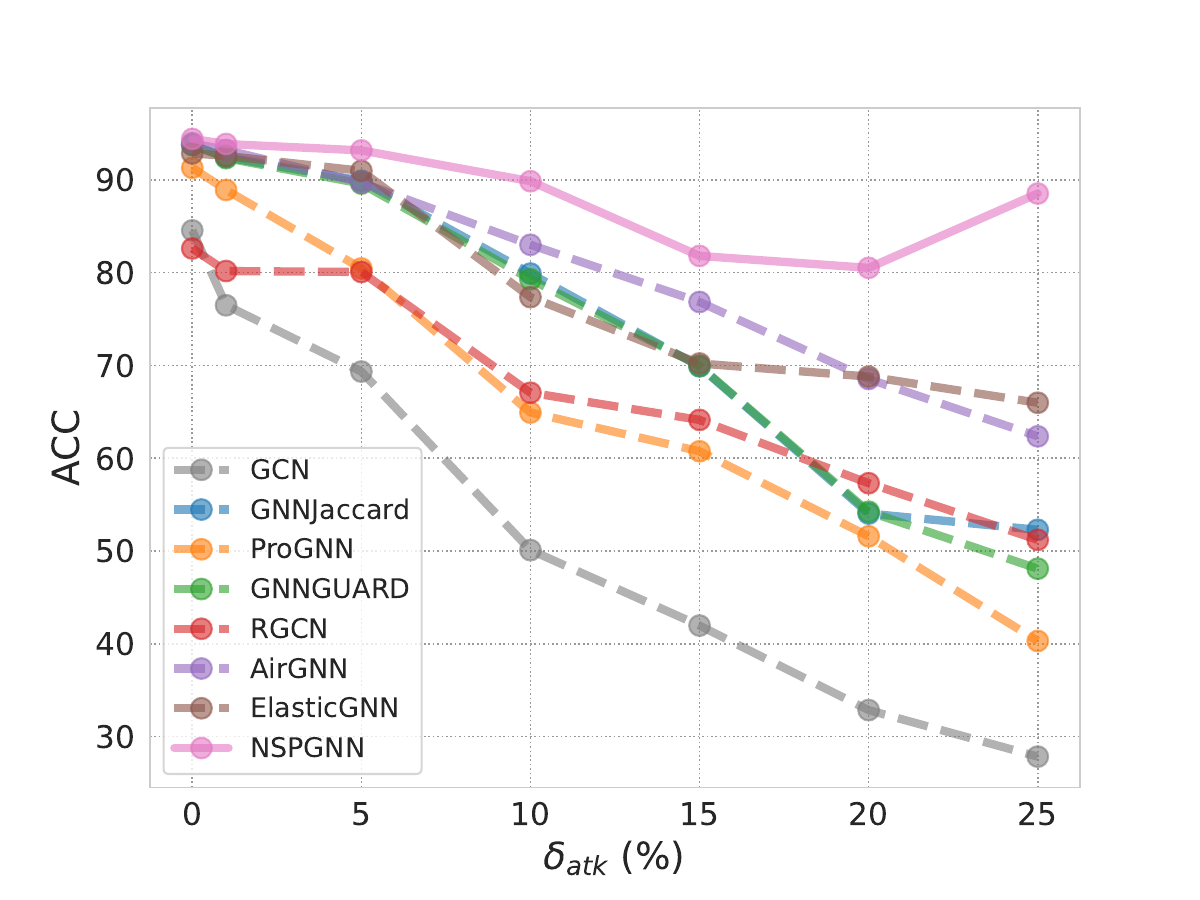}
        \caption{Photo for Mettack}
        \label{Fig-acc-robust-model-photo-mettack}
    \end{subfigure}
    \begin{subfigure}[b]{0.32\textwidth}
        \centering
        \includegraphics[width=\textwidth,height=3.6cm]{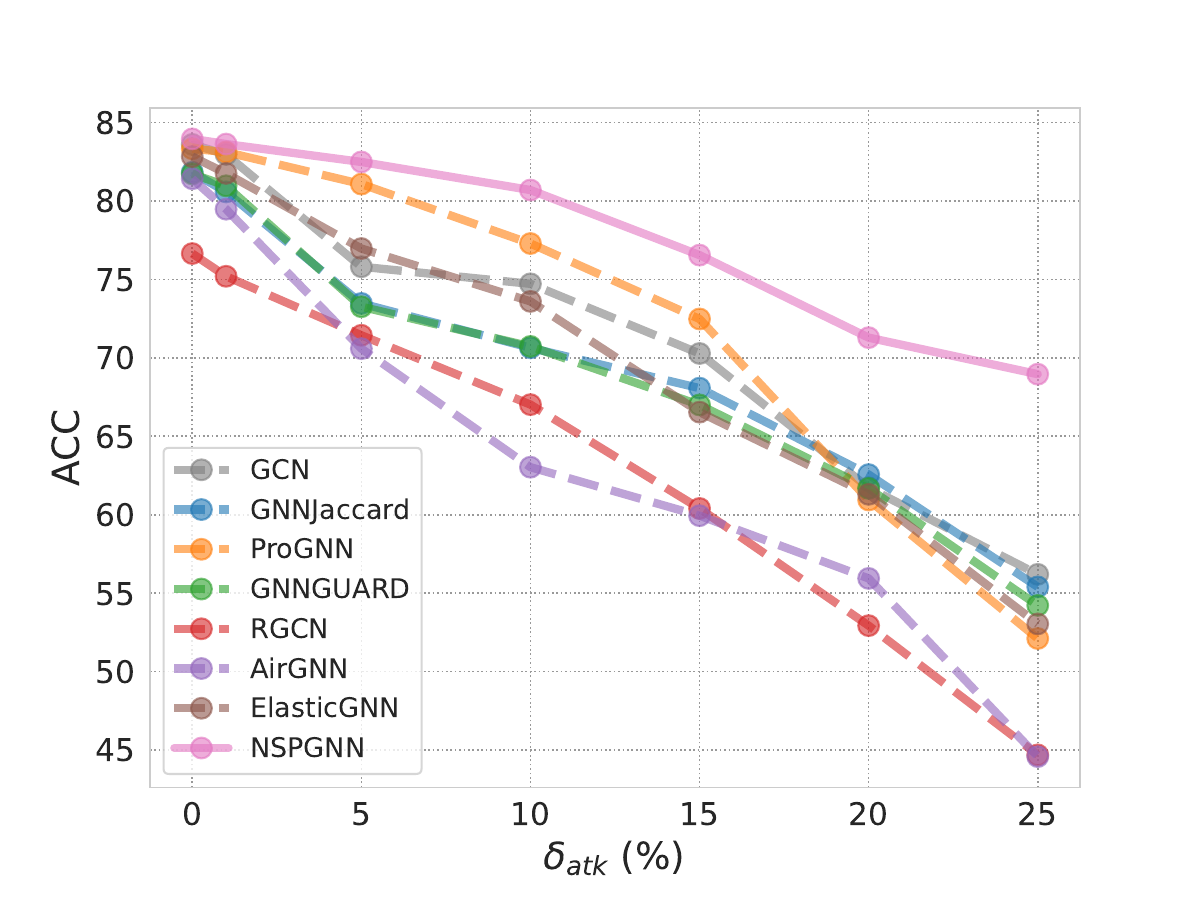}
        \caption{Cora for Minmax}
        \label{Fig-acc-robust-model-cora-minmax}
    \end{subfigure}
    \hfill
    \begin{subfigure}[b]{0.32\textwidth}
        \centering
        \includegraphics[width=\textwidth,height=3.6cm]{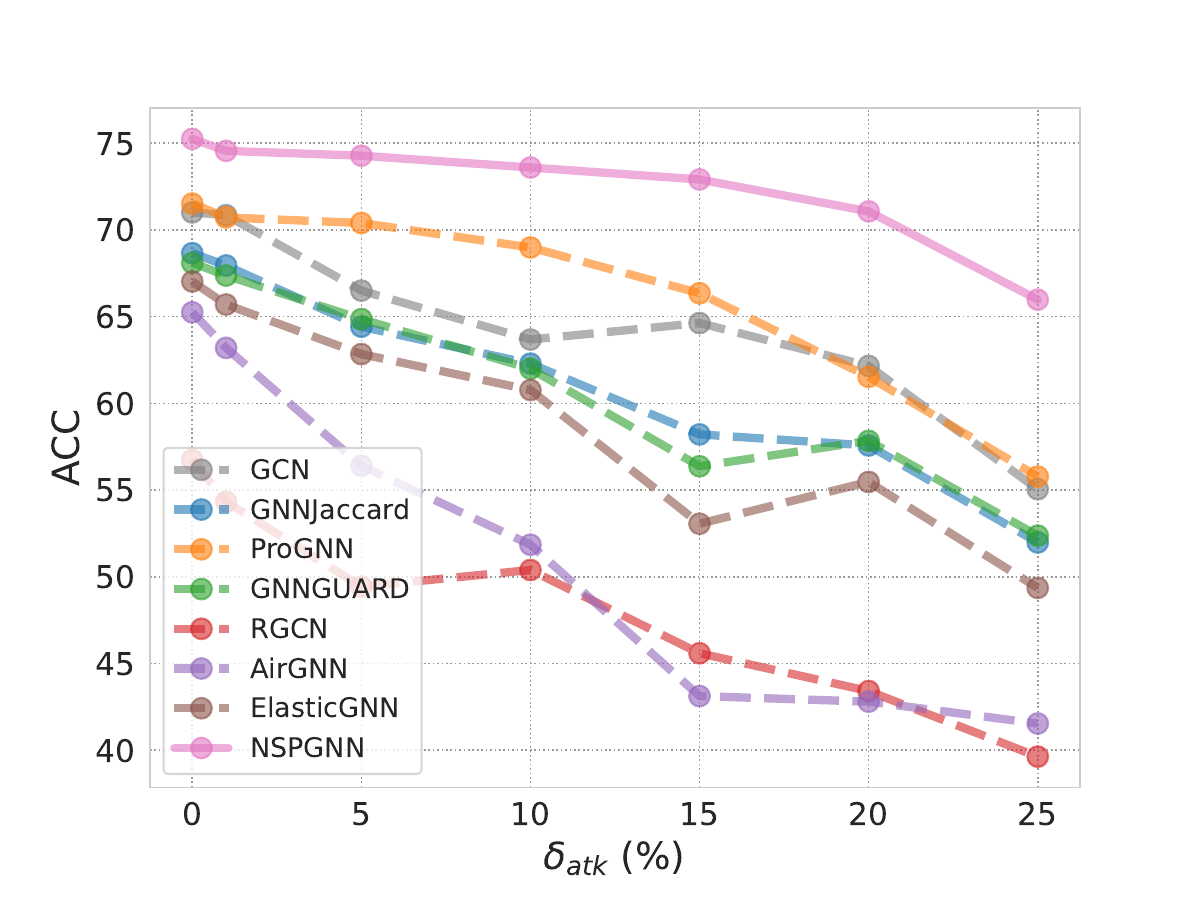}
        \caption{CiteSeer for Minmax}
        \label{Fig-acc-robust-model-citeseer-minmax}
    \end{subfigure}
    \hfill
    \begin{subfigure}[b]{0.32\textwidth}
        \centering
        \includegraphics[width=\textwidth,height=3.6cm]{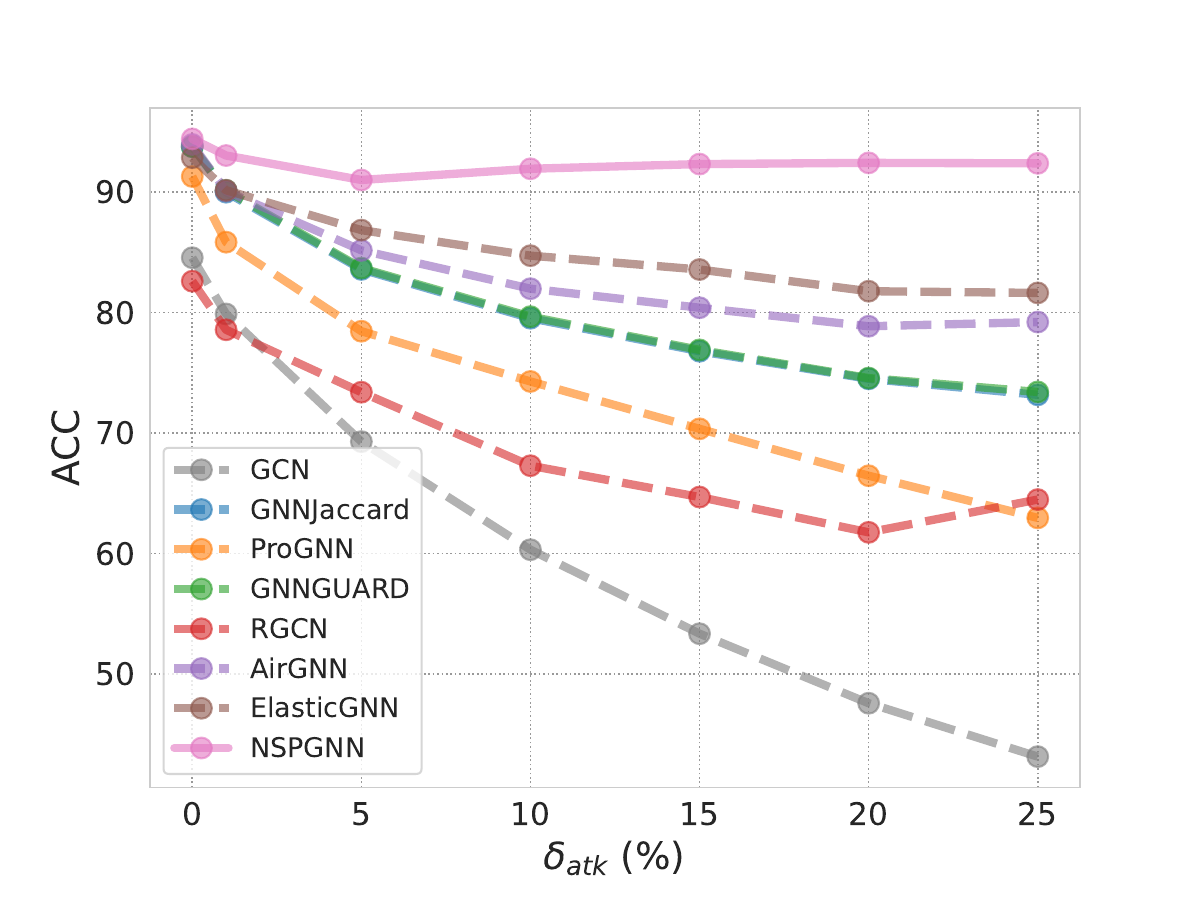}
        \caption{Photo for Minmax}
        \label{Fig-acc-robust-model-photo-minmax}
    \end{subfigure}
    %\vspace{-0.3.97cm}
    \caption{Robust performances of robust GNNs over \textbf{heterophilic} and \textbf{homophilic} graphs.}
    \label{Fig-acc-robust-model}
    %\vspace{-0.3.97cm}
\end{figure*}

\begin{figure*}[h]
    \centering
    \begin{subfigure}[b]{0.32\textwidth}
        \centering
        \includegraphics[width=\textwidth,height=3.6cm]{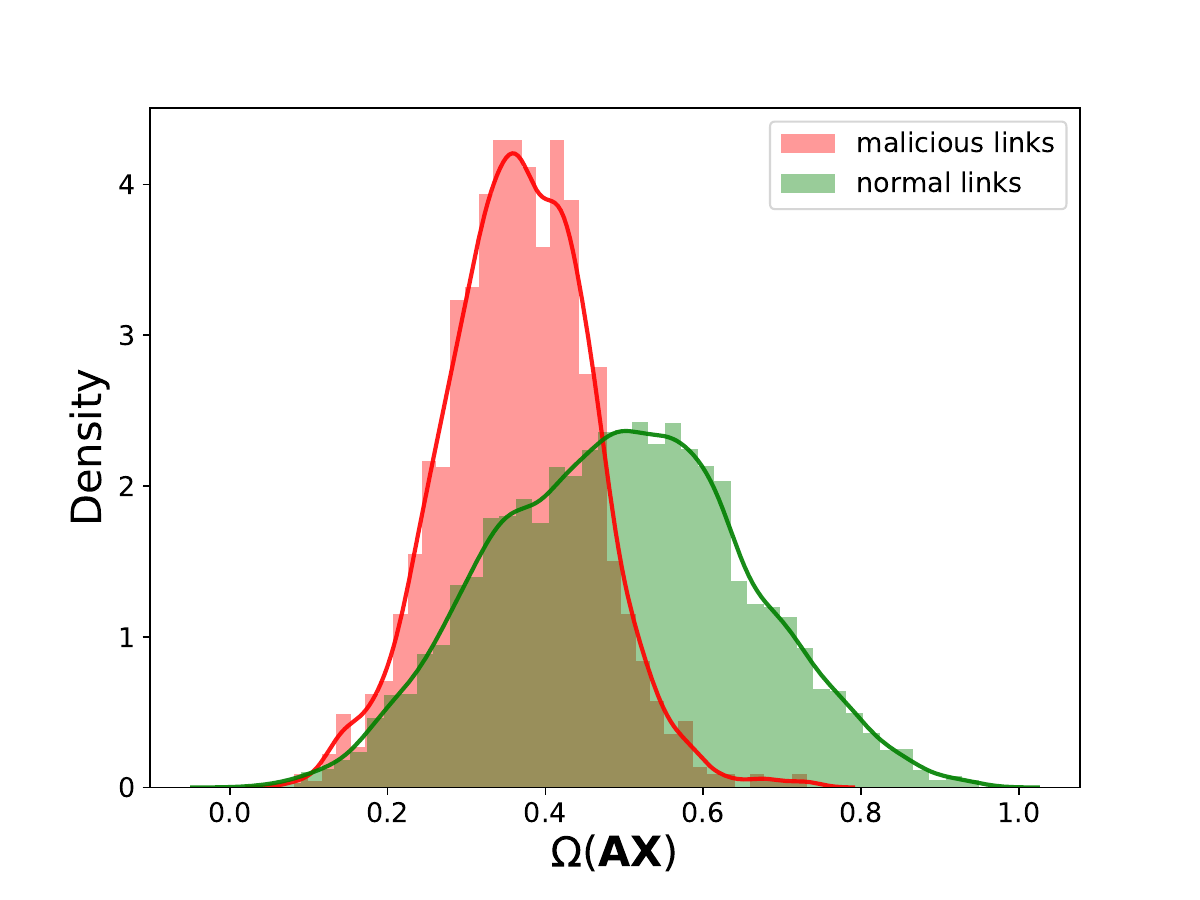}
        \caption{Cora for Mettack}
        \label{Fig-density-similarity-cora-mettack}
    \end{subfigure}
    \hfill
    \begin{subfigure}[b]{0.32\textwidth}
        \centering
        \includegraphics[width=\textwidth,height=3.6cm]{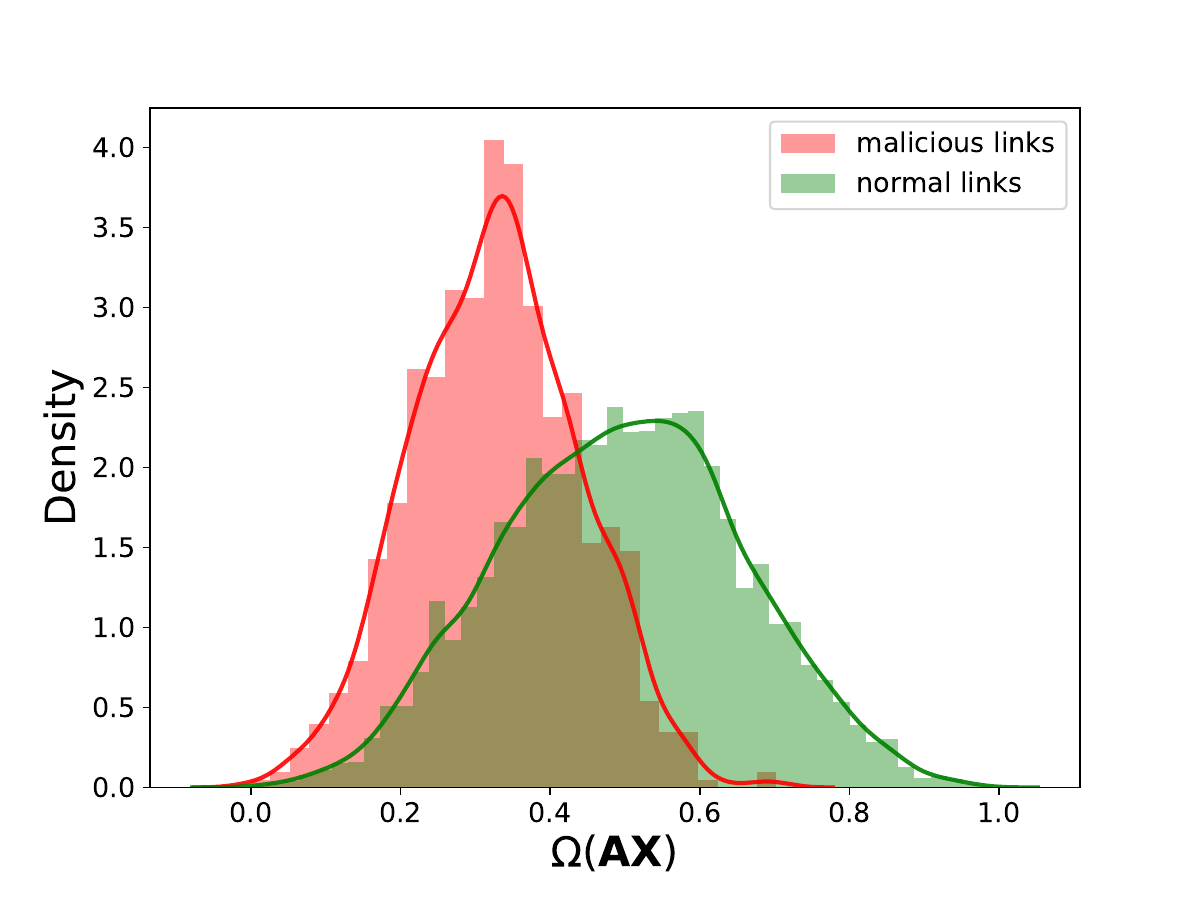}
        \caption{Cora for Minmax}
        \label{Fig-density-similarity-cora-minmax}
    \end{subfigure}
    \hfill
    \begin{subfigure}[b]{0.32\textwidth}
        \centering
        \includegraphics[width=\textwidth,height=3.6cm]{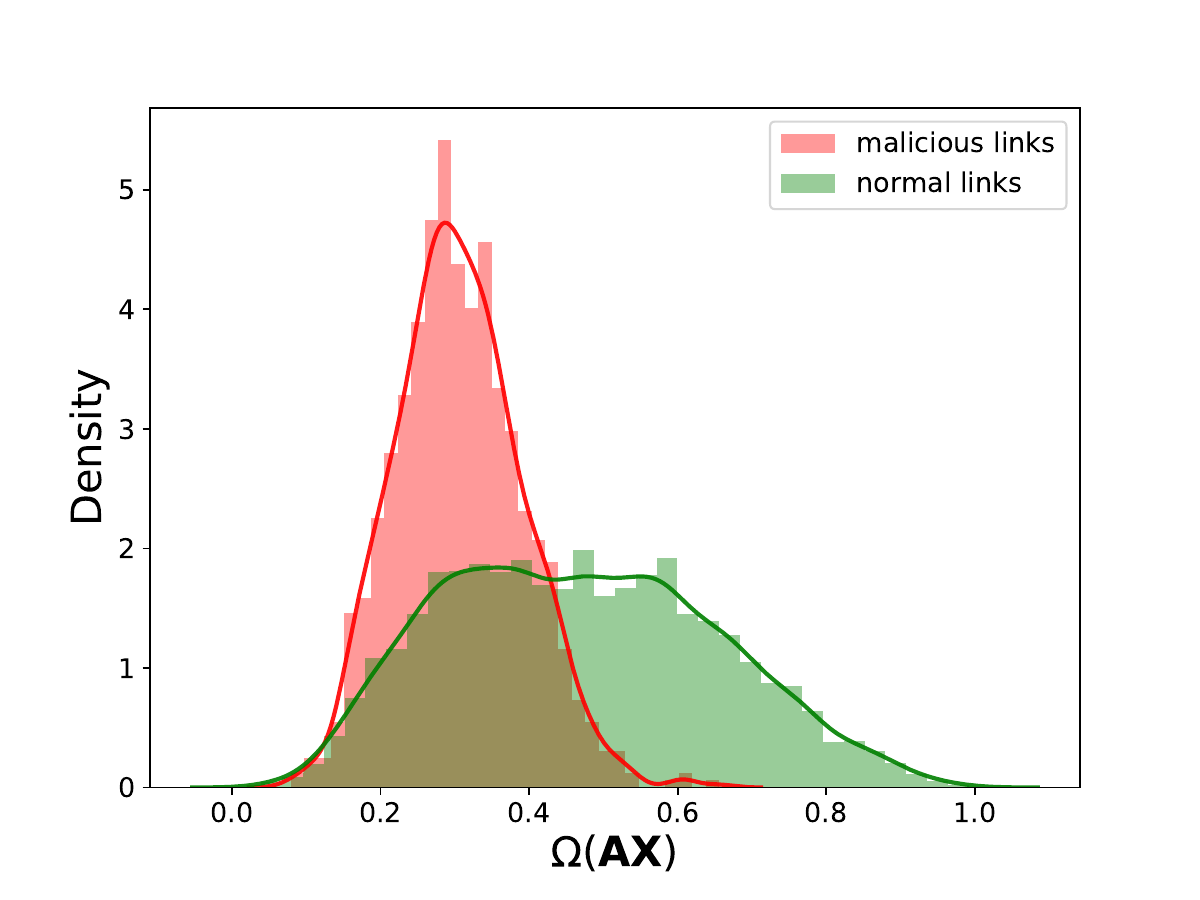}
        \caption{CiteSeer for Mettack}
    \end{subfigure}
    \hfill
    \begin{subfigure}[b]{0.32\textwidth}
        \centering
        \includegraphics[width=\textwidth,height=3.6cm]{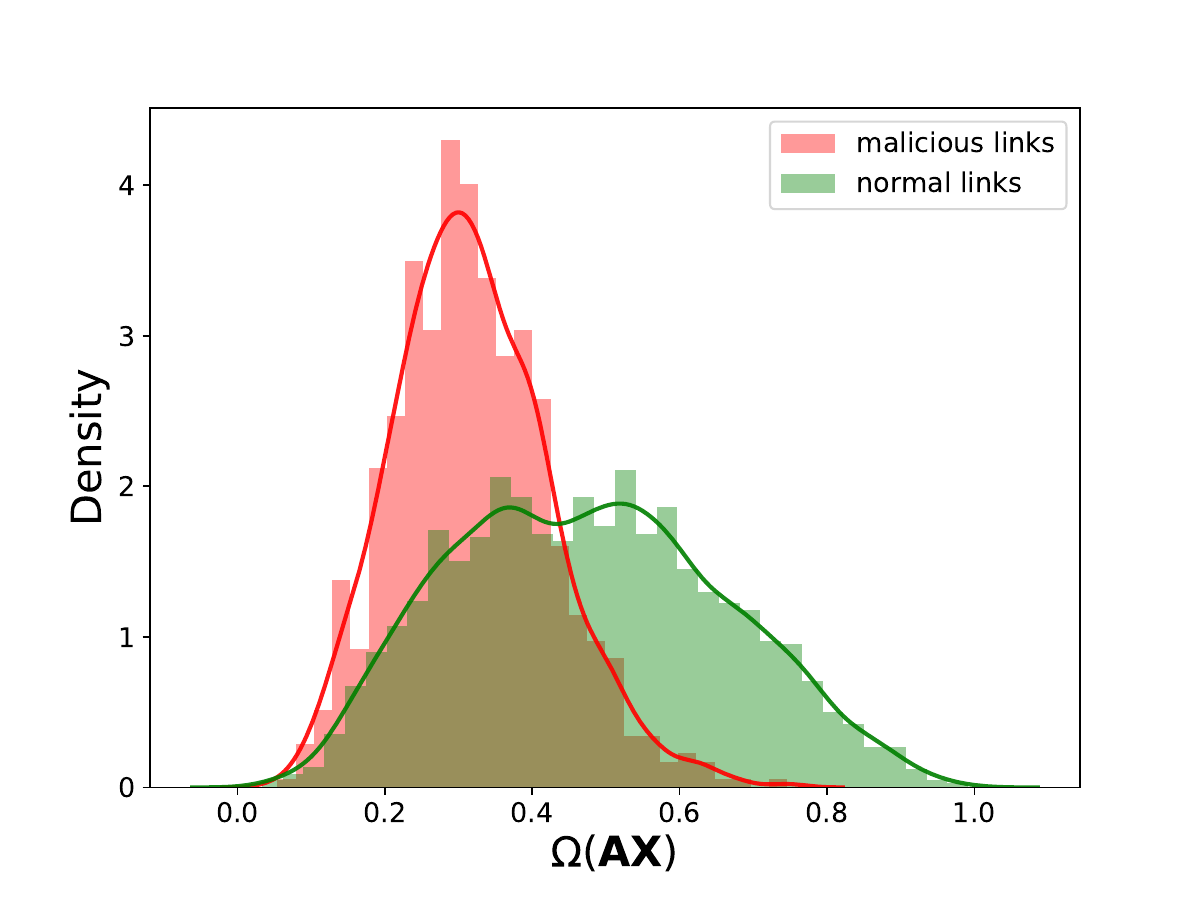}
        \caption{CiteSeer for Minmax}
    \end{subfigure}
    \hfill
    \begin{subfigure}[b]{0.32\textwidth}
        \centering
        \includegraphics[width=\textwidth,height=3.6cm]{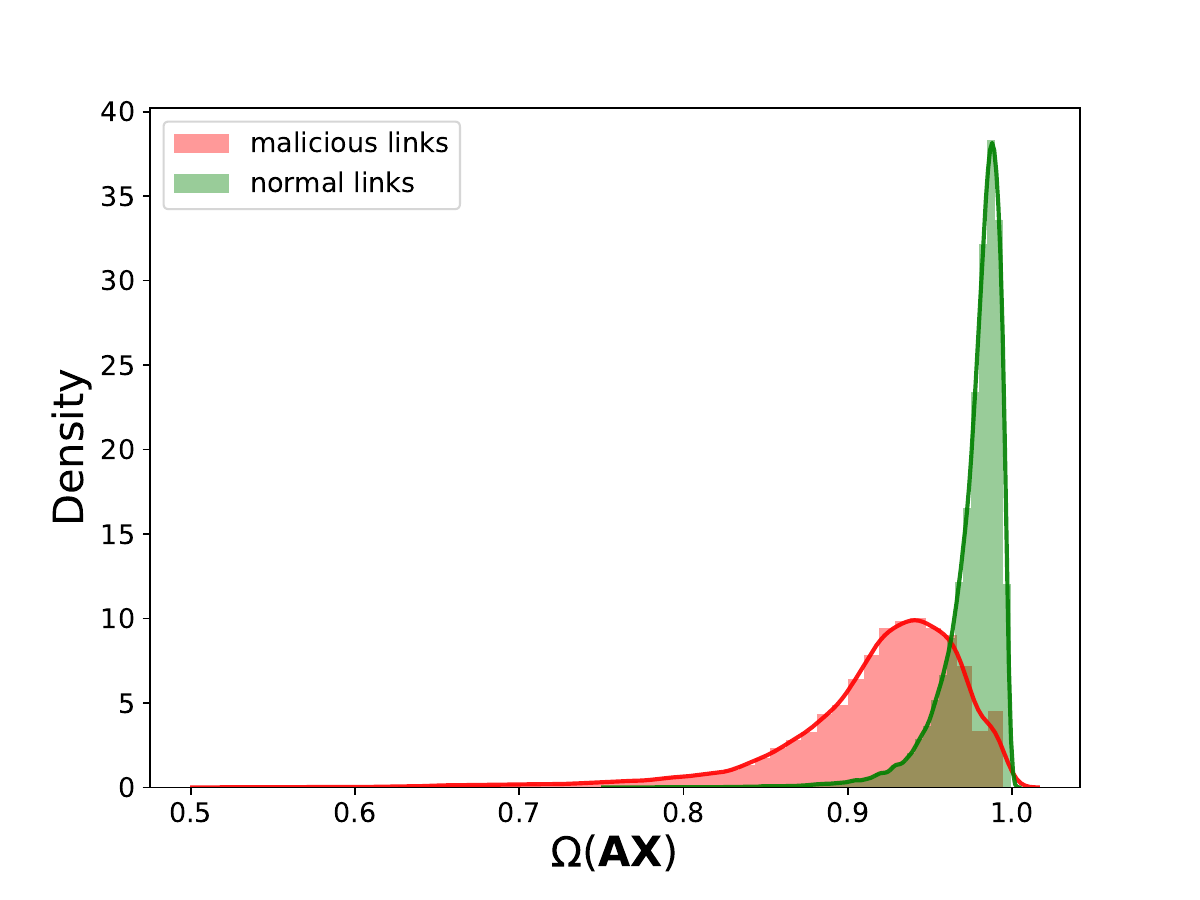}
        \caption{Photo for Mettack}
    \end{subfigure}
    \hfill
    \begin{subfigure}[b]{0.32\textwidth}
        \centering
        \includegraphics[width=\textwidth,height=3.6cm]{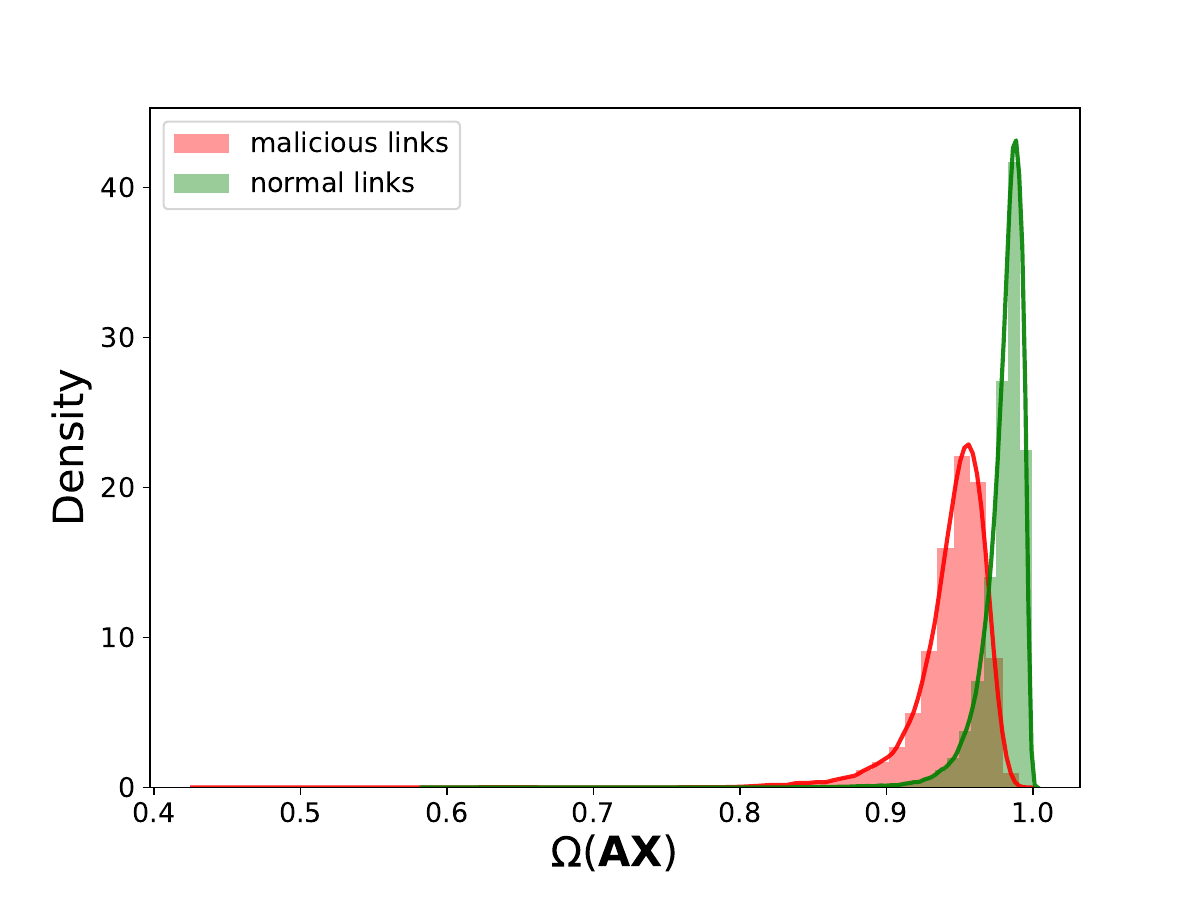}
        \caption{Photo for Minmax}
    \end{subfigure}
    \caption{Density plots for homophilic graphs under different attack methods.}
    \label{Fig-homo-simi}
\end{figure*}

\subsection{Performances on Clean Graphs}

The results in Tab.~\ref{tab-exp-clean-heter} and \ref{tab-exp-clean-homo} present the semi-supervised node classification performances of \textbf{NSPGNN} and other state-of-the-art baselines over four clean heterophilic and three clean homophilic graphs. It is worth noting that \textbf{NSPGNN w.o.} is the ablation of the proposed model where we omit the negative kNN graphs construction. It is surprising that although our proposed neighbor-similarity-preserved message-passing mechanism is specially crafted to alleviate the malicious effects of the potential adversarial manipulations in the poisoned graphs. It is probable that the positive kNN graph constructions can serve as a useful graph data augmentation technique to refine the clean graph's structures and make it particularly suitable for semi-supervised node classification tasks regardless of the homophily degree of the graph data.   

\begin{table}[h]
	\centering
        %\vspace{-0.1cm}
	\caption{Performances of heterophilic GNNs over clean \textbf{heterophilic graphs}.}
	%\vspace{-0.35cm}
        \label{tab-exp-clean-heter}
	\resizebox{0.45\textwidth}{!}{%
        \begin{tabular}{c|cccc}
        \toprule 
        Dataset & Chameleon & Squirrel & Crocodile & Tolokers \\
        \hline
        GPRGNN & 71.01 (1.12) & 51.30 (1.06) & 63.18 (0.31) & 0.70 (0.012)\\
        \hline
        FAGNN & 69.89 (0.63) & 52.62 (0.57) & 66.54 (0.37) & 0.73 (0.010)\\
        \hline
        H2GCN & 61.73 (0.80) & 34.39 (0.60) & 54.71 (3.63) & 0.76 (0.016)\\
        \hline
        GBKGNN & 69.39 (0.67) & 53.68 (0.77) & 69.66 (0.58) & 0.72 (0.015)\\
        \hline
        BMGCN & 67.90 (1.07) & 47.78 (0.61) & 69.42 (0.46) & 0.74 (0.01)\\
        \hline
        ACMGNN & 68.20 (1.20) & 51.21 (1.33) & 69.82 (0.78) & 0.67 (0.005)\\
        \hline
        GARNET & 67.74 (1.10) & 47.82 (0.92) & 68.53 (0.44) & 0.79 (0.010)\\
        \hline
        GCN-Jaccard & 68.60 (0.78) & 53.10 (0.91) & 71.60 (0.49) & 0.77 (0.015)\\
        \hline
        ProGNN & 67.46 (1.38) & 47.17 (1.13) & 65.20 (0.74) & 0.68 (0.007) \\
        \hline
        GNNGUARD & 68.29 (0.20) & 51.97 (0.25) & 67.80 (0.35) & 0.71 (0.015) \\
        \hline
        RGCN & 65.90 (1.03) & 39.88 (2.23) & 64.64 (0.41) & 0.68 (0.013) \\
        \hline
        AirGNN & 61.12 (1.00) & 40.86 (0.64) & 63.76 (0.20) & 0.69 (0.012) \\
        \hline
        ElasticGNN & 55.24 (1.04) & 34.50 (0.55) & 61.69 (0.66) & 0.70 (0.009) \\
        \hline
        NSPGNN w.o. & 72.87 (0.91) & 59.90 (0.63) & 71.85 (0.60) & 0.77 (0.010)\\
        \hline
        NSPGNN & \textbf{73.84 (0.56)} & \textbf{62.21 (0.69)} & \textbf{71.96 (0.93)} & \textbf{0.79 (0.014)}\\
        \bottomrule
        \end{tabular}
	}
\end{table}

\begin{table}[h]
	\centering
        %\vspace{-0.1cm}
	\caption{Performances of robust GNNs over clean \textbf{homophilic graphs}.}
	%\vspace{-0.35cm}
        \label{tab-exp-clean-homo}
	\resizebox{0.4\textwidth}{!}{%
        \begin{tabular}{c|cccc}
        \toprule 
        Dataset & Cora & CiteSeer & Photo \\
        \hline
        GCN & 83.59 (0.27) & 71.01 (0.34) & 84.54 (0.41)\\
        \hline
        GCNJaccard & 81.83 (0.21) & 68.66 (0.28) & 93.75 (0.94)\\
        \hline
        ProGNN & 83.33 (0.58) & 71.52 (1.08) & 91.28 (0.30)\\
        \hline
        GNNGUARD & 81.73 (0.25) & 68.10 (0,18) & 93.77 (0.75)\\
        \hline
        RGCN & 76.65 (1.31) & 56.74 (0.69) & 82.59 (3.21)\\
        \hline
        AirGNN & 81.42 (0.29) & 65.25 (0.50) & 93.94 (0.68)\\
        \hline
        H2GCN-SVD & 75.40 (0.38) & 53.91 (0.45) & -\\
        \hline
        NSPGNN w.o. & 83.46 (0.54) & 72.67 (1.67) & 94.20 (0.98)\\
        \hline
        NSPGNN & \textbf{83.96 (0.19)} & \textbf{75.25 (0.45)} & \textbf{94.40 (1.04)}\\
        \bottomrule
        \end{tabular}
	}
\end{table}

\subsection{Robustness over Heterophilic Graphs}
\subsubsection{Defense Against Mettack}
The results in Tab.~\ref{tab-exp-Mettack} present the semi-supervised node classification performances of \textbf{NSPGNN} and other HGNN baselines over three heterophilic graphs under Mettack with varying attacking powers. It is worth noting that \textbf{NSPGNN w.o.} is the ablation of the proposed model where we omit the negative kNN graphs construction. We have the following three observations: \textbf{1)} The proposed model \textbf{NSPGNN} and its ablation consistently outperform other HGNNs by a large margin. For example, \textbf{NSPGNN} outperforms the second-best performances around $24.72\%$, $27.12\%$, $29.88\%$, $32.99\%$, $34.31\%$, $33.71\%$ for Squirrel dataset with attacking power $\delta_{atk}$ equal to $1\%$, $5\%$, $10\%$, $15\%$, $20\%$, $25\%$ respectively. These phenomenons indicate that preserving the neighbor similarity can effectively defend against Mettack on heterophilic graphs. \textbf{2)} Compared to its ablation, \textbf{NSPGNN} achieves slightly better node classification accuracies for most cases, which demonstrates that the negative kNN graphs play an important role when supervising the neighbor-similarity-guided propagation. \textbf{3)}, the performance gains between \textbf{NSPGNN w.o} and the second-best performances are larger than the performance gains between \textbf{NSPGNN w.o} and \textbf{NSPGNN} indicates that propagation with positive kNN graphs is far more effective than propagation with negative kNN graphs. It makes sense since preserving the high-similarity information will likely prune the malicious links and thus enhance the adversarial robustness of GNNs.  

In the meanwhile, we also compare the proposed model with other robust model benchmarks against Mettack on heterophilic graphs in Fig.~\ref{fig-acc-robust-model-chameleon-Mettack} and \ref{fig-acc-robust-model-squirrel-Mettack}. It is observed that \textbf{NSPGNN} consistently outperforms other robust models by a large margin. This is due to the assumption that pruning links connecting dissimilar ego node features can enhance the adversarial robustness is unsuitable for heterophilic graphs. There already exists a large amount of inter-class links that connect dissimilar ego node features in heterophilic graphs and pruning links according to this strategy may likely delete normal inter-class links. However, \textbf{NSPGNN} prunes links based on the aggregated neighbors' feature similarity instead of ego feature similarity and thus can precisely prune the malicious inter-class links, which makes it particularly suitable for enhancing the adversarial robustness of GNNs over heterophilic graphs.      

\subsubsection{Defense Against Minmax}
We also present the adversarial robustness of GNNs against Minmax--another typical graph structural attack method on heterophilic graphs in Tab.~\ref{tab-exp-Minmax}. Similar to Mettack, it is observed that \textbf{NSPGNN} and its ablation consistently outperform other HGNN baselines by a large margin in most cases. For example, the performance gains between \textbf{NSPGNN} and the second-best performances for Squirrel with $\delta_{atk}=1\%$, $5\%$, $10\%$, $15\%$, $20\%$, $25\%$ are $18.68\%$, $25.89\%$, $33.67\%$, $25.45\%$, $26.15\%$, $25.13\%$. It is demonstrated that preserving the neighbor similarity can also effectively mitigate the malicious effects caused by Minmax. Overall, the phenomenon that \textbf{NSPGNN} performs the best in most cases both against Mettack and Minmax illustrates that our proposed robust model indeed can effectively provide sufficient valid signals to supervise the propagation, which smooth the distance of intra-class nodes and enlarge the distance of inter-class nodes.   

On the other hand, Fig.~\ref{fig-acc-robust-model-chameleon-Minmax}, \ref{fig-acc-robust-model-squirrel-Minmax}, \ref{fig-acc-robust-model-crocodile-Minmax} present the robust performances of the proposed method compared with other robust baselines against Minmax. It is observed that \textbf{NSPGNN} significantly outperforms other robust models with different attacking powers. The largest gap between \textbf{NSPGNN} and the second-best performances are $24.24\%$, $33.41\%$ and $18.86\%$ for Chameleon, Squirrel and Crocodile respectively. In the meanwhile, the performance gaps between the proposed method and other robust models increase as the attacking power increases. This phenomenon demonstrates that preserving neighbor similarity can precisely prune a proportion of malicious effects even when the poisoned graphs are highly contaminated while vanilla robust models fail to effectively mitigate the malicious effects, particularly on highly poisoned graphs.    

\begin{figure}[h]
        %\vspace{-0.3.97cm}
	\centering
    \begin{subfigure}[b]{0.234\textwidth}
        \centering
        \includegraphics[width=\textwidth,height=3.5cm]{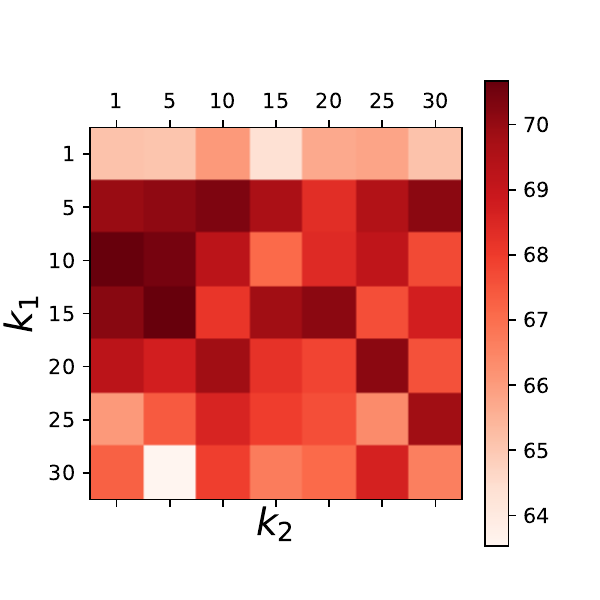}
        \caption{Cora for Mettack}
    \end{subfigure}
    \hfill
    \begin{subfigure}[b]{0.234\textwidth}
        \centering
        \includegraphics[width=\textwidth,height=3.5cm]{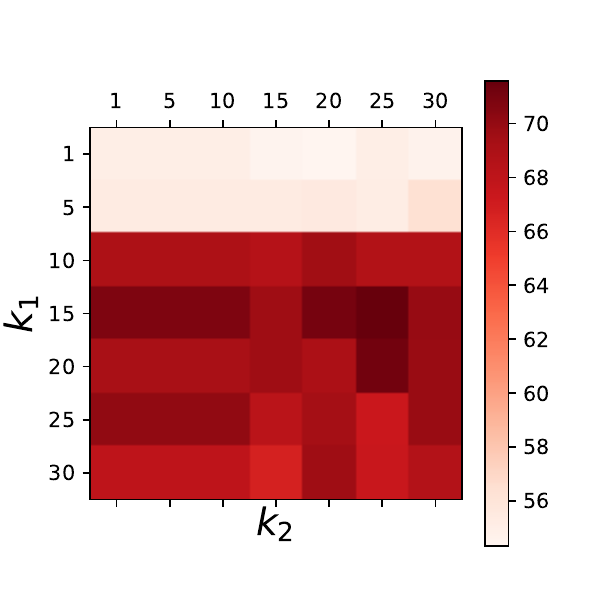}
        \caption{Cora for Minmax}
    \end{subfigure}
    \hfill
    \begin{subfigure}[b]{0.234\textwidth}
        \centering
        \includegraphics[width=\textwidth,height=3.5cm]{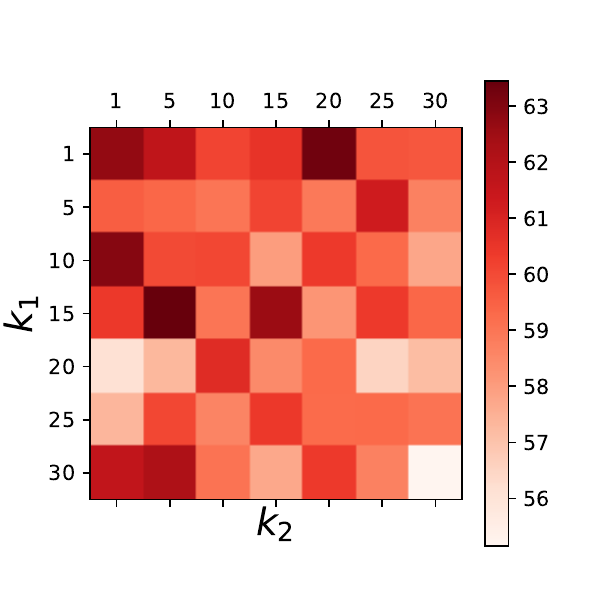}
        \caption{CiteSeer for Mettack}
    \end{subfigure}
    \hfill
    \begin{subfigure}[b]{0.234\textwidth}
        \centering
        \includegraphics[width=\textwidth,height=3.5cm]{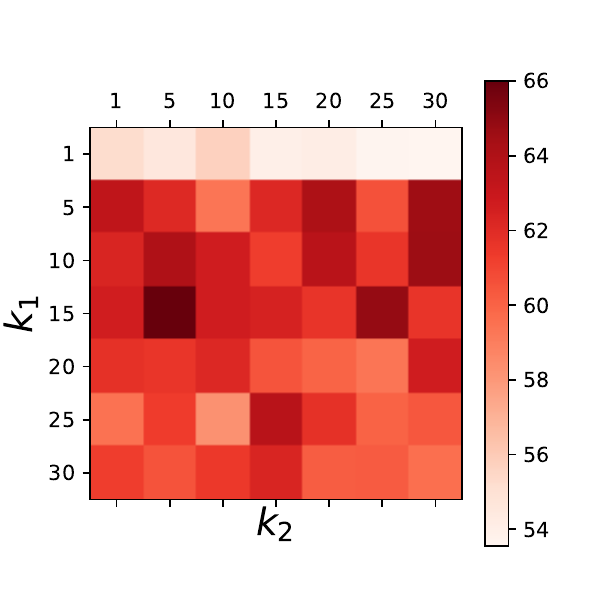}
        \caption{CiteSeer for Minmax}
    \end{subfigure}
    \hfill
    \begin{subfigure}[b]{0.234\textwidth}
    	\centering
    	\includegraphics[width=\textwidth,height=3.5cm]{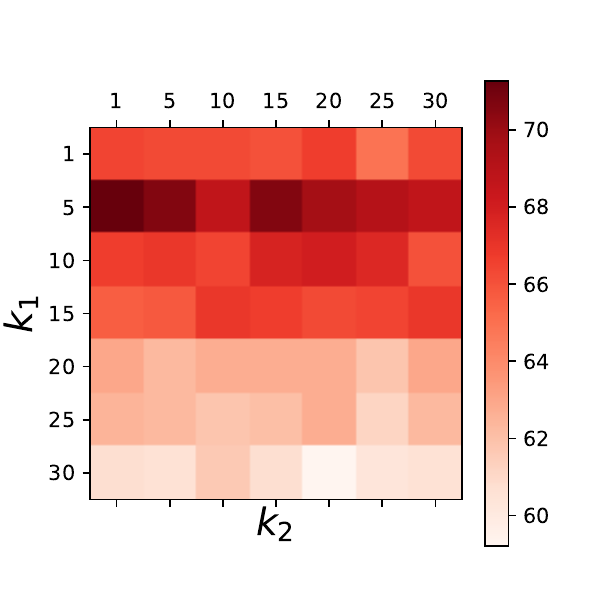}
    	\caption{Chameleon for Mettack}
    \end{subfigure}
	\hfill
    \begin{subfigure}[b]{0.234\textwidth}
    	\centering
     	\includegraphics[width=\textwidth,height=3.5cm]{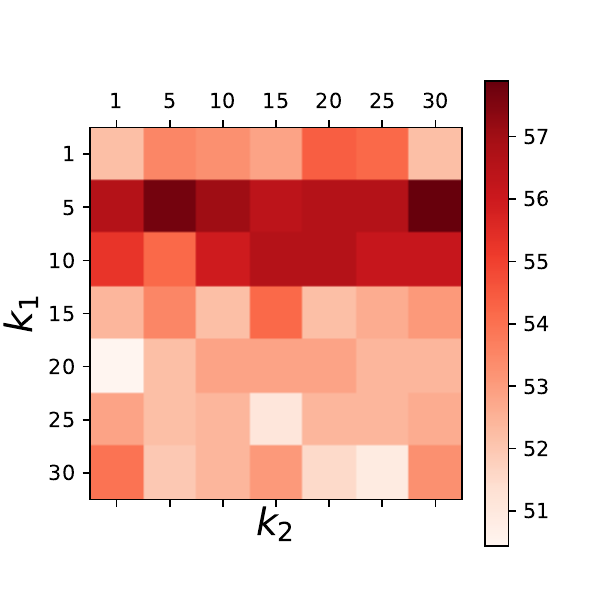}
     	\caption{Chameleon for Minmax}
    \end{subfigure}
    %\vspace{-0.3.97cm}
    \hfill
    \begin{subfigure}[b]{0.234\textwidth}
        \centering
        \includegraphics[width=\textwidth,height=3.5cm]{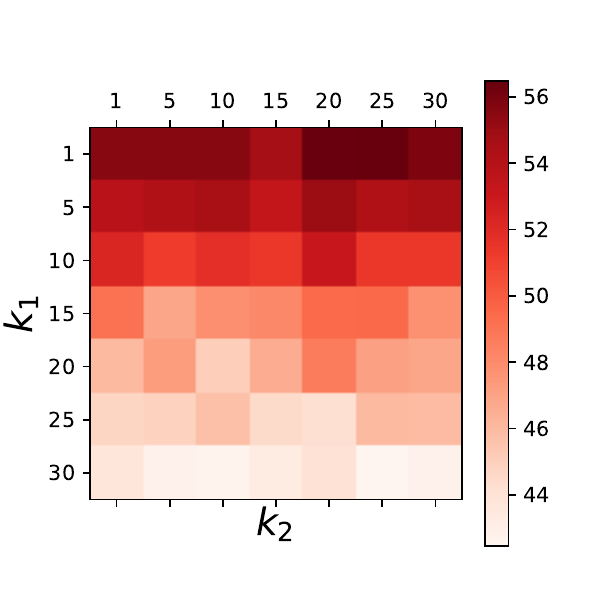}
        \caption{Squirrel for Mettack}
    \end{subfigure}
    \hfill
    \begin{subfigure}[b]{0.234\textwidth}
        \centering
        \includegraphics[width=\textwidth,height=3.5cm]{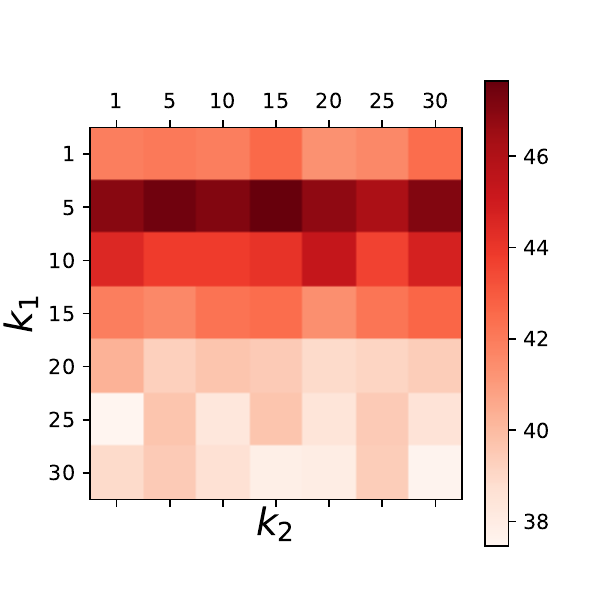}
        \caption{Squirrel for Minmax}
    \end{subfigure}
    \caption{Sensitivity analysis on the number of nearest neighbors for positive kNN graph ($k_{1}$) and negative kNN graph ($k_{2}$).}
    \label{fig-sensitivity-mat-chameleon}
    %\vspace{-0.6cm}
\end{figure}

\subsection{Robustness over Homophilic Graphs}
In this section, we analyze the adversarial robustness of the proposed model over homophilic graphs and experimentally verify that preserving neighbor similarity can also deal with malicious effects on homophilic graphs. Fig.~\ref{Fig-acc-robust-model-cora-mettack}, \ref{Fig-acc-robust-model-cora-minmax}, \ref{Fig-acc-robust-model-citeseer-mettack}, \ref{Fig-acc-robust-model-citeseer-minmax}, \ref{Fig-acc-robust-model-photo-mettack}, \ref{Fig-acc-robust-model-photo-minmax} presents the robust performances of the proposed model compared with current robust baselines against graph adversarial attacks on Cora, CiteSeer and Photo. The observations are two-fold: \textbf{1)} \textbf{NSPGNN} consistently achieves the best performances compared to the robust baselines under different attacking scenarios for homophilic graphs. It indicates that preserving neighbor similarity can also mitigate the adversarial effects of graph adversarial attacks on homophilic graphs and even performs better than preserving ego similarity (GCN-Jaccard, ProGNN, GNNGUARD etc.). This result coincides with the theoretical proof of Theorem.~\ref{theorem-1} since the theoretical result is independent of the homophily ratio of the input graph data. \textbf{2)} Similar to the results in Tab.~\ref{tab-exp-Minmax}, \textbf{NSPGNN} performs the best among the robust models on homophilic clean graphs, which demonstrates that preserving the neighbor similarity can serve as an effective data augmentation technique to refine the clean graph's structure for better semi-supervised node classification performances. In contrast, preserving ego similarity such as GCN-Jaccard may sacrifice the clean graph's accuracy. This phenomenon is supported by Theorem.~\ref{theorem-1} since the attack loss (negative classification loss) is negatively related to neighbor similarity instead of ego similarity. Thus, preserving ego similarity cannot optimize the classification loss for the clean graph.    

\subsection{Similarity on Homophilic Graphs}
It is previously mentioned in Sec.~\ref{sec-vulnerability} that the newly defined similarity matrix $\Omega(\mathbf{A}^{\tau}\mathbf{X})$ can successfully tell apart the malicious links out of normal links based on the density of the similarity scores for each link. In this section, we additionally explore whether this similarity metric can serve as a malicious effect detector for homophilic graphs. Fig.~\ref{Fig-homo-simi} provides the density plots of similarity scores for homophilic graphs. These results indicate that the similarity metric can also distinguish malicious links from normal links on homophilic graphs. Additionally, it can also verify the success of \textbf{NSPGNN} against graph adversarial attacks on homophilic graphs.   
\subsection{Sensitivity Analysis}
In this section, we provide the sensitivity analysis on the number of nearest neighbors of positive kNN graph $k_{1}$ and negative kNN graph $k_{2}$. Fig.~\ref{fig-sensitivity-mat-chameleon} presents the impacts of different choices of $\tau_1=1$ and $\tau_2=2$ in the proposed model. It is observed that the performance of \textbf{NSPGNN} is more sensitive to $k_{1}$ than $k_{2}$. This phenomenon indicates that the impact of the positive kNN graph is larger than the negative kNN graph, which coincides with the ablation results. 

\subsection{Impacts of $\tau$}
In this section, we analyze different choices of the vital hyperparameter $\tau$ for our proposed method. We evaluate the clean accuracies and robust performances of \textbf{NSPGNN} with different settings of $\tau$ on Tab.~\ref{tab-varying-tau}. For example, $\{1,2,3\}$ means we construct the dual kNN graphs based on three kinds of similarity scores $\Omega(\mathbf{AX})$, $\Omega(\mathbf{A}^2\mathbf{X})$ and $\Omega(\mathbf{A}^3\mathbf{X})$. It is observed that choosing $\tau=1$ and $\tau=2$ to construct the dual kNN graphs can achieve the best performances on clean and poisoned graphs, which is consistent with the empirical results on Tab.~\ref{tab-KL-Divergence} that $\tau=1$ and $\tau=2$ can provide the largest distance between the benign links and malicious links based on the similarity scores.    

\begin{table}[h]
	\centering
	\caption{Impacts of different choices of $\tau$.}
	\label{tab-varying-tau}
	\resizebox{0.8\columnwidth}{!}{%
		\begin{tabular}{c|cccccc}
			\toprule[1.pt]
			$\delta_{atk}$ &$\{1\}$&$\{2\}$&$\{1,2\}$&$\{1,3\}$&$\{2,3\}$&$\{1,2,3\}$\\
			\hline
           $0\%$ & $70.39$ & $67.98$ & $\mathbf{73.84}$ & $68.42$ & $68.20$ & $66.67$\\ 
           \hline
           $25\%$ & $56.58$ & $57.46$ & $\mathbf{60.53}$ & $57.24$ & $54.39$ & $57.46$\\
			\bottomrule[1.pt]
		\end{tabular}
	}
\end{table}

\subsection{Impacts of neighbor-similarity-preserved propagation}
In this section, we analyze the impact of adaptive neighbor-similarity-preserved propagation in the proposed model. In the ablation version, we remove the neighbor-similarity-preserved propagation mechanism and directly sanitize the potential malicious links based on the descending order of the similarity scores with $\tau=1$ and $\tau=2$, and then feed the sanitized graph into a graph neural network for training. The results in Tab.~\ref{tab-sanitize} demonstrate that utilizing the adaptive neighbor-similarity-preserved propagation performs much better than directly sanitizing the raw graph data via the similarity scores.

\begin{table}[h]
	\centering
	\caption{Impacts of neighbor-similarity-preserved propagation.}
	\label{tab-sanitize}
	\resizebox{0.8\columnwidth}{!}{%
		\begin{tabular}{c|ccc}
			\toprule[1.pt]
			$\delta_{atk}$ &NSP-Sanitize ($\tau=1$)&NSP-Sanitize ($\tau=2$)&NSPGNN\\
			\hline
           $0\%$ & $62.94$ & $64.04$ & $\mathbf{73.84}$\\ 
           \hline
           $25\%$ & $49.12$ & $47.15$ & $\mathbf{60.53}$\\
			\bottomrule[1.pt]
		\end{tabular}
	}
\end{table}

%\subsection{Node Embeddings Visualization}
%Fig.~\ref{fig-embeds} presents the scatterplot of the node embeddings of \textbf{NSPGNN} and other baselines under $25\%$ Mettack. The node embeddings are pre-processed by t-SNE~\cite{tSNE}--a typical dimension reduction technique for easy data visualization. It is observed that the distribution of the representations of the proposed method forms more separable decision boundaries than baselines, which demonstrates the effectiveness of preserving the neighbor similarity. 

\section{Conclusion}
We discover the vulnerability of the graph data management system and present an effective robust graph structural learning approach to adaptively supervise the reliable message-passing mechanism during training. Specifically, it endeavors to enhance the adversarial robustness of graph learning methods on both homophilic and heterophilic graphs by preserving neighbor similarities. Through comprehensive analysis, we establish a connection between neighbor similarities and the negative classification loss, revealing that malicious adversaries tend to connect node pairs with low-level neighbor similarities. Leveraging this insight, we propose a novel robust graph structural learning approach where node features are adaptively propagated along the positive kNN graphs to smooth the features of node pairs with high similarity scores, and along the negative kNN graphs to discriminate the node pairs with low similarity scores. In addition, preserving neighbor similarity can serve as a form of graph data augmentation, improving the performance of node classification by refining the clean graph's structure. Thus, our proposed method can lay the foundation for enhancing the security of the graph data management system under diverse graph homophily.

\bibliographystyle{IEEEtran}
\bibliography{citation}

\begin{thebibliography}{10}
\providecommand{\url}[1]{#1}
\csname url@samestyle\endcsname
\providecommand{\newblock}{\relax}
\providecommand{\bibinfo}[2]{#2}
\providecommand{\BIBentrySTDinterwordspacing}{\spaceskip=0pt\relax}
\providecommand{\BIBentryALTinterwordstretchfactor}{4}
\providecommand{\BIBentryALTinterwordspacing}{\spaceskip=\fontdimen2\font plus
\BIBentryALTinterwordstretchfactor\fontdimen3\font minus
  \fontdimen4\font\relax}
\providecommand{\BIBforeignlanguage}[2]{{%
\expandafter\ifx\csname l@#1\endcsname\relax
\typeout{** WARNING: IEEEtran.bst: No hyphenation pattern has been}%
\typeout{** loaded for the language `#1'. Using the pattern for}%
\typeout{** the default language instead.}%
\else
\language=\csname l@#1\endcsname
\fi
#2}}
\providecommand{\BIBdecl}{\relax}
\BIBdecl

\bibitem{battaglia2018relational}
P.~W. Battaglia, J.~B. Hamrick, V.~Bapst, A.~Sanchez-Gonzalez, V.~Zambaldi,
  M.~Malinowski, A.~Tacchetti, D.~Raposo, A.~Santoro, R.~Faulkner
  \emph{et~al.}, ``Relational inductive biases, deep learning, and graph
  networks,'' \emph{arXiv preprint arXiv:1806.01261}, 2018.

\bibitem{zhang2020deep}
Z.~Zhang, P.~Cui, and W.~Zhu, ``Deep learning on graphs: A survey,'' \emph{IEEE
  Transactions on Knowledge and Data Engineering}, vol.~34, no.~1, pp.
  249--270, 2020.

\bibitem{wu2020comprehensive}
Z.~Wu, S.~Pan, F.~Chen, G.~Long, C.~Zhang, and S.~Y. Philip, ``A comprehensive
  survey on graph neural networks,'' \emph{IEEE transactions on neural networks
  and learning systems}, vol.~32, no.~1, pp. 4--24, 2020.

\bibitem{GCN}
\BIBentryALTinterwordspacing
T.~N. Kipf and M.~Welling, ``Semi-supervised classification with graph
  convolutional networks,'' in \emph{International Conference on Learning
  Representations}, 2017. [Online]. Available:
  \url{https://openreview.net/forum?id=SJU4ayYgl}
\BIBentrySTDinterwordspacing

\bibitem{GraphSage}
W.~Hamilton, Z.~Ying, and J.~Leskovec, ``Inductive representation learning on
  large graphs,'' \emph{Advances in neural information processing systems},
  vol.~30, 2017.

\bibitem{li2019semi}
J.~Li, Y.~Rong, H.~Cheng, H.~Meng, W.~Huang, and J.~Huang, ``Semi-supervised
  graph classification: A hierarchical graph perspective,'' in \emph{The World
  Wide Web Conference}, 2019, pp. 972--982.

\bibitem{shervashidze2011weisfeiler}
N.~Shervashidze, P.~Schweitzer, E.~J. Van~Leeuwen, K.~Mehlhorn, and K.~M.
  Borgwardt, ``Weisfeiler-lehman graph kernels.'' \emph{Journal of Machine
  Learning Research}, vol.~12, no.~9, 2011.

\bibitem{shibata2012link}
N.~Shibata, Y.~Kajikawa, and I.~Sakata, ``Link prediction in citation
  networks,'' \emph{Journal of the American society for information science and
  technology}, vol.~63, no.~1, pp. 78--85, 2012.

\bibitem{daud2020applications}
N.~N. Daud, S.~H. Ab~Hamid, M.~Saadoon, F.~Sahran, and N.~B. Anuar,
  ``Applications of link prediction in social networks: A review,''
  \emph{Journal of Network and Computer Applications}, vol. 166, p. 102716,
  2020.

\bibitem{homophily}
M.~McPherson, L.~Smith-Lovin, and J.~M. Cook, ``Birds of a feather: Homophily
  in social networks,'' \emph{Annual review of sociology}, vol.~27, no.~1, pp.
  415--444, 2001.

\bibitem{H2GCN}
J.~Zhu, Y.~Yan, L.~Zhao, M.~Heimann, L.~Akoglu, and D.~Koutra, ``Beyond
  homophily in graph neural networks: Current limitations and effective
  designs,'' \emph{Advances in neural information processing systems}, vol.~33,
  pp. 7793--7804, 2020.

\bibitem{GBKGNN}
L.~Du, X.~Shi, Q.~Fu, X.~Ma, H.~Liu, S.~Han, and D.~Zhang, ``Gbk-gnn: Gated
  bi-kernel graph neural networks for modeling both homophily and
  heterophily,'' in \emph{Proceedings of the ACM Web Conference 2022}, 2022,
  pp. 1550--1558.

\bibitem{FAGNN}
D.~Bo, X.~Wang, C.~Shi, and H.~Shen, ``Beyond low-frequency information in
  graph convolutional networks,'' in \emph{Proceedings of the AAAI Conference
  on Artificial Intelligence}, vol.~35, no.~5, 2021, pp. 3950--3957.

\bibitem{GPRGNN}
\BIBentryALTinterwordspacing
E.~Chien, J.~Peng, P.~Li, and O.~Milenkovic, ``Adaptive universal generalized
  pagerank graph neural network,'' in \emph{International Conference on
  Learning Representations}, 2021. [Online]. Available:
  \url{https://openreview.net/forum?id=n6jl7fLxrP}
\BIBentrySTDinterwordspacing

\bibitem{ACMGNN}
S.~Luan, C.~Hua, Q.~Lu, J.~Zhu, M.~Zhao, S.~Zhang, X.-W. Chang, and D.~Precup,
  ``Revisiting heterophily for graph neural networks,'' \emph{Advances in
  neural information processing systems}, vol.~35, pp. 1362--1375, 2022.

\bibitem{Nettack}
D.~Z{\"u}gner, A.~Akbarnejad, and S.~G{\"u}nnemann, ``Adversarial attacks on
  neural networks for graph data,'' in \emph{SIGKDD}, 2018, pp. 2847--2856.

\bibitem{Mettack}
\BIBentryALTinterwordspacing
D.~Zügner and S.~Günnemann, ``Adversarial attacks on graph neural networks
  via meta learning,'' in \emph{International Conference on Learning
  Representations}, 2019. [Online]. Available:
  \url{https://openreview.net/forum?id=Bylnx209YX}
\BIBentrySTDinterwordspacing

\bibitem{TopologyAttack}
K.~Xu, H.~Chen, S.~Liu, P.-Y. Chen, T.-W. Weng, M.~Hong, and X.~Lin, ``Topology
  attack and defense for graph neural networks: An optimization perspective,''
  in \emph{Proceedings of the Twenty-Eighth International Joint Conference on
  Artificial Intelligence, {IJCAI-19}}.\hskip 1em plus 0.5em minus 0.4em\relax
  International Joint Conferences on Artificial Intelligence Organization, 7
  2019, pp. 3961--3967.

\bibitem{BinarizedAttack}
Y.~Zhu, Y.~Lai, K.~Zhao, X.~Luo, M.~Yuan, J.~Ren, and K.~Zhou,
  ``Binarizedattack: Structural poisoning attacks to graph-based anomaly
  detection,'' in \emph{2022 IEEE 38th International Conference on Data
  Engineering (ICDE)}, 2022, pp. 14--26.

\bibitem{ProGNN}
W.~Jin, Y.~Ma, X.~Liu, X.~Tang, S.~Wang, and J.~Tang, ``Graph structure
  learning for robust graph neural networks,'' in \emph{Proceedings of the 26th
  ACM SIGKDD international conference on knowledge discovery and data mining},
  2020, pp. 66--74.

\bibitem{GCNJaccard}
H.~Wu, C.~Wang, Y.~Tyshetskiy, A.~Docherty, K.~Lu, and L.~Zhu, ``Adversarial
  examples for graph data: Deep insights into attack and defense,'' in
  \emph{Proceedings of the Twenty-Eighth International Joint Conference on
  Artificial Intelligence, {IJCAI-19}}.\hskip 1em plus 0.5em minus 0.4em\relax
  International Joint Conferences on Artificial Intelligence Organization, 7
  2019, pp. 4816--4823.

\bibitem{GCNSVD}
\BIBentryALTinterwordspacing
N.~Entezari, S.~A. Al-Sayouri, A.~Darvishzadeh, and E.~E. Papalexakis, ``All
  you need is low (rank): Defending against adversarial attacks on graphs,'' in
  \emph{Proceedings of the 13th International Conference on Web Search and Data
  Mining}, ser. WSDM '20.\hskip 1em plus 0.5em minus 0.4em\relax New York, NY,
  USA: Association for Computing Machinery, 2020, p. 169–177. [Online].
  Available: \url{https://doi.org/10.1145/3336191.3371789}
\BIBentrySTDinterwordspacing

\bibitem{GNNGUARD}
X.~Zhang and M.~Zitnik, ``Gnnguard: Defending graph neural networks against
  adversarial attacks,'' in \emph{Proceedings of the 34th International
  Conference on Neural Information Processing Systems}, ser. NIPS'20.\hskip 1em
  plus 0.5em minus 0.4em\relax Red Hook, NY, USA: Curran Associates Inc., 2020.

\bibitem{chameleon}
\BIBentryALTinterwordspacing
R.~A. Rossi and N.~K. Ahmed, ``The network data repository with interactive
  graph analytics and visualization,'' in \emph{AAAI}, 2015. [Online].
  Available: \url{https://networkrepository.com}
\BIBentrySTDinterwordspacing

\bibitem{BMGNN}
D.~He, C.~Liang, H.~Liu, M.~Wen, P.~Jiao, and Z.~Feng, ``Block modeling-guided
  graph convolutional neural networks,'' in \emph{Proceedings of the AAAI
  conference on artificial intelligence}, vol.~36, no.~4, 2022, pp. 4022--4029.

\bibitem{garnet}
\BIBentryALTinterwordspacing
C.~Deng, X.~Li, Z.~Feng, and Z.~Zhang, ``{GARNET}: Reduced-rank topology
  learning for robust and scalable graph neural networks,'' in \emph{The First
  Learning on Graphs Conference}, 2022. [Online]. Available:
  \url{https://openreview.net/forum?id=kvwWjYQtmw}
\BIBentrySTDinterwordspacing

\bibitem{SGC}
F.~Wu, A.~Souza, T.~Zhang, C.~Fifty, T.~Yu, and K.~Weinberger, ``Simplifying
  graph convolutional networks,'' in \emph{Proceedings of the 36th
  International Conference on Machine Learning}, ser. Proceedings of Machine
  Learning Research, K.~Chaudhuri and R.~Salakhutdinov, Eds., vol.~97.\hskip
  1em plus 0.5em minus 0.4em\relax PMLR, 09--15 Jun 2019, pp. 6861--6871.

\bibitem{AttackCD}
\BIBentryALTinterwordspacing
J.~Li, H.~Zhang, Z.~Han, Y.~Rong, H.~Cheng, and J.~Huang, ``Adversarial attack
  on community detection by hiding individuals,'' in \emph{Proceedings of The
  Web Conference 2020}, ser. WWW '20.\hskip 1em plus 0.5em minus 0.4em\relax
  New York, NY, USA: Association for Computing Machinery, 2020, p. 917–927.
  [Online]. Available: \url{https://doi.org/10.1145/3366423.3380171}
\BIBentrySTDinterwordspacing

\bibitem{cossim}
A.~Singhal \emph{et~al.}, ``Modern information retrieval: A brief overview,''
  \emph{IEEE Data Eng. Bull.}, vol.~24, no.~4, pp. 35--43, 2001.

\bibitem{oversmoothing}
Q.~Li, Z.~Han, and X.-M. Wu, ``Deeper insights into graph convolutional
  networks for semi-supervised learning,'' in \emph{Proceedings of the AAAI
  conference on artificial intelligence}, vol.~32, no.~1, 2018.

\bibitem{KLD}
\BIBentryALTinterwordspacing
I.~Csiszar, ``{$I$-Divergence Geometry of Probability Distributions and
  Minimization Problems},'' \emph{The Annals of Probability}, vol.~3, no.~1,
  pp. 146 -- 158, 1975. [Online]. Available:
  \url{https://doi.org/10.1214/aop/1176996454}
\BIBentrySTDinterwordspacing

\bibitem{kNN-graph}
\BIBentryALTinterwordspacing
G.~L. Miller, S.-H. Teng, W.~Thurston, and S.~A. Vavasis, ``Separators for
  sphere-packings and nearest neighbor graphs,'' \emph{J. ACM}, vol.~44, no.~1,
  p. 1–29, jan 1997. [Online]. Available:
  \url{https://doi.org/10.1145/256292.256294}
\BIBentrySTDinterwordspacing

\bibitem{ReLU}
A.~F. Agarap, ``Deep learning using rectified linear units (relu),''
  \emph{arXiv preprint arXiv:1803.08375}, 2018.

\bibitem{MLP}
S.~Haykin, \emph{Neural networks: a comprehensive foundation}.\hskip 1em plus
  0.5em minus 0.4em\relax Prentice Hall PTR, 1994.

\bibitem{RLBisection}
J.~Chen, H.-r. Fang, and Y.~Saad, ``Fast approximate knn graph construction for
  high dimensional data via recursive lanczos bisection,'' \emph{J. Mach.
  Learn. Res.}, vol.~10, p. 1989–2012, dec 2009.

\bibitem{MapReduce}
\BIBentryALTinterwordspacing
W.~Dong, C.~Moses, and K.~Li, ``Efficient k-nearest neighbor graph construction
  for generic similarity measures,'' in \emph{Proceedings of the 20th
  International Conference on World Wide Web}, ser. WWW '11.\hskip 1em plus
  0.5em minus 0.4em\relax New York, NY, USA: Association for Computing
  Machinery, 2011, p. 577–586. [Online]. Available:
  \url{https://doi.org/10.1145/1963405.1963487}
\BIBentrySTDinterwordspacing

\bibitem{tolokers}
\BIBentryALTinterwordspacing
O.~Platonov, D.~Kuznedelev, M.~Diskin, A.~Babenko, and L.~Prokhorenkova, ``A
  critical look at the evaluation of {GNN}s under heterophily: Are we really
  making progress?'' in \emph{The Eleventh International Conference on Learning
  Representations}, 2023. [Online]. Available:
  \url{https://openreview.net/forum?id=tJbbQfw-5wv}
\BIBentrySTDinterwordspacing

\bibitem{Cora}
A.~K. McCallum, K.~Nigam, J.~Rennie, and K.~Seymore, ``Automating the
  construction of internet portals with machine learning,'' \emph{Information
  Retrieval}, vol.~3, pp. 127--163, 2000.

\bibitem{Photo}
J.~McAuley, C.~Targett, Q.~Shi, and A.~Van Den~Hengel, ``Image-based
  recommendations on styles and substitutes,'' in \emph{Proceedings of the 38th
  international ACM SIGIR conference on research and development in information
  retrieval}, 2015, pp. 43--52.

\bibitem{PytorchGeometric}
M.~Fey and J.~E. Lenssen, ``Fast graph representation learning with {PyTorch
  Geometric},'' in \emph{ICLR Workshop on Representation Learning on Graphs and
  Manifolds}, 2019.

\bibitem{deeprobust}
Y.~Li, W.~Jin, H.~Xu, and J.~Tang, ``Deeprobust: A pytorch library for
  adversarial attacks and defenses,'' \emph{arXiv preprint arXiv:2005.06149},
  2020.

\bibitem{RGCN}
\BIBentryALTinterwordspacing
D.~Zhu, Z.~Zhang, P.~Cui, and W.~Zhu, ``Robust graph convolutional networks
  against adversarial attacks,'' in \emph{Proceedings of the 25th ACM SIGKDD
  International Conference on Knowledge Discovery \& Data Mining}, ser. KDD
  '19.\hskip 1em plus 0.5em minus 0.4em\relax New York, NY, USA: Association
  for Computing Machinery, 2019, p. 1399–1407. [Online]. Available:
  \url{https://doi.org/10.1145/3292500.3330851}
\BIBentrySTDinterwordspacing

\bibitem{AirGNN}
X.~Liu, J.~Ding, W.~Jin, H.~Xu, Y.~Ma, Z.~Liu, and J.~Tang, ``Graph neural
  networks with adaptive residual,'' in \emph{Advances in Neural Information
  Processing Systems}, A.~Beygelzimer, Y.~Dauphin, P.~Liang, and J.~W. Vaughan,
  Eds., 2021.

\bibitem{ElasticGNN}
X.~Liu, W.~Jin, Y.~Ma, Y.~Li, H.~Liu, Y.~Wang, M.~Yan, and J.~Tang, ``Elastic
  graph neural networks,'' in \emph{International Conference on Machine
  Learning}.\hskip 1em plus 0.5em minus 0.4em\relax PMLR, 2021, pp. 6837--6849.

\bibitem{GreatX}
J.~Li, B.~Wu, C.~Hou, G.~Fu, Y.~Bian, L.~Chen, J.~Huang, and Z.~Zheng, ``Recent
  advances in reliable deep graph learning: Inherent noise, distribution shift,
  and adversarial attack,'' 2023.

\bibitem{Adam}
D.~P. Kingma and J.~Ba, ``Adam: A method for stochastic optimization,'' 2017.

\bibitem{H2GCNSVD}
\BIBentryALTinterwordspacing
J.~Zhu, J.~Jin, D.~Loveland, M.~T. Schaub, and D.~Koutra, ``How does
  heterophily impact the robustness of graph neural networks? theoretical
  connections and practical implications,'' in \emph{Proceedings of the 28th
  ACM SIGKDD Conference on Knowledge Discovery and Data Mining}, ser. KDD
  '22.\hskip 1em plus 0.5em minus 0.4em\relax New York, NY, USA: Association
  for Computing Machinery, 2022, p. 2637–2647. [Online]. Available:
  \url{https://doi.org/10.1145/3534678.3539418}
\BIBentrySTDinterwordspacing

\end{thebibliography}

\end{document}